\documentclass{article}

\usepackage[preprint]{neurips_2026}

\usepackage[utf8]{inputenc} 
\usepackage[T1]{fontenc}    
\usepackage{hyperref}       
\usepackage{url}            
\usepackage{booktabs}       
\usepackage{amsfonts}       
\usepackage{nicefrac}       
\usepackage{microtype}      
\usepackage{xcolor}         

\usepackage{graphicx}
\usepackage{paralist}
\usepackage{amsmath}
\usepackage{amssymb}
\usepackage{mathtools}
\usepackage{amsthm}
\usepackage[capitalize,noabbrev]{cleveref}

\PassOptionsToPackage{noend}{algorithmic}
\usepackage{algorithm}
\usepackage{algorithmic}
\usepackage{caption}

\usepackage{comment}
\usepackage{dirtytalk}
\usepackage{xfrac}
\usepackage{makecell}
\usepackage{float}
\usepackage{wrapfig}

\theoremstyle{plain}
\newtheorem{theorem}{Theorem}[section]
\newtheorem{proposition}[theorem]{Proposition}
\newtheorem{lemma}[theorem]{Lemma}
\newtheorem{corollary}[theorem]{Corollary}
\newtheorem{condition}[theorem]{Condition}

\theoremstyle{definition}
\newtheorem{definition}[theorem]{Definition}
\newtheorem{assumption}[theorem]{Assumption}

\theoremstyle{remark}
\newtheorem{remark}[theorem]{Remark}

\newenvironment{sketch}[1][Proof Sketch]
  {\begin{proof}[#1]}
  {\end{proof}}

\DeclareMathOperator*{\argmax}{arg\,max}

\newcommand{\Regval}{\mathcal R}
\newcommand{\UCBV}{\overline U}
\newcommand{\LCBV}{\underline U}
\newcommand{\UCBu}{\overline u}
\newcommand{\LCBu}{\underline u}
\newcommand{\epsf}{\epsilon^{f}}

\newcommand{\Gen}{\mathsf{Gen}}
\newcommand{\epsfT}{\overline{\epsilon}^f}
\newcommand{\epsU}{\epsilon^U}
\newcommand{\epsUm}{\epsilon^U_m}
\newcommand{\Rsusp}{\mathcal{R}^{+}} 
\newcommand{\Ktmax}{N_T}        
\newcommand{\AlgoName}{\textsc{GSR}}        

\newcommand{\Lbar}{\overline L}     
\DeclareMathOperator{\logit}{logit}

\newcommand{\sigmoid}{\sigma} 
\newcommand{\epsclip}{\varepsilon_{\mathrm{clip}}}
\newcommand{\clip}[3]{\mathrm{clip}_{[#2,#3]}\!\left(#1\right)}

\usepackage{enumitem}  

\title{Open-Ended Task Discovery via\\ Bayesian Optimization}

\author{%
  Masaki Adachi, Yuta Suzuki\\
  Lattice Lab\\
  Toyota Motor Corporation\\
  \texttt{adachi\_masaki@mail.toyota.co.jp} \\
  \And
  Juliusz Ziomek\\
  Machine Learning Research Group\\
  University of Oxford\\
  \\
}

\begin{document}

\maketitle

\begin{abstract}
  When applying Bayesian optimization (BO) to scientific workflow, a major yet often overlooked source of uncertainty is the \emph{task} itself—namely, \emph{what} to optimize and \emph{how} to evaluate it—which can evolve as evidence accumulates.
  We introduce \emph{Generate-Select-Refine} (GSR), a open-ended BO framework that alternates between task generation and task optimization. Starting from a user-provided seed task, GSR generates new tasks in a coarse-to-fine manner while a task-acquisition function schedules optimization.
  Asymptotically, it concentrates evaluations on the best task, incurring only logarithmic regret overhead relative to single-task BO.
  We apply GSR to new product development, chemical synthesis scaling, algorithm analysis, and patent repurposing, where it outperforms existing LLM-based optimizers.
\end{abstract}

\section{Introduction}\label{sec:intro}
\begin{wrapfigure}[23]{r}{0.44\textwidth}
    \vspace{-3em}
    \centering
    \includegraphics[width=0.43\textwidth]{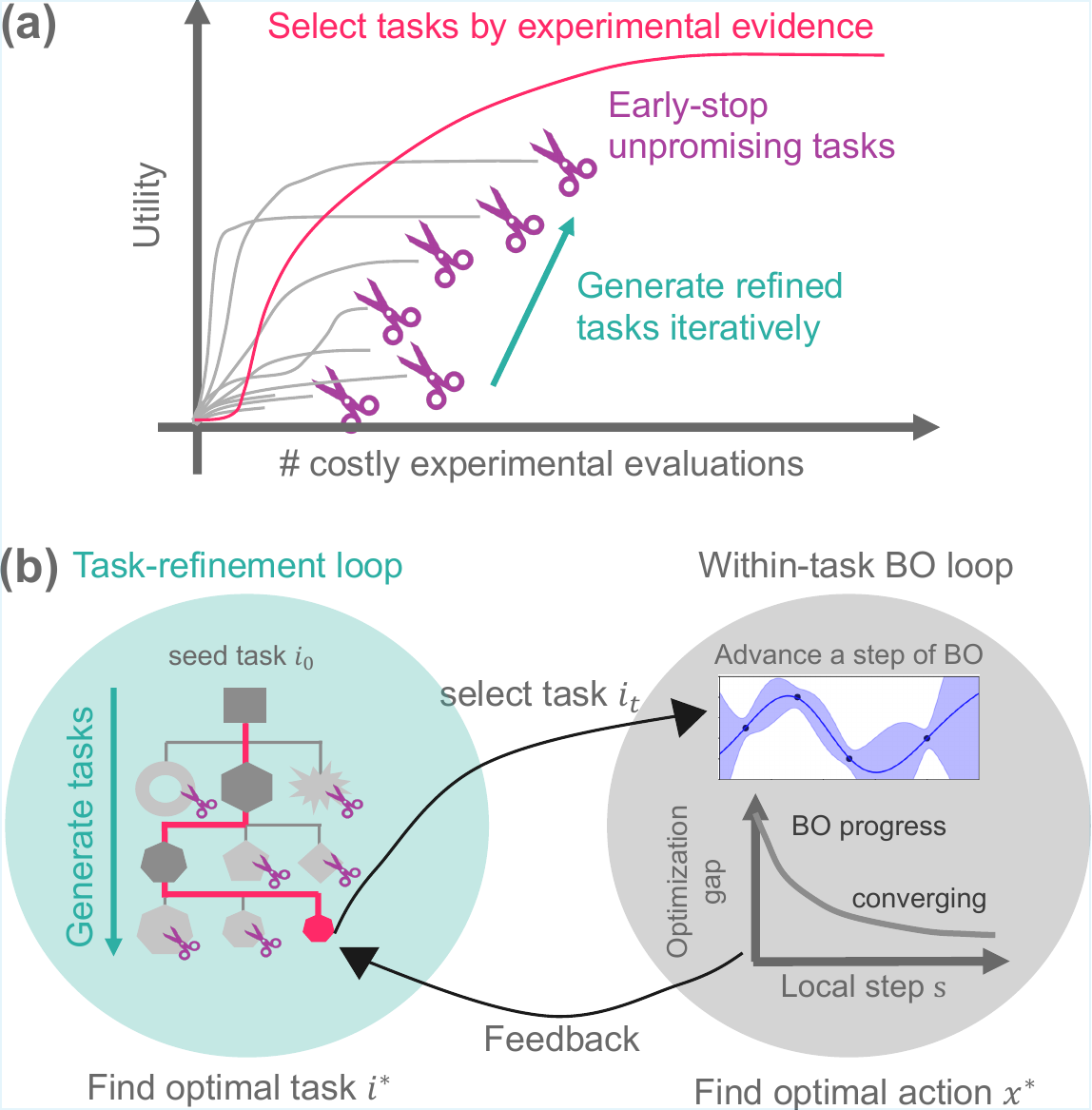}
    \caption{
    \textbf{Conceptual overview}: Starting from a seed task $i_0$, the new tasks are iteratively generated toward the $\epsilon$-optimal task $i^\star$ and identifies its optimum $x^\star$.
    (a) Experimental progress: task generation and elimination over time.
    (b) Two-loop structure: feedback from within-task BO informs task refinement via a coarse-to-fine scheduler.
    }
    \label{fig:overview}
\end{wrapfigure}
Bayesian optimization (BO; \citep{garnett2023bayesian}) is a sample-efficient black-box optimizer, widely used in domains such as drug discovery \citep{gomez2018automatic, pyzer2018bayesian} and materials science \citep{liang2021benchmarking, adachi2024adaptive}. Its practical impact has grown further with self-driving laboratories \citep{abolhasani2023rise}, where BO proposes experiments and robotic platforms execute them autonomously.
A natural next frontier is planning: deciding \emph{which} optimization problem to pursue and \emph{how} to formulate it. 
A canonical example is \emph{inverse optimization}: rather than solving a given task, the goal is to identify the task under which a given solution is most valuable. This setting is common in science and engineering: a newly discovered material or algorithm may be promising, yet its best application is often unknown. For instance, around 12,000 new crystals are discovered each year~\citep{hellenbrandt2004inorganic}, but only a small fraction achieve industrial application, partly because discovering the right application is difficult. 

The main difficulty comes from a finite experimental budget $T$. While recent advances in large language models (LLMs; \citep{achiam2023gpt, team2023gemini}) make automatic task generation plausible \citep{stanley2015greatness, du2025accelerating}, we must judge \emph{achievability} of the generated tasks within the budget. High-reward tasks may be too ambitious to complete, whereas easy tasks may offer limited value. Moreover, achievability is only revealed through trial and error, evaluating tasks itself consumes experimental resources. Naively exploring many tasks can therefore waste budget on infeasible directions.

We study the problem of generating the most rewarding task that is achievable within budget $T$. We call such a task \emph{$\epsilon$-optimal} and formalize this setting as \emph{open-ended BO}: tasks are generated online from past experience, budget is allocated to tasks that appear promising and achievable, and standard BO is run within selected tasks to optimize $x^\star$. Unlike standard BO over a predefined task vector, tasks here may be structured, are not enumerated in advance, and reveal their value only indirectly through partial BO trajectories.
To solve this problem, we propose \emph{Generate-Select-Refine} (GSR; Fig.~\ref{fig:overview}), which treats task discovery itself as an generative-optimization problem and uses regret bounds as a principled \emph{task manager} for validating generated tasks. We prove that GSR incurs only a logarithmic regret overhead compared with running BO on a single fixed task.

\textbf{Our contributions.}
\vspace{-0.3em}
\begin{compactenum}
    \item We formalize \emph{open-ended BO}: sample-efficient online task generation and task optimization.
    \item We introduce GSR, a confidence-gated framework that provably finds $\epsilon$-optimal task and its solution $x^\star$ with logarithmic regret overhead relative to single-task BO.
    \item We provide a broad empirical study, including 10 real-world scenarios and 6 synthetic tests.
\end{compactenum}
\vspace{-0.5em}
\subsection{Related Work}\label{subsec:related}
\vspace{-0.3em}
To our knowledge, no prior work studies our setting. Existing directions such as meta-BO and bi-level BO assume predefined and parameterized task spaces, whereas we consider \emph{open-ended} task optimization: tasks are generated on the fly, may be free-form rather than parameterized, and are evaluated only through partial observations.
\textbf{Meta BO} either transfers across related tasks \citep{multi2013swersky, Volpp2020Meta-Learning, dai2022provably, maraval2023end, osselin2025natural} or adapts hyperparameters online \citep{wang2018regret, berkenkamp2019no, ziomek2024bayesian, ziomek2025time}. Instead, our GSR is the \emph{generative optimization} of the task itself.
\textbf{LLMs + BO} are studied as surrogates \citep{liu2024large, james2024llm, gupta2025llms}, collaborative agents \citep{adachi2024looping, xu2024principled} or kernel generators \citep{suwandi2025adaptive} but under a fixed task. \textbf{Bi-level BO}~\citep{fu2024convergence, chew2025bilbo, dogan2023bilevel} studies hierarchical optimization where the lower-level problem depends on the upper-level one, such as GP hyperparameter tuning for BO. However, it still assumes a fixed, bounded, continuous search space, unlike our setting with unbounded and free-form tasks.
\textbf{BO from partial observations}~\citep{swersky2014freeze, klein2017fast, nguyen2020bayesian, adriaensen2023efficient, rakotoarison2024in} infer long-term utility from partial learning curves, which select among a predefined set of tasks, whereas we must generate new tasks online.
\textbf{BO for Task optimization} optimizes restricted task parameters, such as unknown domain bounds \citep{ha2019bayesian, gupta2020sub} or objective weights from preference feedback \citep{abdolshah2019multi, lin2022preference}. These are special cases of our setting with fixed parameterized task spaces.
\textbf{Evolutionary LLMs}~\citep{zelikman2024self, novikov2025alphaevolve, lange2025shinkaevolve, zhang2025darwin} perform generative optimization with direct feedback, whereas we must infer long-term value from partial lower-level optimization trajectories.

\vspace{-0.75em}
\section{Problem Setting}\label{sec:problem}
\vspace{-0.25em}
\subsection{Tasks, evaluations, and incumbents}\label{sec:tasks_rounds}
\vspace{-0.3em}
\textbf{Tasks.}
There is an unbounded collection of tasks indexed by $i\in\mathbb N$.
Task $i$ is specified by a compact domain $\mathcal X^{(i)}\subset\mathbb R^{d^{(i)}}$ and an unknown objective
$f^{(i)}:\mathcal X^{(i)}\to\mathbb R$.\\
\textbf{Global vs.\ local time.}
We index global rounds by $t\in[T]$.
At each $t$, the task-optimizer selects a task $i_t$ and evaluates one design $x_t^{(i_t)}\in\mathcal X^{(i_t)}$,
observing the noisy oracle
$y_t^{(i_t)} = f^{(i_t)}\!\bigl(x_t^{(i_t)}\bigr) + \xi_t^f$,
where $\xi_t^f \sim \mathrm{SubGauss}(\sigma_f^2)$.
Let $s_t^{(i)}$ be the local counter that task $i$ has been selected until $t$ and $\mathcal I_T$ be the set of tasks instantiated by $T$. For brevity, we write $x_t^{(i)} := x_{s_t^{(i)}}^{(i)}$ (and similarly for other quantities); the notations are summarized in Table~\ref{tab:notation} in the Appendix.\\
\textbf{Incumbent} is best-so-far observed value\footnote{In  analysis, this is noiseless $\overline{y}_s^{(i)} :=  \max_{1\le r\le s} f^{(i)}(x_r^{(i)})$.} after $s$ local evaluations on task $i$: 
$\overline{y}_s^{(i)} := \max_{1\le r\le s} y_r^{(i)}$.\\
\textbf{Optimization gap.}
Each task runs a standard GP-UCB~\citep{srinivas2009gaussian}. For readers unfamiliar with GP-UCB, see Appendix~\ref{app:gp_ucb}.
We only use an \emph{optimization gap bound}: w.p. $\geq 1 - \delta_f^{(i)}$, after $s$ evaluations,
\vspace{-0.5em}
\begin{equation}
\label{eq:epsf_def_main}
f^{\star(i)}-\overline{y}_s^{(i)}
\ \le\
\epsf_s
:= 2\sqrt{C_\lambda\,\beta^{(i)}_s\,\gamma^{(i)}_s /s},
\qquad
f^{\star(i)} := \max_{x\in\mathcal X^{(i)}} f^{(i)}(x),
\end{equation}
where $\beta^{(i)}_s$ and $\gamma^{(i)}_s$ are exploration and information-gain terms and $C_\lambda > 0$ is a constant.
The optimization gap $\epsf_s$ upper-bounds the worst-case error of the global optimum estimate, and satisfies $\epsf_s\to 0$ as $s\to\infty$, i.e., GP-UCB asymptotically identifies the global optimum $f^{\star(i)}$ (no-regret) for commonly used kernels (see \citep{lee2025consequences}).
$\epsf_s$ is computable without access to $f^{\star(i)}$, so it provides a principled cross-task progress metric.
For uniformity over instantiated tasks, define $\epsfT_s:=\max_{i\in\mathcal I_T}\epsf_s$, so $\epsf_s\le \epsfT_s$ for all $i\in\mathcal I_T$ and $s\in[T]$ (Lemma~\ref{lem:uniform_gp_confidence_tasks}).

\vspace{-0.5em}
\subsection{Latent utility}\label{sec:utility}
\vspace{-0.3em}
Since the scale of $f^{(i)}$ can vary across tasks, we introduce a scale-invariant utility function $u^{(i)}:\mathbb R\to[0,1]$ applied to the incumbent $\overline{y}_s^{(i)}$. Any rankable function can be utility; e.g., Bayesian surprise \citep{agarwal2026autodiscovery} or human preference \citep{bradley1952rank, adachi2025bayesian}.
In other words, incomparable tasks are not our scope.
\begin{definition}[Utility regularity]\label{ass:mono_lip}
$u^{(i)}$ is monotone nondecreasing and $L^{(i)}$-Lipschitz:
$|u^{(i)}(y)-u^{(i)}(y')|\le L^{(i)}|y-y'|$ for all $y,y'\in\mathbb R$.
Let $L^\star:=\sup_i L^{(i)}<\infty$.
\end{definition}
\vspace{-0.3em}
Monotonicity matches the optimization intent (better designs should not be judged worse).
Lipschitzness is a standard regularity condition in optimization.\\
\textbf{Noisy utility calls.}
For task $i_t$ selected at $t$, after observing $y^{(i_t)}_t$, we query the noisy utility oracle
$
\tilde{u}^{(i_t)}_t
=
u^{(i_t)}\!\bigl(\overline{y}^{(i_t)}_t\bigr) + \xi^u_t$,
where 
$\xi^u_t \sim \mathrm{SubGauss}\!\Bigl(\sigma^2_u\Bigr)$.

\subsection{Long-run value and BO-achievable optimality}\label{sec:bo_resolution_oracle}
\textbf{Long-run value.}
A task $i$ has an unknown value
$U^{(i)} := u^{(i)}\!\bigl(f^{\star(i)}\bigr)\in[0,1]$, where
$U^\star:=\sup_i U^{(i)}$.\\
\textbf{Resolution.}
With a finite $T$, the optimization gap is nonzero $\epsf_T > 0$. Combining Lipschitzness with \eqref{eq:epsf_def_main} gives, $\forall i, s$
$
0\le U^{(i)}-u^{(i)}(\overline{y}_s^{(i)}) \le L^{(i)}\bigl(f^{\star(i)}-\overline{y}_s^{(i)}\bigr)\le L^\star\,\epsf_s,
$
Consequently, even if all budget $T$ is allocated to a single task, $L^\star \epsf_T$ is the finest utility accuracy that can be reliably achieved. We therefore define an \emph{achievable} utility tolerance (the “resolution”).

\begin{definition}[$\epsilon^U$-optimal tasks]\label{def:gamma_oracle}
For any $\epsilon^U>0$, define
$\mathcal I^\star(\epsilon^U) := \bigl\{i\in\mathbb N:\ U^\star - U^{(i)} \le \epsilon^U\bigr\}$.
\end{definition}
\vspace{-0.5em}
\begin{condition}[BO-comparable target resolution]\label{cond:bo_achievable}
Fix $\epsUm :=\epsU_0 2^{-m}$ for $m \in \mathbb{Z}_{\geq0}$.
Let $\Lambda_T:=\log(eT)$.
Choose a max depth $\overline{m}_T\in\mathbb{Z}_{\ge 0}$ and assume there exists an $m^\star \in \{0,1,\dots,\overline{m}_T\}$ and a constant $c_\epsilon>0$ such that
\[
\epsU_{m^\star}\ \le\ c_\epsilon\,\Lambda_T\,L^\star\,\epsfT_T .
\]
\end{condition}
\vspace{-0.5em}
This ensures that the $\epsU_{m^\star}$ is \emph{achievable} within $T$\footnote{up to the logarithmic factor needed by the
coarse-to-fine reachability schedule in Lemma~\ref{lem:max_depth_reachability}}; otherwise, any finer resolution $\epsU_{\overline{m}_T}<\epsU_{m^\star}$ would exceed what BO can reliably achieve. 
As $T$ increases, $\epsfT_T$ shrinks and $\epsU_{m^\star}$ becomes finer.\\
\textbf{Goal}
is to find and optimize any $\epsU_{m^\star}$-optimal task and we employ the task regret as the metric:
\begin{equation}
\label{prob:goal}
i^\star_{\epsU} \in \mathcal I^\star(\epsU_{m^\star}),
\qquad 
x^{\star(i^\star_{\epsU})} \in \argmax_{x\in\mathcal X^{(i^\star_{\epsU})}} f^{(i^\star_{\epsU})}(x),
\qquad
\Regval_T := \sum_{t=1}^T \Bigl(U^\star - u^{(i_t)}(\overline{y}^{(i_t)}_t)\Bigr).
\end{equation}
\vspace{-1.5em}

\vspace{-0em}
\section{Generate-Select-Refine (GSR)}
\begin{wrapfigure}[12]{r}{0.54\textwidth}
\centering
\vspace{-4em}
\begin{minipage}{0.53\textwidth}
\captionsetup{type=algorithm,aboveskip=2pt,belowskip=2pt}
\hrule
\caption{{\color{teal}Generate}-{\color{purple}Selection}-{\color{blue}Refine} (GSR)}
\hrule
\label{alg:task-ucb}
\raggedright
\begin{algorithmic}[1]
\STATE \textbf{Input:} confidence $(\delta_f,\delta_u,\delta_-)$, Lipschitz bound $\Lbar$,
seed task $i_0$, batch size $J$, $T$, $c_g$, $\overline{m}_T$.
\STATE \textbf{Init:} $m \gets 0$, $\epsU_0 \gets 1$, $\epsU_m \gets \epsU_0$, $\mathcal I_1 \gets \{i_0\}$,
{\color{teal}$\mathcal B_0 \gets \Gen(i_0,0,J)$, $\mathcal I_1 \gets \mathcal I_1 \cup \mathcal B_0$}.
\FOR{$t \in [T]$}
    {\color{purple}\STATE select the next task $i_t \in \argmax_{j\in\mathcal I_t}\UCBV^{(j)}_t$
    \label{alg_line:meta_ucb_select}}
    \STATE $x_t^{(i_t)}, y_t^{(i_t)}, \overline{y}^{(i_t)}_t \gets$ RunOneStepGP-UCB($i_t$).
    \STATE \vspace{-2.3ex} update envelopes for task $i_t$ by Theorem~\ref{thm:value_envelopes_summary}.
    {\color{purple}\STATE select $a_t$, $w_t$ by Eq.~\eqref{eq:anchor}, and $\mathcal I_{t+1} \gets \mathcal I_t$.}
    \IF{$m<\overline{m}_T$ \textbf{and} {\color{blue}$w_t\le c_{g}\epsU_m$}}
    \label{alg_line:scheduler}
        {\color{blue}\STATE step up $m\gets m+1$; $\epsU_m\gets \epsU_0 2^{-m}$.}
        {\color{teal}\STATE draw $\mathcal B_m\gets \Gen(a_t,m,J)$; $\mathcal I_{t+1}\gets \mathcal I_{t+1}\cup \mathcal B_m$.}
    \ENDIF
\ENDFOR
\end{algorithmic}
\hrule
\end{minipage}
\vspace{-0.5em}
\end{wrapfigure}
We propose GSR (Alg.~\ref{alg:task-ucb}), a open-ended-BO framework for \emph{online task discovery} (generate), \emph{budget allocation} (select), and \emph{coarse-to-fine scheduler} (refine). This section is intentionally \emph{algorithm-first}: we define the quantities GSR maintains and explain the control flow.  Concrete LLM instantiations are given in §\ref{sec:llm}. All formal guarantees are collected in §\ref{sec:theory}.

\vspace{-0.3em}
\subsection{Two loops: GP-UCB and task-UCB}
Each task $i$ runs a standard GP-based BO procedure (GP-UCB), maintaining:
(i) an incumbent $\overline{y}_s^{(i)}$ after $s$ local evaluations, and
(ii) a computable optimization gap $\epsf_s$ from Eq.\eqref{eq:epsf_def_main}.
The task-optimizer decides (a) which task to evaluate next (select) and (b) when to generate tasks (refine).

\textbf{Task Selection.}
Core idea is the confidence intervals (CIs):
\begin{gather*}
    U^{(i)} \in [\underline{U}^{(i)}_s, \overline{U}^{(i)}_s],
    \qquad
    \overline{U}_s^{(i)} = \underbrace{\overline{u}_s^{(i)}}_\text{Utility UCB} + \underbrace{\overline{L}\epsf_s}_\text{optimization gap},
\end{gather*}
where $\overline{L} \geq L^\star$ bounds the Lipschitz constant. The UCB naturally decomposes into uncertainty from finite utility samples and an optimization gap due to limited evaluation budget. We select the next task $i_t$ via a task-UCB policy (Line~\ref{alg_line:meta_ucb_select} in Alg.~\ref{alg:task-ucb}; see Fig.\ref{fig:task-ucb}(a))\footnote{Ties are broken in favor of the task with fewer evaluations.}, concentrating budget on the 
$\epsU$-optimal task while early-stopping suboptimal ones.

\begin{wrapfigure}[14]{r}{0.45\textwidth}
    \vspace{-1.5em}
    \centering
    \includegraphics[width=0.44\textwidth]{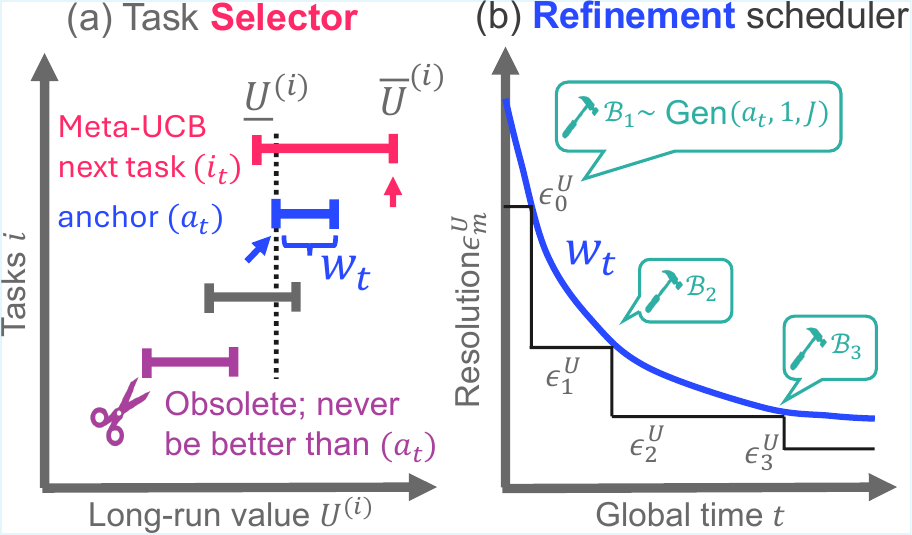}
    \caption{Alg.~\ref{alg:task-ucb}:
    (a) {\color{purple}\textbf{Select}} chooses the next task $i_t$ and anchor task $a_t$. 
    (b) {\color{blue}\textbf{Refine}} advances the resolution levels $\epsU_m$ and {\color{teal}generates} $J$ mutations when anchor width $w_t \leq c_g\epsU_m$.
    }
    \label{fig:task-ucb}
\end{wrapfigure}
\textbf{Task Refinement.}
We generate new tasks by mutating and iteratively refining a seed task. A key step is selecting an \emph{anchor} task to mutate. The anchor must be both promising and sufficiently resolved. We therefore choose $a_t$ as the most promising pessimistic task—ranked by $\LCBV^{(i)}_t$ rather than $\UCBV^{(i)}_t$—whose uncertainty is below a target threshold:
\begin{equation}
\label{eq:anchor}
    a_t \in \argmax_{j\in\mathcal I_t}\LCBV^{(j)}_t \,\,\text{s.t.}\,\, w^{(j)}_t \le \overline{\epsilon}^U_t,
\end{equation}
where $w_t^{(i)} := \overline{U}^{(i)}_t - \underline{U}^{(i)}_t$ is the CI width (uncertainty), and $\overline{\epsilon}^U_t := \max\{c_g \epsU_m,\ \underline{w}_t\}$ is the resolution threshold, with $\underline{w}_t := \min_{j\in\mathcal I_t} w^{(j)}_t$. 
See Fig.\ref{fig:task-ucb}(a) for intuition.
This rule selects an anchor that is promising under $\LCBV$ and sufficiently certain to advance to the next level $m$ (cf. Condition~\ref{cond:bo_achievable}).

\textbf{Generation scheduler.}
We generate new tasks only once the current anchor is sufficiently certain (Line~\ref{alg_line:scheduler}; Fig.\ref{fig:task-ucb}(b)). We then advance to $m{+}1$ and generate a new batch of child tasks at a finer mutation scale $\epsU_m$. See Figure~\ref{fig:scheduler} for ablation study on generation schedulers.

\subsection{LLM-powered openended-BO}\label{sec:llm}
\begin{figure*}
    \centering
    \includegraphics[width=0.9\linewidth]{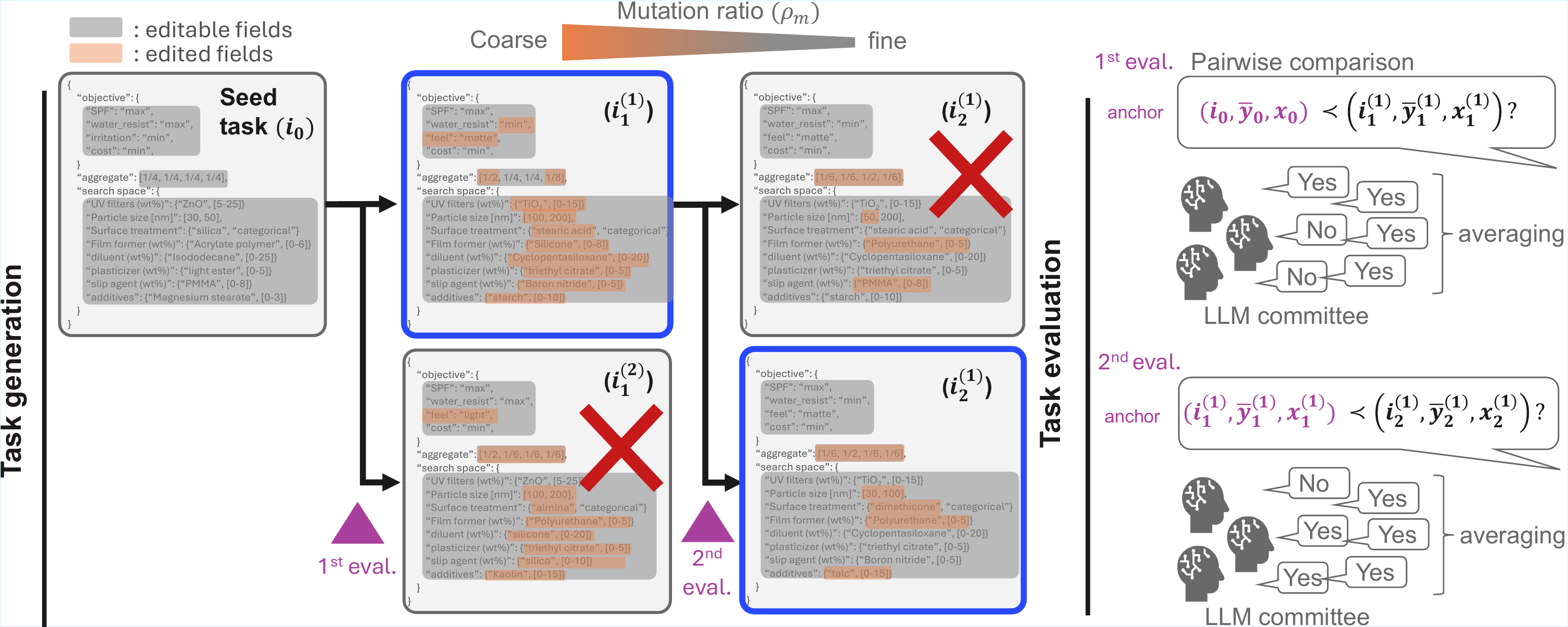}
    \caption{LLM components in GSR: (\textbf{Left}) Coarse-to-fine task generation: we mutate a parent task with target mutation ratio $\rho_m$, producing refined tasks. (\textbf{Right}) Committee-based evaluation: committee returns binary votes if the current task--incumbent pair $(i,\overline{y}_s^{(i)})$ is preferred over an anchor $(a_t, \overline{y}_s^{(a_t)})$. By averaging them, we estimate a confidence interval for the utility $u^{(i)}(\overline{y}_s^{(i)})$.}
    \label{fig:llm}
    \vspace{-1em}
\end{figure*}
We adopt empirically successful paradigms (Fig.\ref{fig:llm}): evolutionary LLMs for task generation \citep{novikov2025alphaevolve, lange2025shinkaevolve}.
We present only the essential details here and refer readers to Appendix~\ref{app:evaluator} for background.

\vspace{-0.5em}
\subsection{Generate tasks by controlled mutations}\label{sec:openended_lb}
We represent each task as a structured specification (e.g., JSON) with editable fields (e.g., objective/domain).
Given an anchor task $a_t$ and level $m$, the generator $\Gen(a_t,m,J)$ prompts an LLM to propose mutated child tasks $\mathcal{B}_m$.

\textbf{Mutation ratio.}
To operationalize the resolution levels $\epsU_m$, we control the mutation ratio:
\[
\rho(i,a_t)
:= \frac{\#\text{modified editable fields}}{\#\text{editable fields}}
\in[0,1].
\]
We target a decreasing schedule $\rho_m := \rho_0 2^{-m}$, so higher $m$ corresponds to smaller edits.

\textbf{Scope of the empirical generator.}
The theory treats $\Gen$ as an abstract open-ended generator. In the experiments, however, we use domain-specific JSON schemas and mutate editable fields. Therefore the present empirical instantiation tests \emph{schema-constrained} task discovery: it can discover new task combinations and refinements within the supplied schema, but it is not claimed to invent arbitrary new functional forms unless those structural degrees of freedom are included in the schema.

\vspace{-0.3em}
\subsection{Utility elicitation via an LLM committee}\label{sec:evaluator} 
When tasks admit an objective comparison, that is ideal. For inverse optimization, for example, we can measure competitiveness by comparing the performance gap between selected method against the best-known baseline. When such comparison is difficult, we can \textbf{optionally} use preference-based LLM judges \citep{rafailov2023direct, fujisawa2025scalable}. Recent LLM-driven discovery work \citep{agarwal2026autodiscovery, bradley2024qualitydiversity, zhang2024omni, yang2026verbalizing} suggests that LLM-based evaluation can support meaningful discovery, and in our experiments LLM judgments correlate strongly with human judgments (Appendix~\ref{app:alignment}).
Following standard LLM-as-a-judge practice \citep{chen2024humans, lee2025correctly}, an LLM committee compares task pairs and returns binary votes (Fig.\ref{fig:llm}). We adopt the Bradley–Terry model \citep{bradley1952rank} to infer latent utilities from these comparisons, defining the utility $u^{(i)}(\cdot)$ as the win rate of task $i_t$ against the seed task $i_0$, i.e., $\Pr(i_t \succ i_0 \mid \overline{y}_t^{(i_t)}, \overline{y}_0^{(i_0)}, x^{(i_t)}_t, x^{(i_0)}_0)$. As both the task $i_t$ and its solution $\overline{y}_t^{(i_t)}$ improve, this win rate provides a consistent cross-task utility.
To prevent saturation when $i_t$ is almost always preferred, we instead use an anchor task $a_t$ as an adaptive reference and recover the win rate by chaining pairwise comparisons (e.g., $i_t \succ a_t \succ i_0$).
Averaging $K_t$ preference votes yields a single utility estimate $\tilde{u}_t \in [0,1]$ with $1/(4K_t)$-sub-Gaussian noise, thus it satisfies Definition~\ref{ass:mono_lip} (see Appendix~\ref{app:evaluator}).

\section{Theoretical Analysis}\label{sec:theory}
\subsection{Main regret bound}
\textbf{Confidence envelopes.}
We now establish the regret bound for GSR. As a first step, we derive CIs for the task-UCB; see Appendix~\ref{app:sparse_utility} for the proofs.
\begin{lemma}[Utility CIs]
\label{lem:utility_ci_hoeffding_main}
Let $\delta_t:=\delta_u/(\pi^2 t^2)$ and $\phi_t:=\sqrt{2\sigma_u^2\log(2/\delta_t)}$.
Then w.p.$\geq 1-\delta_u$, for all $t\in[T]$,
\vspace{-0.3em}
\begin{gather*}
    \LCBu_t^{(i_t)} := \text{clip}_{[0,1]}(\tilde{u}^{(i_t)}_t-\phi_t),
    \,\,
    \UCBu_t^{(i_t)} := \text{clip}_{[0,1]}(\tilde{u}^{(i_t)}_t + \phi_t),
\end{gather*}
\end{lemma}
\vspace{-0.3em}
\begin{remark}[Averaged feedback]
Averaging $K_t\ge 1$ independent samples to form $\tilde{u}_t$ is $(\sigma_u^2/K_t)$-sub-Gaussian.
\end{remark}

\begin{theorem}[Value envelopes]
\label{thm:value_envelopes_summary}
Under Lemma~\ref{lem:utility_ci_hoeffding_main} and the GP-UCB events \eqref{eq:epsf_def_main},
if $\Lbar\ge L^\star$, then for every instantiated task $i\in\mathcal I_t$ with local counter $s_t^{(i)}$,
\[
\LCBV^{(i)}_t := \LCBu^{(i)}_t,
\qquad
\UCBV^{(i)}_t
:= \clip{\UCBu^{(i)}_t + \Lbar\,\epsf_{s_t^{(i)}}}{0}{1},
\]
with initialization $[\LCBV^{(i)}_0,\UCBV^{(i)}_0]=[0,1]$. If the utility interval width is controlled as
$\UCBu^{(i)}_t-\LCBu^{(i)}_t \le c_u\,\Lbar\,\epsf_{s_t^{(i)}}$ for $c_u > 0$,
then $w^{(i)}_t :=\UCBV^{(i)}_t-\LCBV^{(i)}_t
\le (c_u+1)\Lbar\,\epsf_{s_t^{(i)}}$.
\end{theorem}
\vspace{-0.5em}
\begin{condition}[Refinement probability]
\label{cond:stoch_coverage_adaptive}
Under Condition~\ref{cond:bo_achievable}, there exists a nonzero probability $\delta_+\in(0,1]$ such that $\forall m, s, a_t$ that is resolved finer than $\epsU_m$, a single draw $i\sim\Gen(a_t,m,J=1)$ hits an $\epsU_m$-optimal task for $m \geq 1$:
$\Pr\!\Bigl(i\in \mathcal I^\star(\epsU_{m}) \ \Big|\ w_t \le c_g \epsU_{m-1} \Bigr)
\ \ge\ \delta_+$.
\end{condition}
\vspace{-0.5em}
This condition is enforced in Line~\ref{alg_line:scheduler} of Alg.~\ref{alg:task-ucb}. Conditioning on a sufficiently solved anchor enables progressively finer task proposals, down to resolution $\epsU_m \lesssim L^\star \epsfT_T$.

\textbf{Max depth $\overline{m}_T$} For brevity, set $\Psi_T := \max_{i\in\mathcal I_T}\beta_T^{(i)}\gamma_T^{(i)}$. Then, $\overline{m}_T$ is set as:
\begin{lemma}[Reachability]\label{lem:max_depth_reachability}
    Under Theorem~\ref{thm:value_envelopes_summary} and Condition~\ref{cond:stoch_coverage_adaptive}, assume the batch size $J$ is a constant. Set
    \[
    \overline{m}_T := \left\lfloor \frac{1}{2} \log_2 \left( \frac{T}{A_T (\log(e T))^2} \right) \right\rfloor_+,
    \]
    where $A_T:=\left(\frac{2C_v\,\Lbar\sqrt{C_\lambda\,\Psi_T}}{c_g\,\epsU_0}\right)^2$, and $\lfloor x \rfloor_+ = \max\{0, \lfloor x \rfloor\}$. 
    Then, under Condition~\ref{cond:bo_achievable} and  Algorithm~\ref{alg:task-ucb}, the optimal level $m^\star$ is introduced by time $T$.
\end{lemma}

\vspace{-0.3em}
\textbf{Task regret bound.}
For brevity, we set $\beta_T := \max_{i\in\mathcal I_T}\beta_T^{(i)}$ and $\gamma_T := \max_{i\in\mathcal I_T}\gamma_T^{(i)}$.
\begin{theorem}[Regret bound]
\label{thm:main_regret_adaptive}
Under Theorem~\ref{thm:value_envelopes_summary} and Lemma~\ref{lem:max_depth_reachability}, let
$R^\star_T := L^\star\sum_{s=1}^T \epsfT_s$
be the uniform single-task regret bound. Let $N_T:=|\mathcal I_T|$ be the number of instantiated tasks by time $T$ and let $\Lambda_T:=\log(eT)$.
Then Algorithm~\ref{alg:task-ucb} satisfies, with probability at least
$1-(\delta_{f}+\delta_u+\delta_-)$,
\begin{equation}
\label{eq:main_regret_adaptive}
\Regval_T
\ \le\
C\,\left[
\Lambda_T
+
\Big(1+\sqrt{N_T}\Big)\,\frac{\overline{L}}{L^\star}
\right] R^\star_T,
\end{equation}
for $C>0$.
Since $N_T \le 1 + J\,(\overline{m}_T+1)$
and $\overline{m}_T=\tilde{\mathcal O}(\log T)$ for constant $J$, we obtain a logarithmic up to overhead relative to the single-task BO when $\overline L=\Theta(L^\star)$.
\end{theorem}
\vspace{-1em}
\begin{sketch}
Let $U_t^{\max}:=\max_{i\in\mathcal I_t}U^{(i)}$. On the high-probability envelope and generation-success events,
the instantaneous regret can be decomposed to:
\begin{align*}
    R^\star_T \leq \underbrace{(U^\star-U_t^{\max})}_{\text{generation}} + \underbrace{(U_t^{\max}-U^{(i_t)})}_{\text{selection}}+ \underbrace{\bigl(U^{(i_t)}-u^{(i_t)}(\overline{y}^{(i_t)}_t)\bigr)}_{\text{within-task}}
\end{align*}
\begin{compactenum}[(a)]
    \item \emph{Generation:} $U^\star\!-\!U_t^{\max}\le \epsU_{m_t\wedge m^\star}$, so the sum is
    $\mathcal O(\Lambda_T R^\star_T)$ under Condition~\ref{cond:bo_achievable}.
    \item \emph{Selection:} task-UCB implies $U_t^{\max}\!-\!U^{(i_t)}\le w^{(i_t)}_{t-1} \lesssim \Lbar\,\epsf_{s_{t-1}^{(i_t)}}$, giving $\tilde{\mathcal O}((\Lbar/L^\star)\sqrt{N_T}\,R^\star_T)$.
    \item \emph{Within-task:} Lipschitzness + GP-UCB bound give $U^{(i)}-u^{(i)}(\overline{y}_s^{(i)})\le L^\star\epsf_s$, giving $\tilde{\mathcal O}(\sqrt{N_T}\,R^\star_T)$.
\end{compactenum}
Combining yields \eqref{eq:main_regret_adaptive}; see Appendix~\ref{proof:max_depth_reachability} and \ref{app:proof_main_regret_adaptive} for the proof. 
\end{sketch}
\vspace{-1em}
\subsection{Discussion of the theoretical results}
\textbf{Q. What if we use an oracle task as the comparator?}
Let $\underline{R}_T := L^\star \sum_{s=1}^T \epsilon^{f,(i^\star_{\epsU})}_s$. For RBF, the extra overhead is
$\kappa_T := R^\star_T/\underline{R}_T = \tilde{\mathcal O}\!\big((\log T)^{\overline{d}-d^{(i^\star_{\epsU})}}\big)$,
where $\overline{d}:=\max_{i\in\mathcal I_T} d^{(i)}$; see Lemma~\ref{lem:kappa_rbf_matern}.

\textbf{Q. Isn’t $\Lbar$ hard to estimate?} We adapt $\Lbar$ online via regret-balancing~\citep{ziomek2024bayesian}; Alg.~\ref{alg:GSR} in Appendix~\ref{app:hyperparameters}, which incurs only a logarithmic overhead; the same idea applies to other hyperparameters. In experiments, GP hyperparameters are fit by marginal likelihood, while remaining constants are fixed ($\epsU_0=1$, $c_g=0.5$; see Appendix~\ref{app:hyperparameters}). See Figure~\ref{fig:scheduler} for ablation study on $\Lbar$ estimators.

\textbf{Q. Isn’t Condition~\ref{cond:stoch_coverage_adaptive} strong?}
This may be considered strong, but without it, the first term in our regret decomposition is irreducible, making sublinear task-regret information-theoretically impossible.
Our assumption makes explicit what is often left implicit in practice: evolutionary LLMs already show strong empirical performance in code and proof generation \citep{novikov2025alphaevolve, lange2025shinkaevolve, assumpccao2025codeevolve}.
Appendix~\ref{app:ICL} further supports this through a Bayesian in-context learning analysis \citep{wu2024how, wakayama2025context}, showing that, conditioned on mutation history, the posterior mass on the correct refinement type concentrates exponentially fast with the number of in-context examples. In particular, Corollary~\ref{cor:delta_large_k_main_text_form} gives
$\delta_+ \gtrsim \rho_m(\epsU_m) -\epsilon_{\text{ICL}}$,
where $\rho_m(\epsU_m)$ is the problem-specific, LLM-independent success probability of the true refinement, and $\epsilon_{\text{ICL}}$ is the Bayes approximation error, controlled by pretraining scale and typically small. This implies a nonzero success probability for \emph{refinable} problems with \emph{informative} mutation histories. Finally, \S\ref{sec:ablation} empirically confirms that $\delta_+$ remains nonzero across levels $m$.

\vspace{-0.3em}
\section{Experiments}\label{sec:experiments}
\vspace{-0.3em}
\begin{figure*}
    \centering
    \includegraphics[width=1\linewidth]{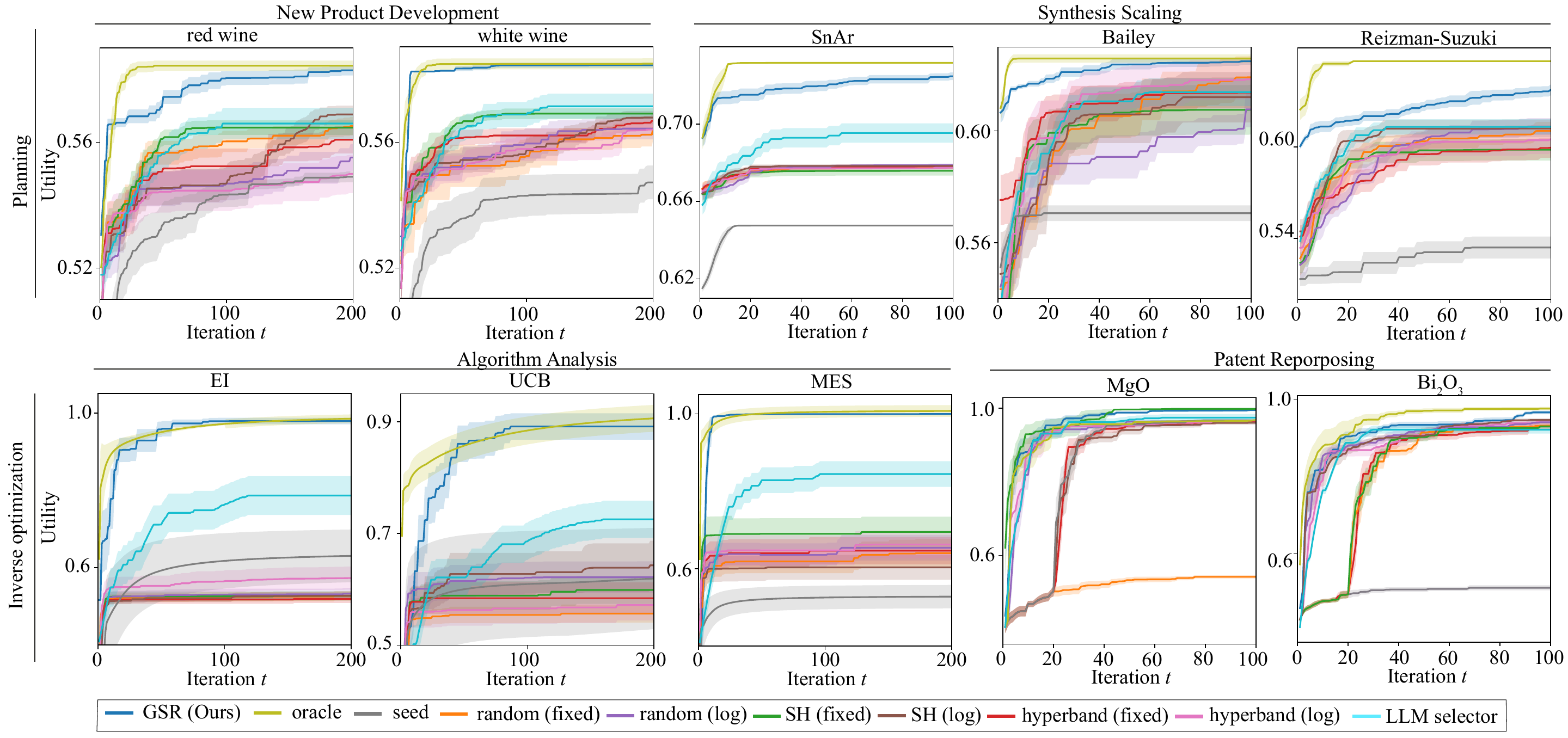}
    \caption{Real-world experiments. Top row: planning tasks—(a) new product development and (b) synthesis scaling. Bottom row: inverse optimization—(c) objective selection. Solid lines and shaded regions indicate the mean ± one standard error over runs. GSR consistently outperforms the baselines.
    }
    \label{fig:real-world}
    \vspace{-1em}
\end{figure*}
We evaluate GSR across diverse settings. Direct comparison to prior methods is nontrivial due to GSR’s open-ended and partial observation nature (see $\S$\ref{subsec:related}).
We therefore adapt ShinkaEvolve~\citep{lange2025shinkaevolve} as a baseline, combining four task selectors—random, successive halving (SH; \citep{jamieson2016non}), and Hyperband~\citep{li2018hyperband}, and LLM selector~\citep{ferreira2026can} that LLM selects tasks and schedule task generation—with two generation schedulers (fixed and log). We additionally include \emph{seed} (single-task BO on the initial task) and \emph{oracle} (single-task BO on the best task in hindsight). The gap to \emph{seed} quantifies improvement over the initial task, while the gap to \emph{oracle} captures suboptimality in task generation and selection.
All methods are implemented in BoTorch/GPyTorch~\citep{balandat2020botorch,gardner2018gpytorch}. In LLM experiments, \textsc{GPT-4o-mini}~\citep{achiam2023gpt} is used only for task generation and committee votes $\tilde{u}_t^{(i)}$; objective evaluations $y_t^{(i)}$ are from the simulator. Full details appear in Appendix~\ref{app:experiments} and the anonymous code.\footnote{\url{https://anonymous.4open.science/r/Generate-Select-Refine-2D48}}.

\vspace{-0.3em}
\subsection{Online Experiment 1: Planning}
\begin{wrapfigure}[19]{r}{0.32\textwidth}
    \vspace{-3.5em}
    \centering
    \includegraphics[width=0.27\textwidth]{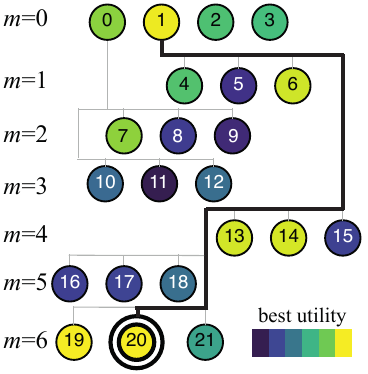}
    \caption{GSR’s task evolution tree shows coarse-to-fine refinement across levels $m$ toward the best task (20). Colors indicate final utility, and evaluation counts vary across tasks; task (1) initially lags behind its children (4–6), improves with more BO evaluations, and is ultimately surpassed by task (20).
    }
    \label{fig:evo_tree}
\end{wrapfigure}
We first apply GSR to \emph{planning}, where one must decide both \emph{what} to optimize and \emph{how} to evaluate it, starting from a vague seed task that is progressively refined (Fig.\ref{fig:real-world}).

\textbf{New product development.}
We consider a product manager at a local winery designing a new wine to attract new customers. This requires jointly identifying a target customer segment (task $i^\star_{\epsU}$) and its optimal wine (solution $x^{(i^\star_{\epsU})}$). 
We use the Wine Quality dataset~\citep{cortez2009modeling}, where inputs are physicochemical properties (e.g., acidity, sugar, alcohol) and outputs are quality scores. Each task objective combines (i) quality maximization and (ii) taste–intent matching, where natural-language descriptions (e.g., fruity, bold) are instantiated as numerical targets $x_\text{target}$ via LLMs, and matching is measured by the Mahalanobis distance between $x_\text{target}$ and $x_t$.
Cross-task utility is defined by majority vote from an LLM committee. We generate diverse consumer personas (e.g., engineers, event hosts) with latent preferences stored as JSON and treat them as new potential customers. For each comparison $(i_t, a_t)$, the committee votes using fixed personas from the pool, which remain hidden from the task generator. An evolutionary LLM then mutates the current best task (Fig.\ref{fig:evo_tree}), generating $J=3$ new tasks per resolution $\epsUm$ by refining taste intent, objective weights (quality vs. taste matching), and the search space.
Across red and white wines, GSR achieves near-oracle performance and discovers persona-aligned styles (e.g., bold\&crisp reds and fruity\&smooth whites). See Appendix~\ref{app:experiments} for full details, including prompts.

\textbf{Synthesis scaling.}
Synthesis routes that perform well in the lab often fail to scale to industrial manufacturing; small-scale trials therefore act as a \emph{planning} phase before committing to costly plant construction.
We use the \textsc{SUMMIT} chemical process simulator~\citep{felton2021summit} with three reactions. Each task is a reaction plan that trades off throughput and sustainability, and cross-task utility is the majority vote of a stakeholder committee (e.g., plant engineers, safety reviewers).
The discovered plan is more conservative than the throughput-optimal solution, reflecting committee preferences for sustainability and safety (see also Appendix~\ref{app:experiments}).

\vspace{-0.5em}
\subsection{Online Experiment 2: Inverse Optimization}\label{sec:io}
\vspace{-0.5em}
\begin{wraptable}[6]{r}{0.5\textwidth}
\centering
\vspace{-1.2em}
\caption{Discovered tasks in the algorithm analysis}
    \label{tab:algo}
    \resizebox{0.48\textwidth}{!}{%
    \begin{tabular}{lllll}
    \toprule
         policy
         & objective
         & dim. $d$
         & bounds
         & noise $\sigma_\text{noise}$\\
         \midrule
         \textsc{EI} & 
 Styblinski-Tang & 5 & $[-5,5]$ & 0.01\\
 \textsc{UCB} 
 & Levy & 3 & $[0,1]$ & 0.1\\
 \textsc{MES}
 & Levy &2 & $[-0.03,3.78]$ & 0.01\\
 \bottomrule
 \end{tabular}
 }
\end{wraptable}
Next, we apply GSR to \emph{inverse optimization} (IO), which seeks a task under which a given solution $x$ is optimal, rather than optimizing $x$ for a fixed task. While classical IO focuses on exact recovery in linear programming~\citep{chan2025inverse}, we extend IO to black-box, free-form task descriptions using BO and LLMs, and solve it approximately by allowing mild refinement of $x$.
Starting from a seed task $i_0$ and solution $x_0$, GSR jointly evolves both until the refined solution $x^\star$ is optimal for the discovered task $i^\star_{\epsU}$.

\textbf{Algorithm analysis.}
We apply IO to acquisition function analysis (EI~\citep{mockus1978application}, GP-UCB~\citep{srinivas2009gaussian}, and MES~\citep{wang2017max}). Which acquisition function performs best depends strongly on the underlying problem setting. 
We study the inverse question: \emph{under what task does each acquisition function perform best?} 
We define utility as the performance gap between the targeted acquisition function and the best-performing alternative, passed through a sigmoid to lie in $[0,1]$. Each task corresponds to standard test functions available in BoTorch, and an evolutionary LLM generates new tasks by mutating problem settings.
The discovered tasks (Table~\ref{tab:algo}) match known behavior: EI performs well on higher-dimensional smooth functions; UCB is more robust under noisy feedback; and MES excels on functions with many low-noise local optima. These findings are consistent with established results: EI is more exploitative~\citep{shahriari2015taking}, GP-UCB admits formal regret guarantees under noisy observations~\citep{srinivas2012information}, and information-theoretic approaches are effective for multimodal objectives \cite{wang2017max}.

\begin{wrapfigure}[23]{r}{0.4\textwidth}
    \vspace{-1em}
    \centering
    \includegraphics[width=0.4\textwidth]{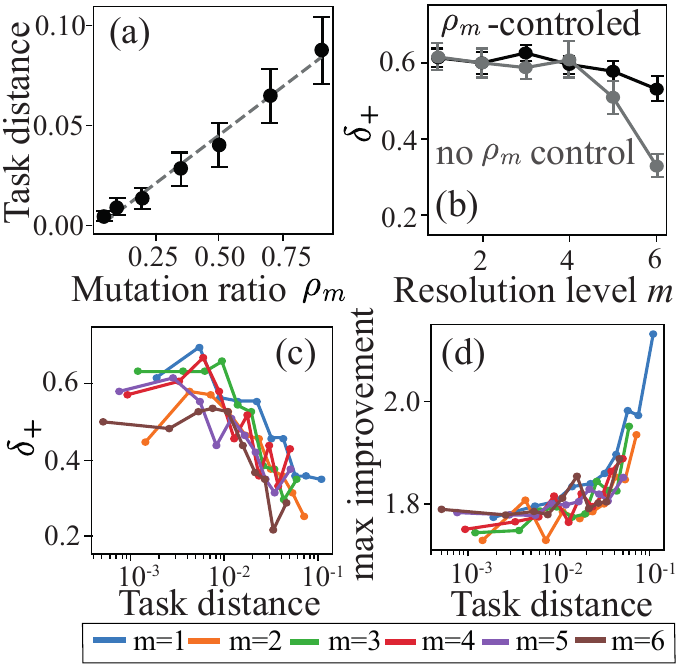}
    \caption{Resolution controllability. (a) Mutation ratio $\rho_m$ correlates with task distance. (b) finer task generation probability $\delta_+$ is non-zero across all resolution levels $m$. (c) Finer mutations yield higher $\delta_+$. (d) Coarser mutations achieve larger improvement. 
    }
    \label{fig:controllability}
\end{wrapfigure}
\textbf{Patent Repurposing.}
In industrial R\&D, materials developed for a target application are often abandoned despite substantial sunk costs; repurposing aims to identify alternative applications where such materials can perform well. We use the \textsc{Materials Project} dataset~\citep{jain2013commentary} and construct input–output pairs, where inputs are crystal structure descriptors encoded by CGCNN~\citep{xie2018crystal} and outputs consist of six material properties.
Each task optimizes two components: (i) surprisal~\citep{levine2009molecular}, defined as the negative log-probability of material properties, computed by fitting a kernel density estimator to each simulated property and linearly aggregating their surprisals as a feedback oracle; and (ii) a regularizer measuring the Euclidean distance between the seed material $x_0$ and the selected material $x^{(i_t)}_t$. 
An evolutionary LLM refines tasks by mutating surprisal aggregation weights, regularization weights, and the search space, while an LLM committee evaluates scientific and product value via pairwise comparisons.
Using MgO and Bi$_2$O$_3$ as seeds, GSR identifies plausible additives—TmCl$_3$ and HgF$_2$ for MgO (shear modulus, band gap), and CuO and PrTe$_2$ for Bi$_2$O$_3$ (Poisson ratio, Fermi energy). Based on these candidates, LLMs further propose repurposed applications, including more durable MRAM~\citep{parkin2004giant}, inorganic photolithography layers, and ion-conducting or active-material coatings for solid oxide fuel cells based on Bi$_2$O$_3$. We find no prior reports for these hypotheses in a targeted keyword search.
Overall, the results illustrate \emph{data-supported hypothesis generation} for materials repurposing, rather than hypotheses derived solely from LLMs.

\subsection{Experimental Analysis, Ablation Study}\label{sec:ablation}
\textbf{Human-LLM evaluation alignment.}
LLM-as-a-judge is powerful, but it risks misalignment with human values. To assess this, we compared its judgments against those of two PhD-level materials scientists on 100 randomly sampled task pairs. Both the human annotators and the LLM were given exactly the same task descriptions and partial optimization histories, and the annotators were blinded to the committee votes. Agreement, measured by Cohen’s kappa, was high: 0.905 for synthesis and 0.937 for repurposing on average. See Appendix~\ref{app:alignment} for the details.

\begin{wrapfigure}[10]{r}{0.4\textwidth}
    \vspace{-1.5em}
    \centering
    \includegraphics[width=0.4\textwidth]{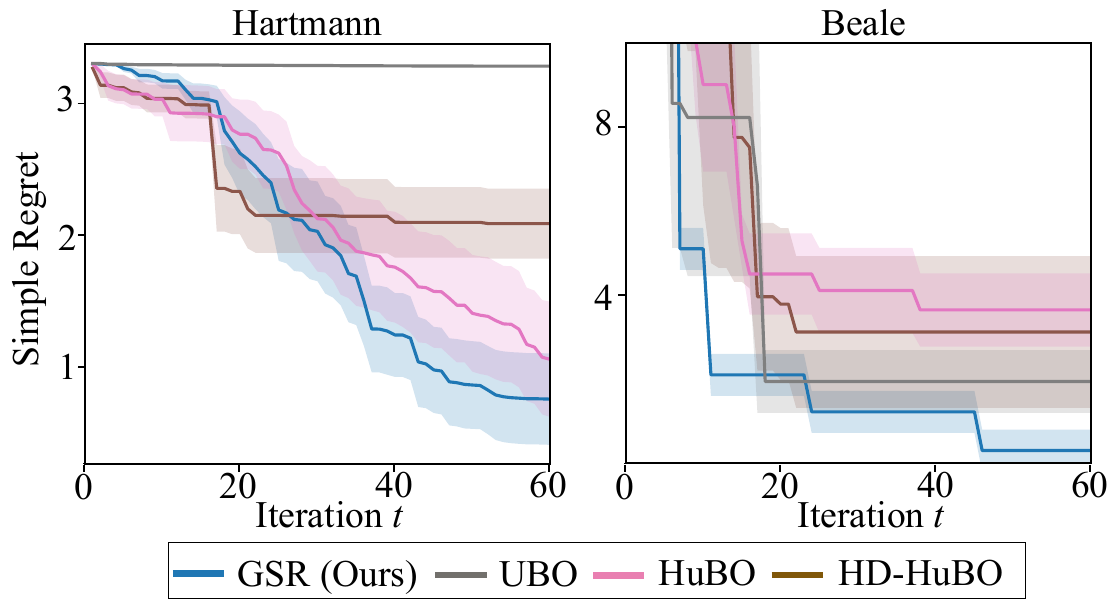}
    \caption{Unknown search space. 
    }
    \label{fig:non-llm}
\end{wrapfigure}
\textbf{Resolution controllability.}
We analyze the LLM experiments using the white wine task. GSR uses mutation to control task resolution. We generate $J=40$ child tasks for 4 anchor tasks across mutation ratios $\rho_m$, resolution levels $m$, and random seeds (Fig.\ref{fig:controllability}).
(a) Task distance—Euclidean distance over numerical task parameters—correlates with $\rho_m$, confirming that JSON mutation controls the effective resolution $\epsU_m$.
(b) The success probability $\delta_+$ (Condition~\ref{cond:stoch_coverage_adaptive}) remains non-zero across all $m$. Without controlling $\rho_m$, $\delta_+$ decreases at higher $m$ (grey), while a coarse-to-fine $\rho_m$ schedule stabilizes $\delta_+$ (black).
(c) $\delta_+$ increases as task distance decreases, further motivating smaller $\rho_m$ at finer resolutions.
(d) Coarser mutations (higher $\rho_m$, lower $m$) yield larger improvements, supporting an exploration-to-exploitation strategy: explore early with larger $\rho_m$, then gradually reduce $\rho_m$ to exploit.

\textbf{Non-LLM task generator.}
Next, we evaluate GSR with a non-LLM generator in the \emph{unknown search space} setting~\cite{ha2019bayesian}, where the true optimum $x^\star$ lies outside the initial domain $\mathcal{X}_0$. The goal is to expand candidate domains to include $x^\star$ while avoiding unnecessary expansion.
This is a single-objective, so tasks (domains) are ranked directly by the objective $f(x)$ without utility. We instantiate the generator by doubling the domain at each resolution level $m$ and compare against UBO~\citep{ha2019bayesian}, HuBO, and HD-HuBO~\citep{gupta2020sub} on two test functions, measuring simple regret $f^\star-\overline{y}_t^{(i_t)}$. Despite being specialized for this setting, these baselines are outperformed by GSR (Fig.~\ref{fig:non-llm}).

\begin{wrapfigure}[10]{r}{0.4\textwidth}
    \vspace{-1em}
    \centering
    \includegraphics[width=0.4\textwidth]{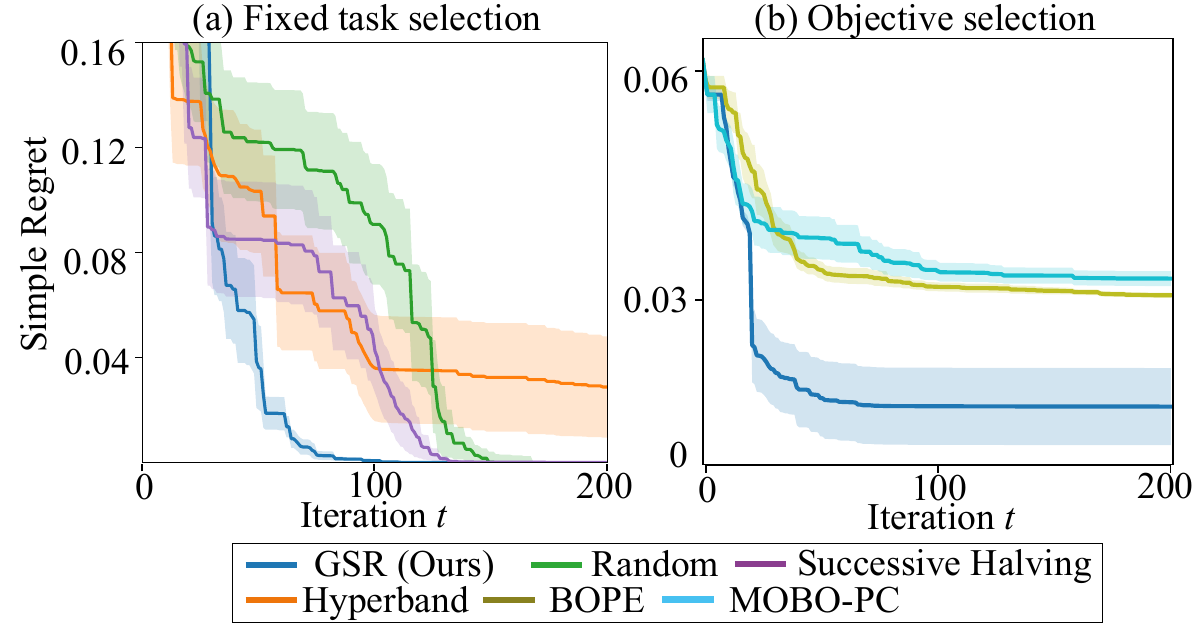}
    \caption{Offline settings: (a) Fixed task selection and (b) objective selection.
    }
    \label{fig:synthetic}
\end{wrapfigure}
\textbf{Offline experiments.}
Finally, we consider an offline setting with no generator: all tasks are fixed a priori, reducing the problem to task selection. Algorithm~\ref{alg:task-ucb} thus simplifies to pure task-UCB (omitting Lines~8–10). We evaluate six BoTorch benchmark functions, defining task utilities via a Gaussian CDF,
$
u^{(i)}(\overline{y}_t) := \Phi\!\Big(\sfrac{\overline{y}^{(i)}_t - \mu^{(i)}}{\sigma^{(i)}}\Big),
$
with task-specific constants $(\mu^{(i)}, \sigma^{(i)})$.
We study two variants: (a) fixed task selection, which allocates budget across independent tasks, and (b) objective selection, where all tasks share a common domain $\mathcal{X}^{(i)}=\mathcal{X} = [0,1]^6$.
Setting (a) relates to classical budget-allocation baselines (random, successive halving, Hyperband), while (b) connects to preference-based multi-objective BO (MOBO-PC~\citep{abdolshah2019multi} and BOPE~\citep{lin2022preference}\footnote{With minor adaptations since only one objective $y_t^{(i)}$ is observed at each step whereas MOBO assumes observe all $\{y_t^{(i)}\}_{i=1}^n$.}. 
Using best-so-far simple regret
$U^\star-\max_{\tau\le t}u^{(i_\tau)}(\overline{y}^{(i_\tau)}_{s_\tau^{(i_\tau)}})$
as the metric, GSR consistently outperforms these baselines (Fig.~\ref{fig:synthetic}).
See Appendix~\ref{app:add_exp} for additional experiments and Appendix~\ref{app:discussion} for extended discussion.

\vspace{-0.5em}
\section{Conclusion and Limitations}
We propose GSR, an open-ended BO framework that refines a vague seed task into an $\epsilon$-optimal task with GP-UCB-style regret guarantees, and apply it to planning and inverse optimization. While the theory is principled and the ablations in \S\ref{sec:ablation} support its assumptions, the current LLM-based task generation remains empirical, a limitation shared with evolutionary LLM methods. A second open challenge is utility specification: although some settings admit objective metrics, such as competitiveness in inverse optimization, and LLM committees show reasonable alignment with human judgments, reliable deployment still requires experimental validation and human oversight. At the same time, this challenge is also an opportunity, since utilities can themselves be learned from data, for example through preference learning with DPO \citep{rafailov2023direct, fujisawa2025scalable} or experimental feedback with GRPO \citep{shao2024deepseekmath}.

\bibliographystyle{plain}
\bibliography{neurips_2026}

\newpage
\appendix
\onecolumn



\newpage
\appendix
\section{Notations}\label{app:notation}
Notations are summarized in Table~\ref{tab:notation}.

\begin{table}
\centering
\caption{Notation (global task-rounds and local within-task BO).}
\label{tab:notation}
\small
\setlength{\tabcolsep}{6pt}
\renewcommand{\arraystretch}{1.15}
\begin{tabular}{@{}l p{0.82\textwidth}@{}}
\toprule
Symbol & Meaning \\
\midrule

\multicolumn{2}{@{}l}{\textbf{Indices, sets, counters, and shorthand}}\\
$t\in[T]$ & Global (meta) round; $T$ is the total evaluation budget.\\
$i\in\mathbb N$ & Task index; $i_0$ is the seed task; $i_t$ is the task selected at round $t$.\\
$\mathcal I_t$ & (Random) set of tasks instantiated and available by round $t$; $\mathcal I_T$ is the set by the horizon.\\
$s_t^{(i)}$ & Local evaluation counter: number of times task $i$ has been selected up to round $t$.\\
$m\in\{0,1,\dots,\overline m_T\}$ & Resolution level; $\overline m_T$ is the maximum level introduced by time $T$.\\
$q_t^{(i)}$ & Shorthand for within-task quantities: $q_t^{(i)} := q_{s_t^{(i)}}^{(i)}$ (used throughout the main text).\\

\addlinespace
\multicolumn{2}{@{}l}{\textbf{Task-level BO and observations}}\\
$\mathcal X^{(i)}\subset\mathbb R^{d^{(i)}}$ & Domain of task $i$; $d^{(i)}$ is its dimension.\\
$f^{(i)}:\mathcal X^{(i)}\to\mathbb R$ & Unknown black-box objective for task $i$.\\
$x_t^{(i_t)}$ & Design evaluated at global round $t$ on task $i_t$ (equivalently $x_{s_t^{(i_t)}}^{(i_t)}$ in local time).\\
$y_t^{(i_t)}$ & Noisy objective observation: $y_t^{(i_t)} = f^{(i_t)}(x_t^{(i_t)}) + \xi_t^f$.\\
$\xi_t^f\sim\mathrm{SubGauss}(\sigma_f^2)$ & Sub-Gaussian objective noise with variance proxy $\sigma_f^2$.\\
$\overline y_s^{(i)}$ & Incumbent after $s$ local evals on task $i$: $\overline y_s^{(i)}:=\max_{1\le r\le s} f^{(i)}(x_r^{(i)})$.\\
$f^{\star(i)}$ & Optimal objective value: $f^{\star(i)}:=\max_{x\in\mathcal X^{(i)}} f^{(i)}(x)$; $x^{\star(i)}\in\argmax_x f^{(i)}(x)$.\\
$\epsf_s$ & GP-UCB optimization gap bound after $s$ local evals (Eq.~\eqref{eq:epsf_def_main}); depends on $i$ via $\beta_s^{(i)},\gamma_s^{(i)}$.\\
$\epsfT_s$ & Uniform gap across instantiated tasks: $\epsfT_s := \max_{i\in\mathcal I_T}\epsf_s$.\\
$\beta_s^{(i)},\gamma_s^{(i)}$ & GP-UCB exploration parameter and information-gain term for task $i$.\\

\addlinespace
\multicolumn{2}{@{}l}{\textbf{Utility, long-run value, and confidence bounds}}\\
$u^{(i)}:\mathbb R\to[0,1]$ & Scale-invariant utility applied to incumbents; monotone and $L^{(i)}$-Lipschitz.\\
$L^{(i)}$ & Lipschitz constant of $u^{(i)}$.\\
$L^\star$ & Global Lipschitz upper bound: $L^\star := \sup_i L^{(i)}$.\\
$\Lbar$ & Supplied Lipschitz bound used by the algorithm, intended to satisfy $\Lbar\ge L^\star$.\\
$\tilde u_t^{(i_t)}$ & Noisy utility call: $\tilde u_t^{(i_t)} = u^{(i_t)}(\overline y_t^{(i_t)}) + \xi_t^u$.\\
$\xi_t^u\sim\mathrm{SubGauss}(\sigma_u^2)$ & Sub-Gaussian utility noise (averaging $K_t$ votes reduces proxy to $\sigma_u^2/K_t$).\\
$U^{(i)}$ & Long-run task value: $U^{(i)} := u^{(i)}(f^{\star(i)})$; and $U^\star := \sup_i U^{(i)}$.\\
$\LCBu_t^{(i)},\UCBu_t^{(i)}$ & Confidence bounds for $u^{(i)}(\overline y_t^{(i)})$.\\
$\LCBV_t^{(i)},\UCBV_t^{(i)}$ & Value envelopes for $U^{(i)}$ (Thm.~\ref{thm:value_envelopes_summary}).\\
$w_t^{(i)}$ & Value interval width: $w_t^{(i)}:=\UCBV_t^{(i)}-\LCBV_t^{(i)}$.\\
$\underline w_t$ & Minimum value-width across tasks: $\underline w_t := \min_{j\in\mathcal I_t} w_t^{(j)}$.\\
$\overline{\epsilon}^U_t$ & Anchor-resolution threshold: $\overline{\epsilon}^U_t := \max\{c_g\,\epsU_m,\ \underline w_t\}$.\\
$w_t$ & Width of the selected anchor: $w_t := w_t^{(a_t)}$.\\
$K_t$ & Number of committee votes averaged to form $\tilde u_t^{(i_t)}$ (if using vote-based feedback).\\

\addlinespace
\multicolumn{2}{@{}l}{\textbf{Resolution, generation, and regret}}\\
$\epsU$ & Utility-scale optimality tolerance (“resolution”).\\
$\epsUm$ & Resolution ladder: $\epsUm := \epsU_0 2^{-m}$; $m^\star$ is the BO-achievable level (Cond.~\ref{cond:bo_achievable}).\\
$\mathcal I^\star(\epsU)$ & $\epsU$-optimal task set: $\{i:U^\star-U^{(i)}\le \epsU\}$.\\
$a_t$ & Anchor task used for mutation at round $t$ (Eq.~\eqref{eq:anchor}).\\
$\mathcal B_m$ & Batch of $J$ mutated tasks at level $m$: $\mathcal B_m\sim\Gen(a_t,m,J)$.\\
$J$ & Number of tasks generated per generation event (batch size).\\
$\delta_+$ & Per-draw probability that a mutation hits an $\epsUm$-optimal task (Cond.~\ref{cond:stoch_coverage_adaptive}).\\
$\rho(i,a_t)$ & Mutation ratio between child task $i$ and parent/anchor $a_t$; schedule $\rho_m=\rho_0 2^{-m}$.\\
$\Regval_T$ & Meta regret: $\Regval_T := \sum_{t=1}^T\bigl(U^\star-u^{(i_t)}(\overline y_t^{(i_t)})\bigr)$.\\
$R_T^\star$ & Uniform single-task oracle scale: $R_T^\star := L^\star\sum_{s=1}^T \epsfT_s$.\\
$N_T$ & Number of instantiated tasks by time $T$: $N_T := |\mathcal I_T|$.\\
$c_g$ & Scheduler/gating constant in Alg.~\ref{alg:task-ucb} (triggers refinement when $w_t\le c_g\epsU_m$).\\

\bottomrule
\end{tabular}
\end{table}

\section{Preliminary}\label{app:gp_ucb}
\subsection{Gaussian process regression background}
\textbf{Gaussian process model.}
Let $f:\mathcal{X}\to\mathbb{R}$ be an unknown function with prior
$f\sim \mathcal{GP}(0,k)$ \citep{stein1999interpolation,williams2006gaussian}.
At each local round $s$, we query $\mathbf{x}_s\in\mathcal{X}$ and observe
$y_s = f(\mathbf{x}_s) + \xi^f_s,$
where $\xi^f_s \sim \text{SubGauss}(\sigma^2_f)$.
In GP regression we model this as Gaussian noise with variance $\lambda>0$.
Let $\mathbf{X}_s=[\mathbf{x}_1,\dots,\mathbf{x}_s]^\top$ and $\mathbf{Y}_s=[y_1,\dots,y_s]^\top$.
Conditioned on the data $\mathbf{D}_s=(\mathbf{X}_s,\mathbf{Y}_s)$, the posterior is
$f_s \mid \mathbf{D}_s \sim \mathcal{GP}(\mu_s,k_s),$
where
\begin{align*}
\mu_s(\mathbf{x})
&= k(\mathbf{x},\mathbf{X}_s)\,(\mathbf{K}_s+\lambda\mathbf{I})^{-1}\mathbf{Y}_s,\\
k_s(\mathbf{x},\mathbf{x}')
&= k(\mathbf{x},\mathbf{x}')-k(\mathbf{x},\mathbf{X}_s)\,(\mathbf{K}_s+\lambda\mathbf{I})^{-1}\,k(\mathbf{X}_s,\mathbf{x}'),
\end{align*}
with $\mathbf{K}_s=[k(\mathbf{x}_r,\mathbf{x}_{r'})]_{r,r'=1}^s$ and $\sigma_s^2(\mathbf{x}) := k_s(\mathbf{x},\mathbf{x})$.

\textbf{GP confidence bounds.}
The $f$ lies in a GP confidence band uniformly over time and space w.h.p.:
\begin{theorem}[Theorem 2 in \citep{chowdhury2017kernelized}]\label{thm:gp-ucb}
Let $f\in\mathcal{H}(k)$ be in the RKHS associated with $k$ with $\|f\|_{k}\le B$.
Let $\lambda>0$ be the noise variance used in GP regression and define
$
\beta_s^{1/2} := B + \sigma_{\mathrm{noise}}\sqrt{2\bigl(\gamma_{s-1}+1+\ln(1/\delta_f)\bigr)},
$
where $\gamma_S$ upper bounds the maximum information gain:
$\tfrac{1}{2}\log\bigl|\mathbf{I}+\lambda^{-1}\mathbf{K}_S\bigr|\le \gamma_S$.
Then, w.p.$\geq1-\delta_f$, for all $\mathbf{x}\in\mathcal{X}$ and all $s\ge 1$,
\vspace{-0.5em}
\[
\bigl|f(\mathbf{x})-\mu_{s-1}(\mathbf{x})\bigr|
\le
\beta^{1/2}_s\,\sigma_{s-1}(\mathbf{x}).
\]
\end{theorem}
\vspace{-0.5em}
For $s\ge 1$, define the pointwise confidence bounds
\begin{align*}
    \underline{f}_s(\mathbf{x}) &:= \mu_{s-1}(\mathbf{x})-\beta^{1/2}_s\sigma_{s-1}(\mathbf{x}),\\
    \overline{f}_s(\mathbf{x}) &:= \mu_{s-1}(\mathbf{x})+\beta^{1/2}_s\sigma_{s-1}(\mathbf{x}).
\end{align*}
\vspace{-1em}
\begin{theorem}[Theorem 3 in \citep{chowdhury2017kernelized}]
\label{thm:gpucb-regret}
Assume the setting of \cref{thm:gp-ucb} and run GP-UCB with decision rule
$
\textbf{x}_s \in \argmax_{\textbf{x}\in\mathcal{X}^{(i)}}\ \overline{f}^{(i)}_{s}(\mathbf{x}).
$
With the cumulative regret 
\small
$
R_{\text{cr}}(S):=\sum_{s=1}^{S} \bigl(f^{\star} - f(\textbf{x}_{s})\bigr),
$ 
\normalsize
we have
\[
R_{\text{cr}}(S)\le 2\,\sqrt{C_\lambda\, S \,\beta_S\,\gamma_S} = S\,\epsf_S,
\]
where $C_\lambda := \frac{2}{\log(1+\lambda^{-1})}$ and $\epsf_s := 2\sqrt{\frac{C_\lambda \beta_s \gamma_s}{s}}.$ .
\end{theorem}

\begin{lemma}[Uniform GP confidence over adaptively instantiated tasks]
\label{lem:uniform_gp_confidence_tasks}
For each task $i\in\mathbb N$, run GP regression / GP-UCB using confidence level
\[
\delta_f^{(i)} \ :=\ \frac{\delta_f}{\pi^2 i^2}.
\]
Let $\beta_s^{(i)}$ be defined as in Theorem~\ref{thm:gp-ucb} but with $\delta_f$ replaced by $\delta_f^{(i)}$.
Then with probability at least $1-\delta_f$, simultaneously for all tasks $i\in\mathbb N$, all $s\ge 1$,
and all $\mathbf{x}\in\mathcal X^{(i)}$,
\[
\bigl|f^{(i)}(\mathbf{x})-\mu^{(i)}_{s-1}(\mathbf{x})\bigr|
\le
\bigl(\beta^{(i)}_s\bigr)^{1/2}\,\sigma^{(i)}_{s-1}(\mathbf{x}).
\]
Consequently, all within-task optimization gap bounds $f^{\star(i)}-\overline{y}_s^{(i)}\le \epsf_s$ hold simultaneously for all instantiated tasks.
\end{lemma}

\begin{proof}
Apply Theorem~\ref{thm:gp-ucb} to each fixed task $i$ with confidence $\delta_f^{(i)}$.
Then, by a union bound over $i\in\mathbb N$,
\[
\Pr\Big(\exists i:\ \text{GP confidence for task $i$ fails}\Big)
\le
\sum_{i=1}^\infty \delta_f^{(i)}
=
\frac{\delta_f}{\pi^2}\sum_{i=1}^\infty \frac{1}{i^2}
\le
\delta_f.
\]
On the complement event, the GP confidence bands hold for all tasks simultaneously, implying the optimization gap envelopes
via the same argument as Lemma~\ref{lem:simple_regret}.
\end{proof}

\section{LLM task generator and evaluator}
\subsection{Evaluator}\label{app:evaluator}

\textbf{Bradley--Terry preference model.}
We model the evaluator's latent preference using the popular Bradley-Terry (BT) model \citep{bradley1952rank}:
\vspace{-0.3em}
\begin{equation}
\label{eq:btl_pairwise_main}
\Pr\!\bigl((i,z)\succ (j,z')\bigr)
\ =\
\sigmoid\!\Big(\theta^{(i)}(z)-\theta^{(j)}(z')\Big),
\end{equation}
where $\sigmoid(x):=\sfrac{1}{1+e^{-x}}$ and $\theta^{(i)}(z)\in\mathbb R$ is the preference score over task-incumbent pairs; $(i,z)$ and $(j, \overline{y}^\prime)$. 

\textbf{Ordinal to cardinal utility.}
However, BT model's $\theta^{(i)}$ is \emph{ordinal}, not cardinal. To convert, we assume a fixed \emph{reference} $r$ (the seed task $i_0$ and its incumbent $\overline{y}^{(i_0)}_0$). We define the utility as a win probability against this fixed reference:
$u^{(i)}(z)
\ :=\
\Pr\!\bigl((i,z)\succ r\bigr)
\in[0,1].$
This yields a bounded \emph{cardinal} utility.
We normalize the reference so that $\theta(r)=0$; i.e., 
$u^{(i)}(z)=\sigmoid(\theta^{(i)}(z))$.

\textbf{Moving-anchor queries.}
\begin{figure}[b!]
  \centering
  \includegraphics[width=0.5\linewidth]{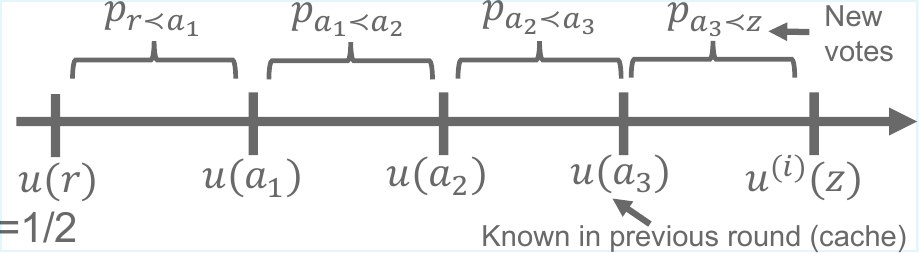}
  \caption{Log-odds chain rule: adding logits and mapping back with sigmoid offers cardinal utility $u^{(i)}(z)$ estimation.}
  \label{fig:anchor}
  \vspace{-1em}
\end{figure}
Comparisons to a fixed reference can saturate (voters always prefer the improved task), making feedback uninformative.
We instead compare against a moving anchor $a_t$ (the current anchor task) and observe
votes $b_{t,k}\in\{0,1\}$, where $b_{t,k}=1$ iff
$(i_t,\overline{y}_t^{(i_t)})\succ(a_t,\overline{y}_t^{(a_t)})$.
Let $p_t:=\Pr(b_{t,k}=1\mid\mathcal F_{t-1})$ be the true win rate and
$\hat p_t:=\frac{1}{K_t}\sum_{k=1}^{K_t} b_{t,k}$ its empirical estimate.
Under the BT model, we have:
\begin{equation}
\label{eq:logodds_transport_main}
\logit u^{(i_t)}\!\bigl(\overline{y}_t^{(i_t)}\bigr)
=
\logit u^{(a_t)}\!\bigl(\overline{y}_t^{(a_t)}\bigr)
+
\logit p_t.
\end{equation}
See Fig.\ref{fig:anchor}; given votes and a cached anchor utility interval,
we can estimate cardinal utilities.

\subsubsection{Preference feedback and sub-Gaussian concentration.}
Our envelope theory only requires the CI interface.
In the main paper we implement utility calls using \emph{pairwise} comparisons calibrated via a global BT scale, where the raw feedback is a collection of bounded votes and thus concentrates.

\begin{lemma}[Bounded utility feedback is conditionally sub-Gaussian]\label{lem:util_bounded_subg}
Fix a utility call $\ell$ and condition on $\mathcal F_{t_\ell-1}$ and $\overline{y}_{s_\ell}^{(i_\ell)}$.
Assume $\tilde{u}_\ell\in[0,1]$ and
$\mathbb E[\tilde{u}_\ell\mid \mathcal F_{t_\ell-1},\overline{y}_{s_\ell}^{(i_\ell)}]=u^{(i_\ell)}(\overline{y}_{s_\ell}^{(i_\ell)})$.
If $\tilde{u}_\ell$ is formed as an average of $K_\ell\ge 1$ conditionally independent bounded samples in $[0,1]$
(with the same conditional mean), then letting
$\xi_\ell:=\tilde{u}_\ell-u^{(i_\ell)}(\overline{y}_{s_\ell}^{(i_\ell)})$ we have for all $\lambda\in\mathbb R$,
\[
\mathbb{E}\!\left[\exp(\lambda \xi_\ell)\mid \mathcal F_{t_\ell-1},\overline{y}_{s_\ell}^{(i_\ell)}\right]
\le
\exp\!\left(\frac{\lambda^2}{8K_\ell}\right).
\]
Equivalently, $\xi_\ell$ is conditionally sub-Gaussian with variance proxy $1/(4K_\ell)$.
In particular, for single-shot bounded feedback we recover the case $K_\ell=1$.
\end{lemma}

\begin{proof}
Fix $\ell$ and condition on $\mathcal{F}_{t_\ell-1}$ and $\overline{y}:=\overline{y}_{s_\ell}^{(i_\ell)}$.
Let $\mu:=u^{(i_\ell)}(z)$.
Let $\tilde{u}_\ell=\frac1{K_\ell}\sum_{k=1}^{K_\ell} b_{\ell,k}$ with $b_{\ell,k}\in[0,1]$ conditionally independent and $\mathbb E[b_{\ell,k}\mid\cdot]=\mu$.
Then each $b_{\ell,k}-\mu$ is mean-zero and lies in $[-\mu,1-\mu]\subset[-1,1]$.
By Hoeffding's lemma, for all $\lambda\in\mathbb R$,
\[
\mathbb{E}\!\left[\exp\!\bigl(\lambda(b_{\ell,k}-\mu)\bigr)\mid \mathcal{F}_{t_\ell-1},z\right]
\le \exp\!\left(\frac{\lambda^2}{8}\right).
\]
By conditional independence and $\xi_\ell=\tilde{u}_\ell-\mu=\frac{1}{K_\ell}\sum_{k=1}^{K_\ell}(b_{\ell,k}-\mu)$,
\[
\mathbb{E}\!\left[\exp\!\bigl(\lambda\xi_\ell\bigr)\mid \mathcal{F}_{t_\ell-1},z\right]
=
\prod_{k=1}^{K_\ell}
\mathbb{E}\!\left[\exp\!\Bigl(\frac{\lambda}{K_\ell}(b_{\ell,k}-\mu)\Bigr)\mid \mathcal{F}_{t_\ell-1},z\right]
\le
\exp\!\left(\frac{\lambda^2}{8K_\ell}\right).
\]
\end{proof}

\paragraph{Confidence interval for the pairwise probability.}
Let $\delta_\ell\in(0,1)$ be the per-call failure budget and define the Hoeffding radius
\[
\theta_p(K_\ell,\delta_\ell)
:=
\sqrt{\frac{\log(2/\delta_\ell)}{2K_\ell}}.
\]
Then with probability at least $1-\delta_\ell$ (conditional on the compared incumbents),
\[
p_\ell \in
\bigl[\hat p_\ell-\theta_p(K_\ell,\delta_\ell),\ \hat p_\ell+\theta_p(K_\ell,\delta_\ell)\bigr].
\]
We clip endpoints to $[0,1]$ (and, when actually evaluating logits in code, one may additionally clip
to $[\epsclip,1-\epsclip]$ to avoid overflow):
\[
p_\ell^- := \clip{\hat p_\ell-\theta_p(K_\ell,\delta_\ell)}{0}{1},
\qquad
p_\ell^+ := \clip{\hat p_\ell+\theta_p(K_\ell,\delta_\ell)}{0}{1}.
\]
(For numerical stability in code, one may additionally clip $p_\ell^\pm$ to
$[\epsclip,1-\epsclip]$; see the practical note below.)

\paragraph{Transported utility confidence interval.}
Assume we already maintain a valid interval for the anchor utility
\[
u^{(a_\ell)}(\overline{y}_{s_\ell}^{(a_\ell)}) \in [\LCBu^{(a_\ell)},\UCBu^{(a_\ell)}].
\]
(For the fixed reference artifact $r$, we may take $\LCBu^{(r)}=\UCBu^{(r)}=\tfrac12$ since $u(r)=\Pr(r\succ r)=\tfrac12$.)
Define the transported interval
\begin{equation}
\label{eq:transported_interval_app}
\LCBu^{(i_\ell)}
:=
\sigmoid\!\Big(\logit(\LCBu^{(a_\ell)})+\logit(p_\ell^-)\Big),
\qquad
\UCBu^{(i_\ell)}
:=
\sigmoid\!\Big(\logit(\UCBu^{(a_\ell)})+\logit(p_\ell^+)\Big).
\end{equation}

\begin{lemma}[Validity of the transported utility interval]\label{lem:transported_interval_valid}
Work on the event that (i) the anchor interval is valid and (ii) $p_\ell\in[p_\ell^-,p_\ell^+]$.
Under BT model,
\[
u^{(i_\ell)}(\overline{y}_{s_\ell}^{(i_\ell)}) \in [\LCBu^{(i_\ell)},\UCBu^{(i_\ell)}].
\]
\end{lemma}
\begin{proof}
By transport identity,
\[
\logit u^{(i_\ell)}(\overline{y}_{s_\ell}^{(i_\ell)})
=
\logit u^{(a_\ell)}(\overline{y}_{s_\ell}^{(a_\ell)})
+
\logit p_\ell.
\]
On the event that $u^{(a_\ell)}(\cdot)\in[\LCBu^{(a_\ell)},\UCBu^{(a_\ell)}]$ and $p_\ell\in[p_\ell^-,p_\ell^+]$,
monotonicity of $\logit(\cdot)$ implies
\[
\logit u^{(i_\ell)}(\cdot)
\in
\bigl[\logit(\LCBu^{(a_\ell)})+\logit(p_\ell^-),\ \logit(\UCBu^{(a_\ell)})+\logit(p_\ell^+)\bigr].
\]
Applying the monotone map $\sigmoid(\cdot)$ to endpoints gives \eqref{eq:transported_interval_app}.
\end{proof}

\paragraph{Uniform confidence across all utility calls (with a summable schedule).}
Let $\delta_u\in(0,1)$ and set $\delta_\ell:=\delta_u/(\pi^2\ell^2)$.
Let $\mathcal E_p$ be the event that $p_\ell\in[p_\ell^-,p_\ell^+]$ holds for all $\ell$ simultaneously.
By a union bound and Hoeffding, $\Pr(\mathcal E_p)\ge 1-\delta_u$.
On $\mathcal E_p$, validity of transported intervals follows by induction over utility calls using Lemma~\ref{lem:transported_interval_valid}
(starting from the reference $r$ which has a trivial valid interval).

\paragraph{Committee size (vote complexity).}
Our envelope analysis only requires that, at each checkpoint, the constructed utility interval width
$\UCBu^{(i)}-\LCBu^{(i)}$ is below a target tolerance (e.g., $2\eta^{(i)}(s)$).
With transported intervals, a practical rule is to start from a small $K$ and keep drawing fresh votes until the width target is met.

The next theorem gives a sufficient condition that ensures a target \emph{utility} half-width, under a mild non-degeneracy condition.
(For near-optimal tasks compared against the current best anchor, this condition is typically satisfied.)

\begin{theorem}[Sufficient committee size for a target transported-utility half-width]
\label{thm:committee_size_btl}
Fix a utility call $\ell$ and let $w_a$ be the \emph{anchor log-odds width}
\[
w_a := \logit(\UCBu^{(a_\ell)})-\logit(\LCBu^{(a_\ell)}).
\]
Assume the true comparison probability satisfies $p_\ell\in[\kappa,1-\kappa]$ for some $\kappa\in(0,\tfrac12]$,
and choose $K_\ell$ so that
\begin{equation}
\label{eq:K_btl_sufficient}
K_\ell
\ \ge\
\max\left\{
\frac{8\log(2/\delta_\ell)}{\kappa^2},\;
\frac{32\log(2/\delta_\ell)}{\kappa^2\,(8\eta - w_a)^2}
\right\},
\qquad\text{for some target }\eta>0\text{ with }w_a<8\eta.
\end{equation}
Then, on the event of Lemma~\ref{lem:transported_interval_valid}, the transported interval satisfies
\[
\UCBu^{(i_\ell)}-\LCBu^{(i_\ell)} \le 2\eta.
\]
\end{theorem}
\begin{proof}
Let $\theta_a^-:=\logit(\LCBu^{(a_\ell)})$ and $\theta_a^+:=\logit(\UCBu^{(a_\ell)})$ so that $w_a=\theta_a^+-\theta_a^-$.
Let $d^-:=\logit(p_\ell^-)$ and $d^+:=\logit(p_\ell^+)$.
Then the transported log-odds interval for $\theta^{(i_\ell)}$ has width
\[
w_i := (\theta_a^+ + d^+) - (\theta_a^- + d^-)
= w_a + (d^+-d^-).
\]
Since $\sigmoid$ has derivative at most $1/4$, it is $1/4$-Lipschitz, hence
\[
\UCBu^{(i_\ell)}-\LCBu^{(i_\ell)}
\le \frac{1}{4}w_i
= \frac{1}{4}w_a + \frac{1}{4}(d^+-d^-).
\]
It remains to control $d^+-d^-$. 
The first term in \eqref{eq:K_btl_sufficient} gives $\theta_p\le \kappa/4$.
On the Hoeffding event, $\hat p_\ell\ge p_\ell-\theta_p$ and
$\hat p_\ell\le p_\ell+\theta_p$, hence
$p_\ell^- \ge p_\ell-2\theta_p\ge \kappa/2$ and
$p_\ell^+ \le p_\ell+2\theta_p\le 1-\kappa/2$.
Thus $p_\ell^\pm\in[\kappa/2,1-\kappa/2]$.
On this interval, $\logit'(p)=1/(p(1-p))\le 4/\kappa$ since $p(1-p)\ge (\kappa/2)(1/2)=\kappa/4$.
By the mean value theorem,
\[
d^+-d^-
=
\logit(p_\ell^+)-\logit(p_\ell^-)
\le \frac{4}{\kappa}(p_\ell^+-p_\ell^-)
\le \frac{4}{\kappa}\cdot 2\theta_p
= \frac{8\theta_p}{\kappa}.
\]
Therefore
\[
\UCBu^{(i_\ell)}-\LCBu^{(i_\ell)}
\le \frac{w_a}{4} + \frac{2\theta_p}{\kappa}.
\]
To make this at most $2\eta$, it suffices that
\[
\theta_p \le \frac{\kappa}{2}\Bigl(2\eta-\frac{w_a}{4}\Bigr)
= \frac{\kappa}{8}(8\eta-w_a).
\]
Since $\theta_p=\sqrt{\log(2/\delta_\ell)/(2K_\ell)}$, the inequality above is equivalent to
\[
K_\ell \ \ge\ \frac{32\log(2/\delta_\ell)}{\kappa^2\,(8\eta-w_a)^2},
\]
which is exactly \eqref{eq:K_btl_sufficient}.
\end{proof}

\begin{remark}[Practical committee sizing via doubling]
\label{rem:committee_doubling}
In practice, $\kappa$ is unknown and may vary across calls.
A robust choice is a doubling rule: start from a small $K$ and keep sampling fresh voters until
the computed width $\UCBu^{(i_\ell)}-\LCBu^{(i_\ell)}$ drops below the target (e.g., $2\eta^{(i)}(s)$).
Theorem~\ref{thm:committee_size_btl} implies that when the comparison is informative (i.e., $p_\ell$ is not extremely close to $0$ or $1$
for tasks that remain competitive), the required $K$ scales as $\tilde{\mathcal O}(\eta^{-2})$.
\end{remark}

\begin{remark}[Anchor-width limitation of transported intervals]
Transported logit intervals inherit the anchor log-odds width $w_a$.  Therefore increasing the
number of votes for the child comparison can shrink only the comparison part of the interval, not
the cached anchor uncertainty.  If $w_a\ge 8\eta$, the target child half-width $\eta$ is unattainable
from that anchor; the implementation must first refine the anchor interval, use a more certain
reference, or fall back to direct bounded utility feedback.  The regret analysis uses only the resulting
valid CI interface and its width, not the particular transport mechanism.
\end{remark}

\paragraph{Choosing $K_s$ explicitly (and why it is not circular).}
If we want the transported interval to have a \emph{target} utility half-width $\eta_s>0$ at step $s$,
we can choose $K_s$ using Theorem~\ref{thm:committee_size_btl} as follows.
Let the anchor \emph{log-odds width} be
\[
w_{a_t}\ :=\ \logit(\UCBu^{(a_t)})-\logit(\LCBu^{(a_t)}),
\]
which is known \emph{before} drawing the $K_s$ votes at step $s$ because it depends only on the already-maintained anchor interval.
Assume the comparison is not saturated: $p_s\in[\kappa,1-\kappa]$ for some fixed $\kappa\in(0,\tfrac12]$.
If $w_{a_t}<8\eta_s$, then it suffices to pick
\begin{equation}
\label{eq:Ks_explicit_non_circular}
K_s
\ :=\
\left\lceil
\max\left\{
\frac{8\log(2/\delta_u)}{\kappa^2},\;
\frac{32\log(2/\delta_u)}{\kappa^2\,(8\eta_s-w_{a_t})^2}
\right\}
\right\rceil.
\end{equation}
With this choice, Theorem~\ref{thm:committee_size_btl} guarantees
$\UCBu_s^{(i)}-\LCBu_s^{(i)}\le 2\eta_s$ on the same event as Lemma~\ref{lem:transported_interval_valid}.
There is no circular dependence: $w_{a_t}$ depends on the \emph{anchor}'s previously computed interval, not on the current $K_s$.

\paragraph{Practical note on $\epsclip$.}
The lemma/proof above does not require any $\epsclip$ (we allow $\logit(0),\logit(1)$ as $\pm\infty$).
In code, to avoid numerical overflow, a safe rule of thumb is to clip only for computation:
replace $p_s^\pm$ by $\clip{p_s^\pm}{\epsclip}{1-\epsclip}$ with, e.g.,
\[
\epsclip = 10^{-6}\ \ (\text{float32})\qquad\text{or}\qquad \epsclip = 10^{-12}\ \ (\text{float64}),
\]
and if $p_s^-\le \epsclip$ (resp.\ $p_s^+\ge 1-\epsclip$) simply set the transported bound to $0$ (resp.\ $1$).
There is no clean way to choose $\epsclip$ purely from $\delta_u$ because $\delta_u$ controls estimation error around $p_s$
while the true $p_s$ itself can be arbitrarily close to $0$ or $1$.

\paragraph{A simple vote budget in terms of $\beta_s^{(i)}$ and $\gamma_s^{(i)}$.}
The envelope theory only needs that, at each utility call, the produced interval
$\bigl[\LCBu,\UCBu\bigr]$ is valid and has width below a target tolerance.
A convenient sufficient condition (covering direct bounded scalar ratings or Bernoulli vote averages)
is obtained by combining Hoeffding concentration with the GP-UCB optimization gap $\epsf_s$.

\begin{lemma}[Sufficient committee size for width $\propto \epsf_s$]
\label{lem:committee_size_epsf_beta_gamma}
Fix a utility call $\ell$ that occurs when task $i_\ell$ has local counter $s_\ell$.
Assume the feedback $\tilde{u}_\ell\in[0,1]$ is formed as the average of $K_\ell\ge 1$
conditionally independent bounded samples in $[0,1]$ with conditional mean
$u^{(i_\ell)}\!\bigl(\overline{y}^{(i_\ell)}_{s_\ell}\bigr)$.
Let the per-call failure budget be $\delta_\ell:=\delta_u/(\pi^2\ell^2)$ and define the Hoeffding radius
\[
\phi_\ell := \sqrt{\frac{\log(2/\delta_\ell)}{2K_\ell}}.
\]
If $K_\ell$ satisfies
\begin{equation}
\label{eq:Kell_sufficient_epsf}
K_\ell
\ \ge\
\left\lceil
\frac{s_\ell}{2\,c_u^2\,C_\lambda\,\Lbar^2\,\beta_{s_\ell}^{(i_\ell)}\,\gamma_{s_\ell}^{(i_\ell)}}
\;\log\!\Big(\frac{2}{\delta_\ell}\Big)
\right\rceil,
\end{equation}
then on the event of the corresponding Hoeffding bound we have
\[
\UCBu_\ell^{(i_\ell)}-\LCBu_\ell^{(i_\ell)}
\;=\; 2\phi_\ell
\;\le\; c_u\,\Lbar\,\epsf_{s_\ell},
\qquad
\epsf_s:=2\sqrt{C_\lambda\,\beta_s^{(i)}\gamma_s^{(i)}/s}
\quad \text{(for the corresponding task $i$).}.
\]
\end{lemma}

\begin{proof}
We require $2\phi_\ell\le c_u\,\Lbar\,\epsf_{s_\ell}$.
Since $\phi_\ell=\sqrt{\log(2/\delta_\ell)/(2K_\ell)}$, this is equivalent to
\[
K_\ell \ \ge\ \frac{2\log(2/\delta_\ell)}{c_u^2\,\Lbar^2\,(\epsilon^{f, (i_\ell)}_{s_\ell})^2}.
\]
Using $(\epsf_s)^2 = 4C_\lambda\,\beta_s^{(i)}\gamma_s^{(i)}/s$ gives \eqref{eq:Kell_sufficient_epsf}.
\end{proof}

\subsubsection{Total number of votes under sparse utility calls.}\label{app:sparse_call_bound}
Assume the sparse schedule $\mathcal S_u=\{1,2,4,\dots\}$ and let
$N_i:=s_T^{(i)}$ be the number of objective evaluations allocated to task $i$ by time $T$.
Let $U_T$ denote the total number of utility calls made by time $T$ (across all tasks).
Then each task $i$ contributes at most $1+\lfloor\log_2 N_i\rfloor$ calls, hence
\[
U_T \ \le\ \sum_{i\in\mathcal I_T}\bigl(1+\lfloor\log_2 N_i\rfloor\bigr)
\ \le\ N_T\bigl(1+\log_2 T\bigr).
\]
Define the total number of votes
\[
\mathrm{Votes}(T)\ :=\ \sum_{\ell=1}^{U_T} K_\ell.
\]
If we choose $K_\ell$ according to \eqref{eq:Kell_sufficient_epsf} at every utility call,
then using $\log(2/\delta_\ell)\le \log\!\bigl(2\pi^2 U_T^2/\delta_u\bigr)$ for all $\ell\le U_T$ yields
\begin{equation}
\label{eq:vote_complexity_beta_gamma}
\mathrm{Votes}(T)
\ \le\
\frac{\log\!\bigl(2\pi^2 U_T^2/\delta_u\bigr)}{2\,c_u^2\,C_\lambda\,\Lbar^2}
\sum_{i\in\mathcal I_T}\ \sum_{r=0}^{\lfloor\log_2 N_i\rfloor}
\frac{2^r}{\beta_{2^r}^{(i)}\,\gamma_{2^r}^{(i)}}
\;+\;U_T,
\end{equation}
where the final $+U_T$ accounts for the ceilings in \eqref{eq:Kell_sufficient_epsf}.
\noindent
(When the ceilings are dropped, the bound holds without the $+U_T$ term.)

\begin{corollary}[Big-$\mathcal O$ vote complexity under sparse utility calls]
\label{cor:vote_complexity_T}
Assume the sparse utility schedule $\mathcal S_u=\{1,2,4,\dots\}$ and choose the committee sizes
$K_\ell$ according to \eqref{eq:Kell_sufficient_epsf} at every utility call.
Recall $N_i:=s_T^{(i)}$, $N_T:=|\mathcal I_T|$, and $U_T\le N_T(1+\log_2 T)$.
Then \eqref{eq:vote_complexity_beta_gamma} implies the deterministic bound
\begin{equation}
\label{eq:vote_complexity_bigO_general}
\mathrm{Votes}(T)
\ \le\
\frac{\log\!\Big(\frac{2\pi^2 N_T^2(1+\log_2 T)^2}{\delta_u}\Big)}{2\,c_u^2\,C_\lambda\,\Lbar^2}
\sum_{i\in\mathcal I_T}\ \sum_{r=0}^{\lfloor\log_2 N_i\rfloor}
\frac{2^r}{\beta_{2^r}^{(i)}\,\gamma_{2^r}^{(i)}}
\;+\;N_T(1+\log_2 T).
\end{equation}
In particular, under Algorithm~\ref{alg:task-ucb} with constant batch size $J$ we have
$N_T \le 1+J(\overline{m}_T+1)=\tilde{\mathcal O}(\log T)$ (hence $U_T=\tilde{\mathcal O}(\log^2 T)$),
so the logarithmic prefactor in \eqref{eq:vote_complexity_bigO_general} is $\tilde{\mathcal O}(\log\log T)$ and
the additive term is $\tilde{\mathcal O}(\log^2 T)$.

Moreover, suppose there exists a nondecreasing function $g:\mathbb N\to\mathbb R_+$ and a constant $c_{\mathrm{ig}}>0$ such that
for all instantiated tasks $i\in\mathcal I_T$ and all $s\ge 1$,
\begin{equation}
\label{eq:beta_gamma_lower_bound_g}
\beta_s^{(i)}\,\gamma_s^{(i)} \ \ge\ c_{\mathrm{ig}}\,g(s).
\end{equation}
Then, absorbing the prefactor in \eqref{eq:vote_complexity_bigO_general} into $\tilde{\mathcal O}(\cdot)$,
\begin{equation}
\label{eq:vote_complexity_g}
\mathrm{Votes}(T)
\ \le\
\tilde{\mathcal O}\!\left(
\sum_{i\in\mathcal I_T}\ \sum_{r=0}^{\lfloor\log_2 N_i\rfloor}\frac{2^r}{g(2^r)}
\right)
\;+\;\tilde{\mathcal O}\!\bigl(N_T\log T\bigr).
\end{equation}
Two useful special cases are:
\begin{compactenum}[(a)]
\item \textbf{Polynomial growth.} If $g(s)=s^{2\alpha}$ for some $\alpha>0$, then letting $(x)_+:=\max\{x,0\}$,
\begin{equation}
\label{eq:vote_complexity_poly_T}
\mathrm{Votes}(T)
\ \le\
\tilde{\mathcal O}\!\bigl(N_T^{2\alpha}\,T^{(1-2\alpha)_+}\bigr),
\end{equation}
and in particular if $N_T=\mathrm{polylog}(T)$ then $\mathrm{Votes}(T)=o(T)$.
\item \textbf{Polylog growth.} If $g(s)=(\log(e s))^{p}$ for some $p>0$, then
\begin{equation}
\label{eq:vote_complexity_polylog_T}
\mathrm{Votes}(T)
\ \le\
\tilde{\mathcal O}\!\left(\sum_{i\in\mathcal I_T}\frac{N_i}{(\log(e N_i))^{p}}\right)
\ \le\
\tilde{\mathcal O}\!\left(\frac{T}{(\log(eT))^{p}} + N_T\right)
\ =\ o(T)
\qquad(\text{if }N_T=\mathrm{polylog}(T)).
\end{equation}
For example, for squared-exponential (RBF) kernels on $[0,1]^d$ one may take $p=2(d+1)$, yielding
$\mathrm{Votes}(T)=\tilde{\mathcal O}\!\bigl(T/(\log T)^{2(d+1)}\bigr)$.
\end{compactenum}
\end{corollary}

\begin{proof}[Proof sketch]
Plug $U_T\le N_T(1+\log_2 T)$ into \eqref{eq:vote_complexity_beta_gamma} to obtain \eqref{eq:vote_complexity_bigO_general}.
If \eqref{eq:beta_gamma_lower_bound_g} holds, then each summand satisfies
$\frac{2^r}{\beta_{2^r}^{(i)}\gamma_{2^r}^{(i)}}\le \frac{1}{c_{\mathrm{ig}}}\frac{2^r}{g(2^r)}$,
yielding \eqref{eq:vote_complexity_g}.
For $g(s)=s^{2\alpha}$, the inner sum is a geometric series in $2^{(1-2\alpha)r}$ and the across-task bound
$\sum_i N_i^{1-2\alpha}\le N_T^{2\alpha}T^{1-2\alpha}$ (for $\alpha\in(0,\tfrac12)$) gives \eqref{eq:vote_complexity_poly_T}
(up to polylog factors, and with the standard modifications at $\alpha\ge \tfrac12$).
For $g(s)=(\log(es))^p$, the sum over $r$ is dominated by the last term, giving the displayed bound in
\eqref{eq:vote_complexity_polylog_T}.
\end{proof}

\subsection{Generator}
\label{app:mutation_coverage}

\subsubsection{Proof of Lemma~\ref{lem:max_depth_reachability}}\label{proof:max_depth_reachability}

\begin{proof}
Let $m_T$ denote the current ladder level in Algorithm~\ref{alg:task-ucb}.
We show that with the stated choice of $\overline m_T$, every level $m\le \overline m_T$
is introduced by time $T$.

\paragraph{Step 1: an explicit upper bound on the per-level sample threshold $s_m$.}
Fix a level $m\ge 0$ and define
\[
s_m \;:=\; \min\Bigl\{s\in\mathbb N:\ C_v\Lbar\,\epsfT_s\ \le\ c_g\,\epsU_m\Bigr\}.
\]
Using the monotonicity $\beta_s^{(i)}\le \beta_T^{(i)}$ and $\gamma_s^{(i)}\le \gamma_T^{(i)}$ for $s\le T$ and the definitions
$\beta_T:=\max_{i\in\mathcal I_T}\beta_T^{(i)}$, $\gamma_T:=\max_{i\in\mathcal I_T}\gamma_T^{(i)}$, we have for all $s\le T$,
\[
\epsfT_s
=\max_{i\in\mathcal I_T}2\sqrt{\frac{C_\lambda\,\beta_s^{(i)}\,\gamma_s^{(i)}}{s}}
\ \le\ 
2\sqrt{\frac{C_\lambda\,\beta_T\,\gamma_T}{s}}
=:U(s).
\]

Here, if we have the inequality
\[
C_v\Lbar\,U(s)\le c_g\epsU_m,
\]
Then, we can say
\[
C_v \Lbar \epsfT \leq C_v\Lbar\,U(s)\le c_g\epsU_m.
\]
i.e.,
\[
C_v\Lbar\cdot 2\sqrt{\frac{C_\lambda\,\beta_T\,\gamma_T}{s}}
\le c_g\epsU_m
\quad\Longleftrightarrow\quad
s \ge \left(\frac{2C_v\Lbar\sqrt{C_\lambda\,\beta_T\,\gamma_T}}{c_g\,\epsU_m}\right)^2.
\]
Consequently, the following condition can enforce the first inequality:
\[
s_m\ \le\ \left\lceil \left(\frac{2C_v\Lbar\sqrt{C_\lambda\,\beta_T\,\gamma_T}}{c_g\,\epsU_m}\right)^2\right\rceil.
\]
Since $\epsU_m=\epsU_0 2^{-m}$, this yields
\[
s_m\ \le\ \left\lceil
\left(\frac{2C_v\Lbar\sqrt{C_\lambda\,\beta_T\,\gamma_T}}{c_g\,\epsU_0}\right)^2\,4^m
\right\rceil
\ =\ \lceil A_T\,4^m\rceil,
\]
where $A_T:=\left(\frac{2C_v\,\Lbar\sqrt{C_\lambda\,\beta_T\,\gamma_T}}{c_g\,\epsU_0}\right)^2$.

\paragraph{Step 2: once level $m$ is active, it must advance within $\tau_m$ rounds.}
Let $t_m$ be the first global round at which the algorithm's current level becomes $m$.
Up to and including level $m$, the algorithm has instantiated at most $K_m:=1+J(m{+}1)$ tasks deterministically
(one batch of size $J$ per level, including level $0$).

Consider the $\tau_m:=K_m s_m$ rounds $\{t_m,t_m+1,\dots,t_m+\tau_m-1\}$ during which the current level remains $m$.
Exactly one task is evaluated per round, so by the pigeonhole principle some instantiated task $\tilde i$
is evaluated at least $s_m$ times within these $\tau_m$ rounds.
At the moment $\tilde i$ reaches local count $s_m$, the envelope-width bound on the event of
Theorem~\ref{thm:value_envelopes_summary} implies
\[
w^{(\tilde i)}_{s_m} \ \le\ C_v\Lbar\,\epsfT_{s_m}\ \le\ c_g\,\epsU_m.
\]
Hence at some round $t\le t_m+\tau_m$ we have $\underline{w}_t:=\min_{j\in\mathcal I_t} w^{(j)}_t \le c_g\epsU_m$.
By the anchor definition \eqref{eq:anchor}, $\overline{\epsU}_t=\max\{c_g\epsU_m,\ \underline{w}_t\}=c_g\epsU_m$,
so the anchor $a_t$ is chosen from a nonempty set of tasks satisfying $w^{(j)}_{s_t^{(j)}} \le \overline{\epsU}_t$.
Therefore the anchor width satisfies
\[
w_t=w^{(a_t)}_{s_t^{(a_t)}} \ \le\ \overline{\epsU}_t\ =\ c_g\epsU_m,
\]
and the scheduler condition in Algorithm~\ref{alg:task-ucb} triggers, incrementing the level to $m{+}1$.
Thus, level $m{+}1$ is introduced by time at most $t_m+\tau_m$.

\paragraph{Step 3: summing $\tau_m$ up to $\overline m_T$ fits within the horizon.}
By iterating the previous argument, level $M$ is introduced by time at most
\[
\sum_{m=0}^{M-1}\tau_m.
\]
We now bound this sum for $M=\overline m_T$.

Using $s_m\le \lceil A_T4^m\rceil\le A_T4^m+1$ and $N_m=J(m{+}1)$,
\[
\sum_{m=0}^{M-1}\tau_m
\ \le\
J\sum_{m=0}^{M-1}(m{+}1)\,(A_T4^m+1)
=
J A_T\sum_{m=0}^{M-1}(m{+}1)4^m
\;+\;
J\sum_{m=0}^{M-1}(m{+}1).
\]
For the first sum, use $(m{+}1)\le M$ and $\sum_{m=0}^{M-1}4^m\le \frac{4^M}{3}$:
\[
\sum_{m=0}^{M-1}(m{+}1)4^m
\ \le\
M\sum_{m=0}^{M-1}4^m
\ \le\
\frac{M\,4^M}{3}.
\]
For the second sum, $\sum_{m=0}^{M-1}(m{+}1)=\frac{M(M{+}1)}{2}\le M^2$.
Thus
\[
\sum_{m=0}^{M-1}\tau_m
\ \le\
\frac{J A_T M\,4^M}{3}
\;+\;
J M^2.
\]

Now set $M=\overline m_T$.
By definition of $\overline m_T$ we have $4^M \le \frac{T}{A_T(\log(eT))^2}$, hence
\[
\frac{J A_T M\,4^M}{3}
\ \le\
\frac{J M}{3(\log(eT))^2}\,T.
\]
If $J$ is an absolute constant (or more generally $J\le \log(eT)$), then since $M=\mathcal O(\log T)$ we have
$\frac{J M}{(\log(eT))^2}=o(1)$, and for all $T$ large enough (or by absorbing finitely many small $T$ into constants)
\[
\sum_{m=0}^{M-1}\tau_m \ \le\ T.
\]
Therefore every level $m\le \overline m_T$ is introduced by time $T$.

\paragraph{Step 4: concluding $m^\star$ and the order of $\overline m_T$.}
If $m^\star\le \overline m_T$ (Condition~\ref{cond:bo_achievable}), then level $m^\star$ is introduced by time $T$.
Finally,
\[
\overline m_T
=
\left\lfloor \tfrac12 \log_2\!\Big(\tfrac{T}{A_T(\log(eT))^2}\Big)\right\rfloor_+
=
\mathcal O(\log T).
\]
\end{proof}

\subsection{Bayesian in-context learning interpretation of Condition~\ref{cond:stoch_coverage_adaptive}}\label{app:ICL}

\subsubsection{Known results}
\paragraph{Setup.}
\citep{wakayama2025context} analyze in-context learning in a \emph{finite mixture of task types} (their Def.~1) and prove (a) a Bayes-risk decomposition (Thm.~1), (b) a finite-sample Bayes-gap bound for a uniform-attention Transformer (Thm.~2), and (c) exponential \emph{task-type identification} in mixtures (Thm.~3).

\begin{theorem}[Bayes gap bound; Theorem~2 in \citep{wakayama2025context}]\label{thm:ws-bayes-gap}
(Informal restatement.) Under the prompt-generating process (Def.~1) and boundedness/independence assumptions (Assumptions~1--2), and assuming the Bayes predictor is H\"older smooth, the ERM Transformer \(M_{\hat\theta}\) satisfies an upper bound of the form
\[
\mathbb{E}\, R_{\mathrm{BG}}(M_{\hat\theta})
\ \lesssim\
d_\text{feature}^{-2\alpha/d_{\mathrm{eff}}}
\;+\;
\widetilde O\!\left(\frac{d_\text{feature}}{pN} + \frac{1}{N}\right),
\]
where \(p\) is the pretraining context length, \(N\) is the number of pretraining
prompts, and \(d_\text{feature}\) is the learned feature dimension.
\end{theorem}

\begin{assumption}[task-type identification conditions;  Theorem~3 in \citep{wakayama2025context}]
\label{as:ws-ident}
Fix a (true) task-type index \(i^\star\in\{1,\dots,T\}\).
For each wrong type \(j\neq i^\star\), define the predictive log-likelihood ratio
increment
\[
Z_{j,t}\ :=\ \log \frac{p_j(y_t\mid x_t, D_{t-1})}{p_{i^\star}(y_t\mid x_t, D_{t-1})}.
\]
Assume there exist constants \(D_j>0\) and \((\nu_j,b_j)\) such that, under the
true type \(i^\star\) and for all \(t\),
\[
\mathbb{E}\!\left[ Z_{j,t}\mid \mathcal G_{t-1}, I=i^\star \right]\ \le\ -D_j,
\qquad
\mathbb{E}\!\left[\exp\{\lambda(Z_{j,t}+D_j)\}\mid \mathcal G_{t-1}, I=i^\star\right]
\ \le\ \exp\!\left(\frac{\lambda^2\nu_j^2}{2}\right)
\]
for all \(|\lambda|\le 1/b_j\), where \(\mathcal G_{t-1}\) is the natural filtration.
Define
\[
D_{\min} := \min_{j\neq i^\star} D_j > 0,
\qquad
C := \min_{j\neq i^\star}\frac{D_j^2}{8(\nu_j^2 + b_j D_j/2)} > 0.
\]
\end{assumption}

\begin{lemma}[Posterior mass on the true type concentrates exponentially; Theorem~3 in \citep{wakayama2025context}]
\label{lem:ws-posterior-mass}
Let \(\pi_i(D_k):=\Pr(I=i\mid D_k)\) be the Bayes posterior over the type index.
Under Assumption~\ref{as:ws-ident}, for every \(k\ge 1\),
\begin{align}
\mathbb{E}\bigl[1-\pi_{i^\star}(D_k)\mid I=i^\star\bigr]
&\le
\frac{1-\alpha_{i^\star}}{\alpha_{i^\star}}e^{-D_{\min}k/2}
\;+\;
(T-1)e^{-Ck},
\label{eq:ws-exp-mean}
\\
\Pr\!\left(
\pi_{i^\star}(D_k)\ \ge\
\frac{1}{1+\frac{1-\alpha_{i^\star}}{\alpha_{i^\star}}e^{-D_{\min}k/2}}
\ \middle|\ I=i^\star
\right)
&\ge\ 1-(T-1)e^{-Ck}.
\label{eq:ws-exp-hp}
\end{align}
\end{lemma}

\subsubsection{Refinement model}
\paragraph{Task generation adjustment.}
At each refinement level \(m\), interpret the generator's behavior as selecting among finitely many \emph{refinement types} (or modes) \(I^{(m)}\in\{1,\dots,T_m\}\)
based on in-context evidence (utility-improvement history).

\begin{assumption}[Refinement types form a finite mixture; adapted from Def.~1 in \citep{wakayama2025context}]
\label{as:new-mixture}
Fix a time \(t\) and refinement level \(m\).
Conditioned on the event \(w_t \le c_g \epsU_{m-1}\), there is a latent
refinement type \(I^{(m)}\in\{1,\dots,T_m\}\) with prior weights
\(\alpha^{(m)}_i>0\) (\(\sum_i \alpha^{(m)}_i=1\)).
The prompt contains \(k\) in-context examples
\[
\mathcal D_{t,k}^{(m)}=\{s_1,\dots,s_k\},
\]
where each \(s_r\) encodes (anchor \(a_t\), level \(m\), and feedback such as
utility improvements / preferences).
The conditional predictive distributions \(p_i^{(m)}(\cdot\mid \mathcal D_{t,r-1}^{(m)})\)
satisfy the identifiability conditions (Assumption~\ref{as:ws-ident})
with constants \(D_{\min}^{(m)}\), \(C^{(m)}\) and mixture size \(T_m\).
\end{assumption}

\begin{assumption}[Correctness of the true refinement type implies \(\epsU_m\)-success]
\label{as:new-true-type-success}
Let \(\mathcal I_m^\star(a_t)\) denote the set of \(\epsU_m\)-successful refinements.
There exists a distinguished (true) refinement type \(i_m^\star=i_m^\star(a_t)\) such that,
conditioned on \(w_t\le c_g\epsU_{m-1}\), the data \(\mathcal D_{t,k}^{(m)}\) are
generated under \(I^{(m)}=i_m^\star\), and moreover a single refinement proposal drawn
\emph{given the true type} succeeds with probability at least \(1-\beta_m\):
\[
\Pr\!\left(\iota\in\mathcal I_m^\star(a_t)\ \middle|\ \widehat I=i_m^\star,\ a_t,m,\mathcal D_{t,k}^{(m)}\right)
\ \ge\ 1-\beta_m,
\qquad \beta_m\in[0,1).
\]
\end{assumption}

\begin{assumption}[LLM approximately samples the Bayes posterior over refinement types]\label{as:new-approx-posterior}
Let \(q_{m,k}(\cdot\mid a_t,\mathcal D_{t,k}^{(m)})\) be the LLM-induced distribution over
refinement types \(\widehat I\in\{1,\dots,T_m\}\) (e.g., induced by decoding / sampling).
Let \(\pi^{(m)}(\cdot\mid a_t,\mathcal D_{t,k}^{(m)})\) be the Bayes posterior over types.
Assume a (uniform) total-variation approximation error
\[
\mathrm{TV}\!\left(q_{m,k}(\cdot\mid a_t,\mathcal D_{t,k}^{(m)}),\ \pi^{(m)}(\cdot\mid a_t,\mathcal D_{t,k}^{(m)})\right)
\ \le\ \varepsilon_{\mathrm{ICL}}(m,k).
\]
\end{assumption}

\subsubsection{Main theorem: Condition 4.4 with explicit $\delta_+(m,k)$}

\begin{theorem}[Bayesian ICL \(\Rightarrow\) refinement probability]
\label{thm:bayes-icl-refinement}
Fix \(t,m\) and suppose \(w_t\le c_g\epsU_{m-1}\).
Under Assumptions~\ref{as:new-mixture}--\ref{as:new-approx-posterior},
a single call \(\iota\sim \Gen(a_t,m,J{=}1)\) satisfies
\[
\Pr\!\left(\iota\in\mathcal I_m^\star(a_t)\ \middle|\ w_t\le c_g\epsU_{m-1}\right)
\ \ge\ \delta_+(m,k),
\]
where the (level-dependent) lower bound can be chosen as
\begin{align}
\delta_+(m,k)
\;:=\;
\rho_m(\epsU_m)\Bigg[
1
-
\underbrace{\frac{1-\alpha^{(m)}_{i_m^\star}}{\alpha^{(m)}_{i_m^\star}}e^{-D_{\min}^{(m)}k/2}
-
(T_m-1)e^{-C^{(m)}k}}_{=:~\eta_{m}(k)}
\Bigg]_+
\;-\;
\rho_m(\epsU_m)\,\varepsilon_{\mathrm{ICL}}(m,k),
\label{eq:delta-plus-mk}
\end{align}
with \([u]_+:=\max\{u,0\}\).
In particular, for any fixed \(m\), \(\delta_+(m,k)\) improves \emph{exponentially fast} in \(k\)
up to the approximation error \(\varepsilon_{\mathrm{ICL}}(m,k)\).
\end{theorem}

\begin{proof}
Write the success event as \(S := \{\iota\in\mathcal I_m^\star(a_t)\}\).
Let \(\widehat I\in\{1,\dots,T_m\}\) denote the sampled refinement type underlying the proposal.
By the law of total probability and Assumption~\ref{as:new-true-type-success},
\[
\Pr(S\mid a_t,m,\mathcal D_{t,k}^{(m)})
\ \ge\
\Pr(S\mid \widehat I=i_m^\star,a_t,m,\mathcal D_{t,k}^{(m)})\,
\Pr(\widehat I=i_m^\star\mid a_t,m,\mathcal D_{t,k}^{(m)})
\ \ge\
(1-\beta_m)\, q_{m,k}(i_m^\star\mid a_t,\mathcal D_{t,k}^{(m)}).
\]
By total-variation control (Assumption~\ref{as:new-approx-posterior}),
\[
q_{m,k}(i_m^\star\mid a_t,\mathcal D_{t,k}^{(m)})
\ \ge\
\pi^{(m)}(i_m^\star\mid a_t,\mathcal D_{t,k}^{(m)})
\;-\;
\varepsilon_{\mathrm{ICL}}(m,k).
\]
Taking expectations over \(\mathcal D_{t,k}^{(m)}\) under the true type
\(I^{(m)}=i_m^\star\), we obtain
\[
\Pr(S\mid w_t\le c_g\epsU_{m-1})
\ \ge\
(1-\beta_m)\Big(\mathbb{E}[\pi^{(m)}(i_m^\star\mid a_t,\mathcal D_{t,k}^{(m)})] - \varepsilon_{\mathrm{ICL}}(m,k)\Big).
\]
Finally, Lemma~\ref{lem:ws-posterior-mass} (applied at level \(m\) with \(T=T_m\),
\(\alpha=\alpha^{(m)}\), and constants \(D_{\min}^{(m)},C^{(m)}\))
yields the exponential lower bound on
\(\mathbb{E}[\pi^{(m)}(i_m^\star\mid a_t,\mathcal D_{t,k}^{(m)})]\) via
\(\mathbb{E}[\pi]=1-\mathbb{E}[1-\pi]\) and~\eqref{eq:ws-exp-mean}.
Combining these inequalities gives~\eqref{eq:delta-plus-mk}.
\end{proof}

\paragraph{Shorthand used in the main text.}
Assumption~\ref{as:new-true-type-success} lower bounds success \emph{conditional on the true refinement type}.
To match the main-text discussion, define the within-type success probability
\begin{equation}
\label{eq:def_rho_m}
\rho_m(\epsU_m)\ :=\ 1-\beta_m \ \in (0,1].
\end{equation}
Moreover, we will use a simplified (large-$k$) Bayes-approximation error proxy
\begin{equation}
\label{eq:def_epsICL}
\epsilon_{\mathrm{ICL}}
\ :=\
\sup_{m\le \overline m_T}\ \sup_{k\ge k_0}\ \varepsilon_{\mathrm{ICL}}(m,k),
\end{equation}
for some reference in-context length $k_0$ at which the exponential type-identification term becomes negligible
(cf.\ Lemma~\ref{lem:ws-posterior-mass}).

\begin{corollary}[Large-$k$ simplification (main-text form)]
\label{cor:delta_large_k_main_text_form}
Work under the assumptions of Theorem~\ref{thm:bayes-icl-refinement}.
Fix any rung $m$ and suppose the in-context length $k$ is large enough that the WS type-identification term
satisfies $\eta_m(k)\le 1/2$ (e.g., by Lemma~\ref{lem:ws-posterior-mass}).
Then the single-draw refinement success probability satisfies
\begin{equation}
\label{eq:delta_main_text_form}
\delta_+(m,k)
\ \ge\
\frac{1}{2}\rho_m(\epsU_m)\;-\;\rho_m(\epsU_m)\,\varepsilon_{\mathrm{ICL}}(m,k)
\ \ge\
\frac{1}{2}\rho_m(\epsU_m)\;-\;\epsilon_{\mathrm{ICL}}.
\end{equation}
In particular, up to constant factors (absorbed by $\gtrsim$), this matches the main-text summary
$\delta_+\gtrsim \rho_m(\epsU_m)-\epsilon_{\mathrm{ICL}}$.
\end{corollary}

\begin{proof}
By Theorem~\ref{thm:bayes-icl-refinement} and the definition of $\eta_m(k)$,
\[
\delta_+(m,k)\ \ge\ \rho_m(\epsU_m)\bigl(1-\eta_m(k)\bigr)\;-\;\rho_m(\epsU_m)\varepsilon_{\mathrm{ICL}}(m,k).
\]
If $\eta_m(k)\le 1/2$, then $\rho_m(\epsU_m)(1-\eta_m(k))\ge \rho_m(\epsU_m)/2$, giving the first inequality.
The second uses $\varepsilon_{\mathrm{ICL}}(m,k)\le \epsilon_{\mathrm{ICL}}$ for $k\ge k_0$ by \eqref{eq:def_epsICL}.
\end{proof}

\begin{corollary}[Batching \(J\) samples at level \(m\)]
\label{cor:batching}
Under the conditions of Theorem~\ref{thm:bayes-icl-refinement}, if you draw
\(J\) i.i.d.\ refinements \(\iota^{(1)},\dots,\iota^{(J)}\) from \(\Gen(a_t,m,J{=}1)\),
then
\[
\Pr\!\left(\exists j\in\{1,\dots,J\}: \iota^{(j)}\in\mathcal I_m^\star(a_t)\ \middle|\ w_t\le c_g\epsU_{m-1}\right)
\ \ge\
1-(1-\delta_+(m,k))^J
\ \ge\
1-\exp\{-J\,\delta_+(m,k)\}.
\]
\end{corollary}

\begin{remark}[How to get a clean level-dependent \(\delta_+(m)\)]
If your Appendix~B uses a per-level constant \(\delta_+(m)\) (no explicit \(k\)),
you can fix a per-level prompt budget \(k=k_m\) and define
\[
\delta_+(m)\ :=\ \delta_+(m,k_m)
\]
using~\eqref{eq:delta-plus-mk}. The exponential terms show that
\(k_m = O(\log(1/\eta))\) is sufficient to make the mixture-identification loss
\(\le \eta\) (up to \(\varepsilon_{\mathrm{ICL}}(m,k_m)\)).
\end{remark}

\subsubsection{Near-optimal task generation: batch-size selection}
\label{app:unknown_delta_plus}

This subsection collects the batch-success guarantees used when Algorithm~\ref{alg:task-ucb} introduces a new rung and calls the generator.
We start with a one-shot i.i.d.\ batching bound that holds for any (possibly unknown) generator success probability~$\delta_+$, and then
specialize it to (i) a $\delta_+$-calibrated choice that yields a uniform success guarantee across all introduced rungs, and (ii) a
$\delta_+$-free choice whose success probability depends explicitly on~$\delta_+$.

\paragraph{Single-batch success (arbitrary $J$, possibly unknown $\delta_+$).}
\begin{lemma}[Batch success with arbitrary size]
\label{lem:batch_success_unknown_delta}
Assume Condition~\ref{cond:stoch_coverage_adaptive} holds with some (unknown) $\delta_+\in(0,1]$.
Consider any generator call at rung $m\ge 1$ for which the anchor task $a_t$ satisfies the precondition of
Condition~\ref{cond:stoch_coverage_adaptive} at the moment of generation (e.g., $a_t$ is $\epsU_m$-solved;
Definition~\ref{def:gamma_oracle}).
Let $J\ge 1$ and draw $i_1,\dots,i_J\stackrel{\mathrm{i.i.d.}}{\sim}\Gen(a_t,m)$.
Then
\[
\Pr\Big(\forall k\in[J]:\ i_k\notin \mathcal I^\star(\epsU_m)\Big)
\le
(1-\delta_+)^{J}
\le
\exp(-\delta_+ J).
\]
Equivalently, the batch contains an $\epsU_m$-optimal task with probability at least $1-\exp(-\delta_+ J)$.
\end{lemma}

\begin{proof}
Under the stated precondition, Condition~\ref{cond:stoch_coverage_adaptive} implies that a single draw fails with probability at most
$1-\delta_+$.
Independence of the $J$ draws yields
\[
\Pr\Big(\forall k\in[J]:\ i_k\notin \mathcal I^\star(\epsU_m)\Big)\le (1-\delta_+)^J.
\]
Finally, $(1-\delta_+)^J\le \exp(-\delta_+J)$ is standard.
\end{proof}

\paragraph{Uniform success across introduced rungs (known $\delta_+$).}
\begin{lemma}[Near-optimal task generation]
\label{lem:batch_success_adaptive}
Assume Condition~\ref{cond:stoch_coverage_adaptive}.
Fix $\delta_-\in(0,1)$.
Suppose that whenever Algorithm~\ref{alg:task-ucb} introduces a new rung $m$,
it draws an i.i.d.\ batch of size $J$ from $\Gen(a,m)$ where $a$ is the anchor parent at that call.

Choose
\[
J
:=\left\lceil
\frac{\log\!\bigl((\overline m_T+1)/\delta_-\bigr)}{\log\!\bigl(1/(1-\delta_+)\bigr)}
\right\rceil
\ \le\
\left\lceil
\frac{1}{\delta_+}\log\frac{\overline m_T+1}{\delta_-}
\right\rceil.
\]
Then with probability at least $1-\delta_-$, simultaneously for every rung $m\in\{1,\dots,\overline m_T\}$
that is introduced, the corresponding batch contains at least one $\epsU_m$-optimal task:
there exists $i_m$ in that batch such that $U^\star - U^{(i_m)} \le \epsU_m$.
\end{lemma}

\begin{proof}
Fix a rung $m\in\{1,\dots,\overline m_T\}$ and consider the moment it is introduced.
By Algorithm~\ref{alg:task-ucb}, $\Gen(a,m)$ is called only after the scheduler condition holds at the previous rung, so the
anchor $a$ satisfies the generation-gating condition required by Condition~\ref{cond:stoch_coverage_adaptive} at the moment of generation.
Applying Lemma~\ref{lem:batch_success_unknown_delta} at rung $m$ gives that the failure probability for this batch is at most
$(1-\delta_+)^J$.

By the choice of $J$, $(1-\delta_+)^J\le \delta_- /(\overline m_T+1)$.
A union bound over the at most $(\overline m_T+1)$ rung-introductions gives total failure probability at most $\delta_-$.
\end{proof}

\begin{lemma}[Oracle-gap bound for the best instantiated value]
\label{lem:Vtmax_gap}
Work on the event of Lemma~\ref{lem:batch_success_adaptive}.
Define, at each global round $t$,
\[
U_t^{\max} \;:=\; \max_{i\in\mathcal I_t} U^{(i)}.
\]
Let $m_t$ denote the current ladder level in Algorithm~\ref{alg:task-ucb} at round $t$.
Then for every $t$,
\[
U^\star - U_t^{\max} \ \le\ \epsU_{\min\{m_t,\,m^\star\}}.
\]
In particular, before reaching level $m^\star$ we have $U^\star - U_t^{\max}\le \epsU_{m_t}$,
and after level $m^\star$ is introduced we have $U^\star - U_t^{\max}\le \epsU_{m^\star}$ for all subsequent rounds.
\end{lemma}

\begin{proof}
Fix $t$ and let $m=m_t$.
If $m\le m^\star$, then level $m$ has already been introduced by time $t$.
On the event of Lemma~\ref{lem:batch_success_adaptive}, at the moment level $m$ was introduced,
its generation batch contained some task $i_m$ with $U^\star - U^{(i_m)}\le \epsU_m$.
Since tasks are only added (never removed), we have $i_m\in\mathcal I_t$, hence
$U_t^{\max}\ge U^{(i_m)}\ge U^\star-\epsU_m$, i.e.\ $U^\star - U_t^{\max}\le \epsU_m$.

If $m>m^\star$, then level $m^\star$ has also been introduced by time $t$, so the same argument gives
$U^\star - U_t^{\max}\le \epsU_{m^\star}$.
Combining the two cases yields the claim with $\epsU_{\min\{m_t,m^\star\}}$.
\end{proof}

\paragraph{Choosing $J$ without knowing $\delta_+$.}
\begin{theorem}[Choosing $J$ without knowing $\delta_+$]
\label{thm:unknown_delta_plus_choice}
Fix any target failure probability $\delta_-\in(0,1)$ and choose a batch size schedule that does not depend on $\delta_+$,
for example
\[
J_T \;:=\; \left\lceil \log^2\!\left(\frac{2T}{\delta_-}\right)\right\rceil.
\]
Then, whenever a rung $m$ is introduced and the parent used for generation satisfies the precondition of
Condition~\ref{cond:stoch_coverage_adaptive} at that call (e.g., it is $\epsU_m$-solved; Definition~\ref{def:gamma_oracle}),
Lemma~\ref{lem:batch_success_unknown_delta} implies the probability that the batch contains an $\epsU_m$-optimal task is at least
\[
1-\exp(-\delta_+ J_T).
\]
In particular, at the moment level $m^\star$ is introduced, the probability of generating an $\epsU_{m^\star}$-optimal task is at least
$1-\exp(-\delta_+ J_T)$, and whenever $\delta_+ \ge \log(1/\delta_-)/J_T$, the failure probability is at most $\delta_-$.
\end{theorem}

\begin{remark}[Effect on regret via $N_T$]
Using a $\delta_+$-free upper bound such as $J_T$ changes only the number of instantiated tasks:
$N_T\le J_T(\lfloor\log_2 g_{m_T}\rfloor+1)$ deterministically.
The regret bound of Theorem~\ref{thm:main_regret_adaptive} continues to hold on the event that an $\epsU_{m^\star}$-optimal task is instantiated,
and the probability of this event is controlled by $1-\exp(-\delta_+ J_T)$.
\end{remark}

\begin{corollary}[Regret overhead with $\delta_+$-free batching]
\label{cor:unknown_delta_plus_logT}
Assume Condition~\ref{cond:stoch_coverage_adaptive} holds for some unknown $\delta_+\in(0,1]$,
and use
\[
J_T \;:=\; \left\lceil \left(\log\!\frac{2T}{\delta_-}\right)^2\right\rceil
\]
whenever a new rung is introduced.  If $\overline m(T)$ is fixed, then deterministically
$N_T \le (\overline m(T)+1)J_T$.
Moreover, on the joint event that (i) the envelope confidence event of Theorem~\ref{thm:value_envelopes_summary} holds and
(ii) every introduced batch up to level $m^\star$ contains an $\epsU_m$-optimal task, Algorithm~\ref{alg:task-ucb} satisfies
\[
\Regval(T)\ \le\ \tilde{\mathcal O}(\log T)\,R^\star_T .
\]
The uniform batch-success event in (ii) has probability at least
\[
1-(m^\star+1)\exp(-\delta_+J_T),
\]
by Lemma~\ref{lem:batch_success_unknown_delta} and a union bound over all introduced levels up to $m^\star$.
\end{corollary}

\begin{proof}
Since GSR introduces each level at most once and generates at most $J_T$ tasks per introduced level,
we have $N_T \le (\overline{m}(T)+1)J_T$ deterministically.
On the envelope event and the oracle-instantiation event, the proof of Theorem~\ref{thm:main_regret_adaptive}
goes through with the same regret-vs.-$N_T$ dependence, yielding
\[
\Regval(T)\ \le\ C\bigl(1+\sqrt{N_T}\bigr)\,R^\star_T
\ \le\ C\Bigl(1+\sqrt{\overline{m}(T)+1}\,\sqrt{J_T}\Bigr)\,R^\star_T.
\]
Let $a:=\log(2T/\delta_-)$ so that $J_T=\lceil a^2\rceil$.
Then $\sqrt{J_T}\le \sqrt{a^2+1}\le a+1$, hence
$1+\sqrt{\overline{m}(T)+1}\sqrt{J_T}=\tilde{\mathcal O}(\log(2T/\delta_-))=\tilde{\mathcal O}(\log T)$ for fixed $\delta_-$.
\end{proof}

\subsubsection{Decaying generator success across refinement levels: a logit-linear model}\label{app:decaying_delta_plus}
In Condition~\ref{cond:stoch_coverage_adaptive} we lower bounded the success probability of generating an $\epsU_m$-optimal task by a constant $\delta_+$.
Empirically (Fig.\ref{fig:controllability}(b)), the success probability may decrease as the refinement level $m$ increases.
This subsection analyzes a simple parametric model for this decay and its implication on the overhead factor via $N_T$.

\paragraph{Level-dependent generation condition.}
We replace Condition~\ref{cond:stoch_coverage_adaptive} by the following level-dependent version.

\begin{condition}[{$\epsU_m$-optimal task generation with level-dependent probability}]
\label{cond:delta_m}
Under Condition~\ref{cond:bo_achievable}, assume there exists a function $\delta_+(\cdot):( \mathbb{Z}_{\ge 1}\to (0,1])$ such that for every
$rung~m\ge 1$, every global time $t$, and every anchor task $a_t$ satisfying the generation-gating condition
$w_t \le c_g\epsU_{m-1}$ at the moment of generation,
a single i.i.d.\ draw $i\sim\mathrm{Gen}(a_t,m,J{=}1)$ hits an $\epsU_m$-optimal task with probability at least $\delta_+(m)$:
\[
\Pr\!\Big(i \in \mathcal I^\star(\epsU_m)\,\Big|\, w_t \le c_g\epsU_{m-1}\Big)\ \ge\ \delta_+(m).
\]
\end{condition}

\paragraph{Logit-linear decay model.}
Motivated by boundedness and the approximately smooth monotone decay in Fig.\ref{fig:controllability}(b), we consider the log-odds to be linear:
\begin{assumption}[Logit-linear decay]
\label{ass:logit_linear}
There exist $\alpha\in\mathbb{R}$ and $\beta\ge 0$ such that for all $m\ge 1$,
\[
\mathrm{logit}(\delta_+(m)) = \alpha - \beta m,
\qquad\text{i.e.}\qquad
\delta_+(m)=\sigma(\alpha-\beta m),
\]
where $\sigma(x)=1/(1+e^{-x})$ and $\mathrm{logit}(p)=\log(p/(1-p))$.
\end{assumption}

\paragraph{Main result.}
We now show how to pick per-rung batch sizes to maintain a high-probability generation-success event, and how the resulting
$N_T$ scales under Assumption~\ref{ass:logit_linear}.

\begin{proposition}[Regret overhead under logit-linear $\delta_+(m)$]
\label{prop:logit_delta_overhead}
Assume the conditions of Theorem~\ref{thm:main_regret_adaptive}, except that Condition~\ref{cond:stoch_coverage_adaptive} is replaced by Condition~\ref{cond:delta_m}.
Fix a target failure probability $\delta_- \in (0,1)$ and, whenever rung $m$ is introduced, draw an i.i.d.\ batch of size
\[
J_m \ :=\ \left\lceil \frac{\log\!\big((m_T+1)/\delta_-\big)}{\log\!\big(1/(1-\delta_+(m))\big)} \right\rceil
\ \le\
\left\lceil \frac{1}{\delta_+(m)}\log\!\Big(\frac{m_T+1}{\delta_-}\Big)\right\rceil ,
\]
where the inequality uses $-\log(1-x)\ge x$ for $x\in(0,1)$.
Let $N_T$ be the total number of instantiated tasks by time $T$.

Then:

\begin{enumerate}
\item \textbf{(Generation success across rungs).}
With probability at least $1-\delta_-$, simultaneously for every rung $m\in\{1,\dots,m_T\}$ that is introduced,
the corresponding batch contains at least one $\epsU_m$-optimal task.

\item \textbf{(Regret bound with variable batch sizes).}
On the joint event of Theorem~\ref{thm:value_envelopes_summary}’s envelope confidence and the generation-success event in (1),
the regret bound of Theorem~\ref{thm:main_regret_adaptive} continues to hold with the (random) $N_T$ induced by $\{J_m\}$:
\[
R_{\mathrm{val}}(T)\ \le\ C\bigl(1+\sqrt{N_T}\bigr)\frac{L}{L^\star}\,R^\star_T,
\]
for the same oracle scale $R^\star_T=L^\star\sum_{s=1}^T\epsfT_s$.

\item \textbf{(Scaling under logit-linear decay).}
Under Assumption~\ref{ass:logit_linear} with $\beta>0$, $\delta_+(m)$ is nonincreasing in $m$ and hence $J_m$ is
nondecreasing. Therefore,
\[
N_T \ \le\ 1 + \sum_{m=0}^{m_T} J_m \ \le\ 1+(m_T+1)J_{m_T}
\ =\ \tilde O\!\left(\frac{m_T}{\delta_+(m_T)}\right).
\]
Moreover, since $\delta_+(m)^{-1}=1+e^{\beta m-\alpha}$, we have
$\delta_+(m_T)^{-1}=\tilde O\!\big(e^{\beta m_T}\big)$ and thus
\[
\sqrt{N_T} \ =\ \tilde O\!\big(e^{\beta m_T/2}\big).
\]
If $m_T$ is chosen as in Lemma~\ref{lem:max_depth_reachability} (so $m_T=\Theta(\log T)$), then
$e^{\beta m_T}=\tilde O\!\big(T^{\beta/(2\ln 2)}\big)$ and hence
\[
\sqrt{N_T} \ =\ \tilde O\!\big(T^{\beta/(4\ln 2)}\big),
\qquad\text{so}\qquad
\frac{R_{\mathrm{val}}(T)}{R^\star_T}\ =\ \tilde O\!\big(T^{\beta/(4\ln 2)}\big)
\quad (\text{when }L=\Theta(L^\star)).
\]
In particular, if $R^\star_T=\tilde O(\sqrt{T})$ (as for standard GP-UCB regimes),
then $R_{\mathrm{val}}(T)=\tilde O\!\big(T^{1/2+\beta/(4\ln 2)}\big)$ is sublinear whenever $\beta<2\ln 2$.
\end{enumerate}
\end{proposition}

\begin{proof}
(1) Fix a rung $m\in\{1,\dots,m_T\}$ and consider the moment it is introduced.
Condition~\ref{cond:delta_m} implies that each i.i.d.\ draw fails to be $\epsU_m$-optimal with probability at most
$1-\delta_+(m)$. Hence the probability that all $J_m$ draws fail is at most
\[
(1-\delta_+(m))^{J_m} \ \le\ \exp\big(-\delta_+(m)J_m\big).
\]
By the choice of $J_m$, we have $\exp(-\delta_+(m)J_m)\le \delta_-/(m_T+1)$.
Applying a union bound over at most $m_T$ rung introductions gives total failure probability at most $\delta_-$.

(2) The proof of Theorem~\ref{thm:main_regret_adaptive} uses the generator only through the event that each introduced rung $m$ contains a
$\epsU_m$-optimal task (cf.\ Lemmas~\ref{lem:batch_success_adaptive} and~\ref{lem:Vtmax_gap}) and through the induced $N_T$.
On the event established in (1), the analogue of Lemma~B.11 holds verbatim:
for all $t$, $U^\star - U^{\max}_t \le \epsU_{\min\{m_t,m^\star\}}$.
The remainder of the regret decomposition in Appendix~\ref{app:proof_main_regret_adaptive} (selection and within-task terms) is unchanged, except that
the bound is expressed in terms of the realized $N_T$.
This yields the displayed bound.

(3) Under Assumption~\ref{ass:logit_linear} with $\beta>0$, $\delta_+(m)$ is nonincreasing, so $J_m$ is nondecreasing and
$\sum_{m\le m_T}J_m \le (m_T+1)J_{m_T}$.
Using $\delta_+(m)^{-1}=1+e^{\beta m-\alpha}$ gives the stated $\tilde O(\cdot)$ scaling.
Finally, Lemma~\ref{lem:max_depth_reachability} implies $m_T=\Theta(\log T)$, so
$e^{\beta m_T}=\tilde O(T^{\beta/(2\ln 2)})$ and the conclusion follows.
\end{proof}

\section{Confidence envelopes of long-run values}
\label{app:sparse_utility}
\label{app:envelopes_full}
This appendix provides the construction and proof of Theorem~\ref{thm:value_envelopes_summary}. 

\subsection{Uniform sub-Gaussian confidence over utility calls}
\label{app:proof_hoeffding}
\begin{lemma}[Uniform confidence over BTL-calibrated utility calls]
\label{lem:hoeffding}
Fix $\delta_u\in(0,1)$ and let $\delta_\ell:=\delta_u/(\pi^2\ell^2)$.
Assume utility calls are implemented via the pairwise-comparison + transport construction, with $K_\ell$ fresh independent pairwise votes at call $\ell$.
Then with probability at least $1-\delta_u$, simultaneously for all utility calls $\ell\ge 1$,
\[
u^{(i_\ell)}(\overline{y}^{(i_\ell)}_{s_\ell})
\in
\bigl[\LCBu_\ell^{(i_\ell)},\UCBu_\ell^{(i_\ell)}\bigr],
\]
where $(\LCBu_\ell^{(i_\ell)},\UCBu_\ell^{(i_\ell)})$ is the transported interval computed at call $\ell$
(e.g., Eq.~\eqref{eq:transported_interval_app}).
\end{lemma}

\begin{proof}
Let $\mathcal E_\ell$ be the event that the Hoeffding CI for $p_\ell$ holds at call $\ell$.
By Hoeffding and the choice of $\delta_\ell$, we have $\Pr(\mathcal E_\ell)\ge 1-\delta_\ell$ and thus
\[
\Pr\Big(\bigcap_{\ell\ge 1}\mathcal E_\ell\Big)
\ \ge\
1-\sum_{\ell=1}^\infty \delta_\ell
\ \ge\
1-\delta_u.
\]
Work on $\cap_{\ell\ge 1}\mathcal E_\ell$.
The reference artifact $r$ has a trivial valid interval ($u(r)=1/2$).
Assume inductively that the anchor interval used at call $\ell$ is valid.
Then Lemma~\ref{lem:transported_interval_valid} implies the transported interval computed at call $\ell$
contains the true utility $u^{(i_\ell)}(\overline{y}_{s_\ell}^{(i_\ell)})$.
This closes the induction and proves simultaneous validity over all calls.
\end{proof}

\subsection{GP-UCB optimization gap: best-so-far envelope}
\label{app:proof_simple_regret}

For a task $i$, define the within-task optimization gap bound
\[
\epsf_s:=2\sqrt{\frac{C_\lambda\,\beta^{(i)}_s\,\gamma^{(i)}_s}{s}},
\qquad
C_\lambda:=\frac{2}{\log(1+\lambda^{-1})}.
\]

\begin{lemma}[Best-so-far optimization gap under GP-UCB]
\label{lem:simple_regret}
On the GP-UCB confidence event, for each task $i$ and all $s\ge 1$,
\[
f^{\star(i)} - \max_{1\le r\le s} f^{(i)}(x_r^{(i)}) \le \epsf_s.
\]
\end{lemma}

\begin{proof}
Fix a task $i$ and suppress superscripts.
Let $x_1,\dots,x_s$ be the queried points and let $R_s:=\sum_{t=1}^s (f^\star-f(x_t))$ be the cumulative regret.
Then
\[
\max_{1\le t\le s} f(x_t)\ \ge\ \frac{1}{s}\sum_{t=1}^s f(x_t)
= f^\star - \frac{R_s}{s},
\]
so $f^\star-\max_{t\le s}f(x_t)\le R_s/s$.
Apply Theorem~\ref{thm:gpucb-regret} at horizon $s$ to get
$R_s\le 2\sqrt{C_\lambda s \beta_s\gamma_s}$, hence
$f^\star-\max_{t\le s}f(x_t)\le 2\sqrt{C_\lambda\beta_s\gamma_s/s}=\epsf_s$.
\end{proof}

\subsection{Envelope construction}
Utility is queried at local times $s\in\mathcal S_u\subseteq\mathbb N$.
The dense choice in the main text corresponds to $\mathcal S_u=\mathbb N$ (so the most recent checkpoint is always $\nu_s^{(i)}=s$),
while the optional sparse schedule uses $\mathcal S_u=\{1,2,4,\dots\}$ (Appendix~\ref{app:sparse_utility}).

For each task $i$ and local time $s$, let
\[
\nu_s^{(i)} := \max\{r\le s:\ r\in\mathcal S_u\}.
\]
be the most recent checkpoint at or before $s$ (if no checkpoint has occurred yet, treat $\nu_s^{(i)}$ as undefined).

At a checkpoint $\nu=\nu_s^{(i)}$, the evaluator returns a valid interval
$u^{(i)}(\overline{y}_\nu^{(i)})\in[\LCBu_\nu^{(i)},\UCBu_\nu^{(i)}]$ on the event of Lemma~\ref{lem:hoeffding}.
For our main instantiation (pairwise BTL with moving anchors), $(\LCBu_\nu^{(i)},\UCBu_\nu^{(i)})$ is computed by the transported interval
construction.

Propagate them to the current incumbent using monotonicity + Lipschitzness (Assumption~\ref{ass:mono_lip}).
Since incumbents are nondecreasing ($\overline{y}_s^{(i)}\ge \overline{y}_\nu^{(i)}$), monotonicity implies
$u^{(i)}(\overline{y}_s^{(i)})\ge u^{(i)}(\overline{y}_\nu^{(i)})$, so the lower bound should never decrease between checkpoints:
\begin{align}
\UCBu_s^{(i)}(L)
&:= \min\{1,\ \UCBu_\nu^{(i)} + L\bigl(\overline{y}_s^{(i)}-\overline{y}_\nu^{(i)}\bigr)\}, \label{eq:lipschitz_ucb}\\
\LCBu_s^{(i)}(L)
&:= \LCBu_\nu^{(i)}. \label{eq:lipschitz_lcb}
\end{align}
(If no checkpoint exists yet, set $\UCBu_s^{(i)}(L)\equiv 1$ and $\LCBu_s^{(i)}(L)\equiv 0$.)

Now define the value envelopes:
\begin{align}
\LCBV^{(i)}_s(L)
&:= \LCBu^{(i)}_s(L), \label{eq:lcbv}\\
\UCBV^{(i)}_s(L)
&:= \min\{1,\ \UCBu^{(i)}_s(L) + L\epsf_s\}, \label{eq:ucbv}\\
w^{(i)}_s(L)
&:= \UCBV^{(i)}_s(L)-\LCBV^{(i)}_s(L). \label{eq:width}
\end{align}

\begin{lemma}[Validity of the envelopes]
\label{lem:envelope_valid}
Work on the joint event of Lemma~\ref{lem:hoeffding} and Theorem~\ref{thm:gp-ucb}.
Fix a task $i$ and $L\ge L^{(i)}$. Then for all $s\ge 1$,
\[
\LCBV^{(i)}_s(L)\ \le\ U^{(i)}\ \le\ \UCBV^{(i)}_s(L).
\]
\end{lemma}

\begin{proof}
Fix $i,s$ and let $\nu=\nu_s^{(i)}$.
On Lemma~\ref{lem:hoeffding}, we have $u^{(i)}(\overline{y}_\nu^{(i)})\in[\LCBu_\nu^{(i)},\UCBu_\nu^{(i)}]$.
Since incumbents are nondecreasing ($\overline{y}_s^{(i)}\ge \overline{y}_\nu^{(i)}$), monotonicity yields
\[
u^{(i)}(\overline{y}_s^{(i)})\ge u^{(i)}(\overline{y}_\nu^{(i)})\ge \LCBu_\nu^{(i)}=\LCBu_s^{(i)}(L).
\]
Moreover, by Lipschitzness,
\[
u^{(i)}(\overline{y}_s^{(i)})\le u^{(i)}(\overline{y}_\nu^{(i)}) + L\bigl(\overline{y}_s^{(i)}-\overline{y}_\nu^{(i)}\bigr)
\le \UCBu_\nu^{(i)} + L\bigl(\overline{y}_s^{(i)}-\overline{y}_\nu^{(i)}\bigr)
= \UCBu_s^{(i)}(L).
\]
Thus $\LCBV_s^{(i)}(L)=\LCBu_s^{(i)}(L)\le u^{(i)}(\overline{y}_s^{(i)})$.
Also $u^{(i)}(\overline{y}_s^{(i)})\le U^{(i)}$ by monotonicity and $\overline{y}_s^{(i)}\le f^{\star(i)}$, giving the lower bound on $U^{(i)}$.

For the upper bound, Lipschitzness gives
\[
U^{(i)}=u^{(i)}(f^{\star(i)})\le u^{(i)}(\overline{y}_s^{(i)}) + L\bigl(f^{\star(i)}-\overline{y}_s^{(i)}\bigr).
\]
On Lemma~\ref{lem:simple_regret}, $f^{\star(i)}-\overline{y}_s^{(i)}\le \epsf_s$, and on Lemma~\ref{lem:hoeffding},
$u^{(i)}(\overline{y}_s^{(i)})\le \UCBu_s^{(i)}(L)$, hence
$U^{(i)}\le \UCBu_s^{(i)}(L)+L\epsf_s=\UCBV_s^{(i)}(L)$.
\end{proof}

\subsection{Proof of Theorem~\ref{thm:value_envelopes_summary}}\label{proof:thm:value_envelopes_summary}
\begin{proof}[Proof of Theorem~\ref{thm:value_envelopes_summary}]
The envelope validity is exactly Lemma~\ref{lem:envelope_valid}.

For the width bound, fix task $i$ and time $s$ and let $\nu=\nu_s^{(i)}$ be the most recent checkpoint.
By construction (using \eqref{eq:lipschitz_ucb}--\eqref{eq:lipschitz_lcb} and $\overline{y}_s^{(i)}\ge \overline{y}_\nu^{(i)}$),
\[
w^{(i)}_s(L)
\le
\bigl(\UCBu_\nu^{(i)}-\LCBu_\nu^{(i)}\bigr)
+ L\bigl(\overline{y}_s^{(i)}-\overline{y}_\nu^{(i)}\bigr)
+ L\epsf_s.
\]
At the checkpoint, we ensure the utility interval width satisfies
$\UCBu_\nu^{(i)}-\LCBu_\nu^{(i)}\le c_u\,L\,\epsf_\nu$
(e.g., via Lemma~\ref{lem:committee_size_epsf_beta_gamma}).
Moreover, since $\overline{y}_s^{(i)}\le f^{\star(i)}$,
\[
\overline{y}_s^{(i)}-\overline{y}_\nu^{(i)} \le f^{\star(i)}-\overline{y}_\nu^{(i)} \le \epsf_\nu
\]
by Lemma~\ref{lem:simple_regret}.
Thus
\[
w^{(i)}_s(L)\le (c_u+1)L\epsf_\nu + L\epsf_s.
\]
Finally, for the geometric sparse schedule $\mathcal S_u=\{1,2,4,\dots\}$ we have $\nu\ge s/2$,
and for the dense schedule $\mathcal S_u=\mathbb N$ we have $\nu=s$.
In either case, $\nu\ge s/2$, so using $\epsf_s=2\sqrt{C_\lambda \beta_s^{(i)}\gamma_s^{(i)}/s}$ and monotonicity of $\beta_s^{(i)}\gamma_s^{(i)}$,
\[
\epsf_\nu \le 2\sqrt{\frac{C_\lambda \beta_s^{(i)}\gamma_s^{(i)}}{\nu}}
\le 2\sqrt{\frac{C_\lambda \beta_s^{(i)}\gamma_s^{(i)}}{s/2}}
= \sqrt{2}\,\epsf_s.
\]
Therefore
\[
w^{(i)}_s(L)\le \Bigl(\sqrt2(c_u+1)+1\Bigr)L\,\epsf_s.
\]
This is the claimed bound with $C_v:=\sqrt2(c_u+1)+1$.
\end{proof}

\section{Regret bounds}\label{app:lb_task_proofs}
\subsection{Auxiliary inequalities}\label{app:aux_ineq_main_regret}

\begin{lemma}[Square-root harmonic sum]\label{lem:sqrt_harmonic_sum}
For any integer $n\ge 1$,
\[
\sum_{s=1}^{n}\frac{1}{\sqrt{s}}
\ \le\
2\sqrt{n}.
\]
\end{lemma}
\begin{proof}
For $s\ge 2$, since $x\mapsto x^{-1/2}$ is decreasing we have
$\int_{s-1}^{s} x^{-1/2}\,dx \ge s^{-1/2}$.
Summing from $s=2$ to $n$ yields
\[
\sum_{s=2}^{n}\frac{1}{\sqrt{s}}
\ \le\
\int_{1}^{n} x^{-1/2}\,dx
\ =\
2(\sqrt{n}-1).
\]
Adding the $s=1$ term gives
$\sum_{s=1}^{n} s^{-1/2} \le 1 + 2(\sqrt{n}-1)=2\sqrt{n}-1\le 2\sqrt{n}$.
\end{proof}

\subsubsection{task-UCB cross-task regret bound (UCB-style)}
\label{app:meta_ucb_playcount}

\begin{lemma}[Standard UCB step: gap is controlled by the selected task's envelope width]
\label{lem:ucb_playcount}
Work on the envelope confidence event of Theorem~\ref{thm:value_envelopes_summary}, so that for all instantiated tasks $j$
and all local times $s$,
\[
\LCBV^{(j)}_{s}(\overline L)\ \le\ U^{(j)}\ \le\ \UCBV^{(j)}_{s}(\overline L).
\]
Fix a global round $t$ and define the best instantiated value at that time
\[
U_t^{\max} \;:=\; \max_{j\in\mathcal I_t} U^{(j)}.
\]
Let $i_t$ be the task selected by task-UCB at round $t$ and let $s_{t-1}^{(i_t)}$ be its local play-count
\emph{before} selection.
Then
\[
U_t^{\max}-U^{(i_t)}
\ \le\
w^{(i_t)}_{s_{t-1}^{(i_t)}}(\overline L)
:=\UCBV^{(i_t)}_{s_{t-1}^{(i_t)}}(\overline L)-\LCBV^{(i_t)}_{s_{t-1}^{(i_t)}}(\overline L).
\]
Consequently, writing $N_i:=s_T^{(i)}$,
\[
\sum_{t=1}^{T}\bigl(U_t^{\max}-U^{(i_t)}\bigr)
\ \le\
\sum_{i\in\mathcal I_T}\ \sum_{s=1}^{N_i} w^{(i)}_{s-1}(\overline L).
\]
\end{lemma}

\begin{proof}
Fix $t$.
By the task-UCB selection rule,
\[
\UCBV^{(i_t)}_{s_{t-1}^{(i_t)}}(\overline L)
\ =\
\max_{j\in\mathcal I_t}\UCBV^{(j)}_{s_{t-1}^{(j)}}(\overline L).
\]
On the confidence event, $\UCBV^{(j)}_{s_{t-1}^{(j)}}(\overline L)\ge U^{(j)}$ for all instantiated $j\in\mathcal I_t$,
hence the maximum is at least $U_t^{\max}$:
\[
\UCBV^{(i_t)}_{s_{t-1}^{(i_t)}}(\overline L)\ \ge\ U_t^{\max}.
\]
Also on the same event, $U^{(i_t)}\ge \LCBV^{(i_t)}_{s_{t-1}^{(i_t)}}(\overline L)$.
Subtracting yields
\[
U_t^{\max}-U^{(i_t)}
\ \le\
\UCBV^{(i_t)}_{s_{t-1}^{(i_t)}}(\overline L)-\LCBV^{(i_t)}_{s_{t-1}^{(i_t)}}(\overline L)
\ =\ w^{(i_t)}_{s_{t-1}^{(i_t)}}(\overline L).
\]
Summing over $t=1,\dots,T$ and regrouping by tasks gives the final inequality.
\end{proof}

\subsubsection{Generator}
\subsection{Proof of Theorem~\ref{thm:main_regret_adaptive}}
\label{app:proof_main_regret_adaptive}

\begin{proof}
Fix horizon $T$.
Let $\mathcal I_T$ be the set of tasks instantiated by time $T$, with $N_T:=|\mathcal I_T|$.
For each $i\in\mathcal I_T$, define the number of objective evaluations allocated to task $i$ up to time $T$ as
\[
N_i\ :=\ s_T^{(i)}\ =\ \sum_{t=1}^{T}\mathbf 1\{i_t=i\},
\qquad\text{so that}\qquad
\sum_{i\in\mathcal I_T} N_i = T.
\]
For each global round $t$, define the best instantiated long-run value
\[
U_t^{\max} \;:=\; \max_{i\in\mathcal I_t} U^{(i)}.
\]

\paragraph{Events.}
Let $\mathcal E_{\mathrm{env}}$ denote the confidence event of Theorem~\ref{thm:value_envelopes_summary};
it holds with probability at least $1-(\delta_f+\delta_u)$ and implies:
(i) the GP-UCB optimization gap bound $f^{\star(i)}-\overline{y}_s^{(i)}\le \epsf_s$ for all instantiated tasks and all $s\le T$, and
(ii) validity and the width bound for the value envelopes.
Let $\mathcal E_{\mathrm{gen}}$ denote the batch-success event of Lemma~\ref{lem:batch_success_adaptive};
it holds with probability at least $1-\delta_-$.
Work on $\mathcal E_{\mathrm{env}}\cap \mathcal E_{\mathrm{gen}}$.

\paragraph{Step 1: regret decomposition.}
For each round $t$, add and subtract $U_t^{\max}$ and $U^{(i_t)}$:
\begin{align}
\label{eq:regret_decomp_appendix_new}
U^\star - u^{(i_t)}\!\bigl(\overline{y}^{(i_t)}_t\bigr)
&=
\underbrace{(U^\star - U_t^{\max})}_{\text{generation}}
+
\underbrace{(U_t^{\max} - U^{(i_t)})}_{\text{selection}}
+
\underbrace{\Bigl(U^{(i_t)} - u^{(i_t)}(\overline{y}^{(i_t)}_t)\Bigr)}_{\text{within-task}}.
\end{align}
Summing over $t\in[T]$ yields
\begin{equation}
\label{eq:regret_decomp_appendix_sum_new}
\Regval(T)
\ \le\
\sum_{t=1}^{T}(U^\star - U_t^{\max})
+
\sum_{t=1}^{T}\bigl(U_t^{\max} - U^{(i_t)}\bigr)
+
\sum_{t=1}^{T}\Bigl(U^{(i_t)} - u^{(i_t)}(\overline{y}^{(i_t)}_t)\Bigr).
\end{equation}

\paragraph{Step 2: bound the generation term $\sum_{t}(U^\star-U_t^{\max})$.}
On $\mathcal E_{\mathrm{gen}}$, Lemma~\ref{lem:Vtmax_gap} gives
\[
U^\star - U_t^{\max} \ \le\ \epsU_{m_t\wedge m^\star}\qquad\forall t.
\]
Let $t_m$ be the first global round at which the current ladder level becomes $m$
(with the convention $t_0=1$), and note $m_t$ is nondecreasing.
Then
\[
\sum_{t=1}^{T}(U^\star - U_t^{\max})
\ \le\
\sum_{m=0}^{m^\star-1} (t_{m+1}-t_m)\,\epsU_m
\;+\;
(T-t_{m^\star})\,\epsU_{m^\star}
\ \le\
\sum_{m=0}^{m^\star-1} (t_{m+1}-t_m)\,\epsU_m
\;+\;
T\,\epsU_{m^\star}.
\]
By Lemma~\ref{lem:max_depth_reachability}, once level $m$ is active it advances within at most
$\tau_m:=K_m s_m$ rounds, where $K_m:=J(m+1)$ and
\[
s_m := \min\Bigl\{s\in\mathbb N:\ C_v\overline L\,\epsfT_s\le c_g\epsU_m\Bigr\}.
\]
Hence $t_{m+1}-t_m\le \tau_m$ for all $m<m^\star$, so
\begin{equation}
\label{eq:generation_term_bound_new}
\sum_{t=1}^{T}(U^\star - U_t^{\max})
\ \le\
\sum_{m=0}^{m^\star-1} \tau_m\,\epsU_m
\;+\;
T\,\epsU_{m^\star}.
\end{equation}
Moreover, using $\epsfT_s\le 2\sqrt{C_\lambda\beta_T\gamma_T/s}$ and $\epsU_m=\epsU_0 2^{-m}$ gives
$s_m \le \lceil A_T4^m\rceil$ with
$A_T=\left(\frac{2C_v\,\overline L\sqrt{C_\lambda\beta_T\gamma_T}}{c_g\epsU_0}\right)^2$,
as in the proof of Lemma~\ref{lem:max_depth_reachability}.
With $m^\star\le \overline m_T$ and $K_m\le N_T$, 
For the pre-$m^\star$ sum, use Cauchy--Schwarz:
\[
\sum_{m=0}^{m^\star-1}\tau_m\epsU_m
\le
\left(\sum_{m=0}^{m^\star-1}\tau_m\right)^{1/2}
\left(\sum_{m=0}^{m^\star-1}\tau_m(\epsU_m)^2\right)^{1/2}.
\]
The first factor is at most $\sqrt{T}$ by Lemma~\ref{lem:max_depth_reachability}.  For the second,
$\tau_m=K_m s_m$, $s_m\le A_T 4^m +1$, and ${\epsU_m}^2={\epsU_0}^2 4^{-m}$ imply
\[
\sum_{m=0}^{m^\star-1}\tau_m(\epsU_m)^2
\le
\sum_{m=0}^{m^\star-1}K_m(A_T4^m+1){\epsU_0}^2 4^{-m}
\le
\tilde{\mathcal O}\!\left(N_T A_T{\epsU_0}^2\right),
\]
because $K_m\le N_T$ and $m^\star\le \overline m_T=\tilde{\mathcal O}(\log T)$.
Therefore
\[
\sum_{m=0}^{m^\star-1}\tau_m\epsU_m
\le
\tilde{\mathcal O}\!\Bigl(\overline L\sqrt{C_\lambda\Psi_T}\sqrt{N_TT}\Bigr).
\]

\paragraph{Step 3: bound the selection term $\sum_t(U_t^{\max}-U^{(i_t)})$.}
By Lemma~\ref{lem:ucb_playcount},
\[
\sum_{t=1}^{T}\bigl(U_t^{\max}-U^{(i_t)}\bigr)
\ \le\
\sum_{i\in\mathcal I_T}\ \sum_{s=1}^{N_i} w^{(i)}_{s-1}(\overline L).
\]
Using the envelope width bound from Theorem~\ref{thm:value_envelopes_summary}
(for $s\ge 1$, and absorbing $w^{(i)}_0\le 1$ into constants), we have
$w^{(i)}_{s-1}(\overline L)\le C_v\,\overline L\,\epsf_{s-1}$, so
\[
\sum_{t=1}^{T}\bigl(U_t^{\max}-U^{(i_t)}\bigr)
\ \le\
\tilde{\mathcal O}\!\Big(\overline L\sum_{i\in\mathcal I_T}\sum_{s=1}^{N_i}\epsf_s\Big).
\]
Since $\epsf_s\le 2\sqrt{C_\lambda\beta_T\gamma_T/s}$ for all $s\le T$ and $i\in\mathcal I_T$,
Lemma~\ref{lem:sqrt_harmonic_sum} yields
$\sum_{s=1}^{N_i}\epsf_s\le 4\sqrt{C_\lambda\beta_T\gamma_T}\,\sqrt{N_i}$, hence
\begin{equation}
\label{eq:selection_term_bound_new}
\sum_{t=1}^{T}\bigl(U_t^{\max}-U^{(i_t)}\bigr)
\ \le\
\tilde{\mathcal O}\!\Bigl(\overline L\sqrt{C_\lambda\beta_T\gamma_T}\,\sqrt{N_T T}\Bigr).
\end{equation}

\paragraph{Step 4: bound the within-task term.}
For any task $i$ and local time $s$,
Lipschitzness and monotonicity of $u^{(i)}$ give
\[
U^{(i)} - u^{(i)}(\overline{y}_s^{(i)})
\le
L^\star\bigl(f^{\star(i)}-\overline{y}_s^{(i)}\bigr)
\le
L^\star\,\epsf_s.
\]
Therefore
\[
\sum_{t=1}^{T}\Bigl(U^{(i_t)} - u^{(i_t)}(\overline{y}^{(i_t)}_t)\Bigr)
=
\sum_{i\in\mathcal I_T}\sum_{s=1}^{N_i}\Bigl(U^{(i)}-u^{(i)}(\overline{y}^{(i)}_{s})\Bigr)
\le
L^\star\sum_{i\in\mathcal I_T}\sum_{s=1}^{N_i}\epsf_s.
\]
Using the same $\epsf_s\le 2\sqrt{C_\lambda\beta_T\gamma_T/s}$ and Lemma~\ref{lem:sqrt_harmonic_sum} argument,
\begin{equation}
\label{eq:within_task_bound_appendix_new}
\sum_{t=1}^{T}\Bigl(U^{(i_t)} - u^{(i_t)}(\overline{y}^{(i_t)}_t)\Bigr)
\ \le\
\tilde{\mathcal O}\!\Bigl(L^\star\sqrt{C_\lambda\beta_T\gamma_T}\,\sqrt{N_T T}\Bigr).
\end{equation}

\paragraph{Step 5: combine.}
Plugging \eqref{eq:generation_term_bound_new}, \eqref{eq:selection_term_bound_new},
and \eqref{eq:within_task_bound_appendix_new} into \eqref{eq:regret_decomp_appendix_sum_new} yields
\[
\Regval(T)
\ \le\
T\,\epsU_{m^\star}
+
\tilde{\mathcal O}\!\Bigl((\overline L+L^\star)\sqrt{C_\lambda\beta_T\gamma_T}\,\sqrt{N_T T}\Bigr).
\]
Since $\overline L\ge L^\star$, this simplifies to
\begin{equation}
\label{eq:pre_oracle_scale_bound_new}
\Regval(T)
\ \le\
T\,\epsU_{m^\star}
+
\tilde{\mathcal O}\!\Bigl(\overline L\sqrt{C_\lambda\beta_T\gamma_T}\,\sqrt{N_T T}\Bigr).
\end{equation}

\paragraph{Step 6: relate to $R^\star_T$.}
Recall $R^\star_T:=L^\star\sum_{s=1}^{T}\epsfT_s$ and $\epsfT_s$ is nonincreasing, so
$\sum_{s=1}^{T}\epsfT_s\ge T\epsfT_T$.
Condition~\ref{cond:bo_achievable} implies
\[
T\epsU_{m^\star}
\le
c_\epsilon\Lambda_T\,L^\star T\epsfT_T
\le
c_\epsilon\Lambda_T\,R^\star_T .
\]
Moreover, with $\Psi_T:=\max_{i\in\mathcal I_T}\beta_T^{(i)}\gamma_T^{(i)}$,
\[
\epsfT_T = 2\sqrt{C_\lambda\Psi_T/T},
\qquad
R^\star_T\ge L^\star T\epsfT_T
=2L^\star\sqrt{C_\lambda\Psi_TT}.
\]
Hence
\[
\overline L\sqrt{C_\lambda\Psi_T}\sqrt{N_TT}
\le
\frac{\overline L}{2L^\star}\sqrt{N_T}\,R^\star_T .
\]
Applying these displays to \eqref{eq:pre_oracle_scale_bound_new} gives
\[
\Regval(T)
\le
C\left[
\Lambda_T+
\Bigl(1+\sqrt{N_T}\Bigr)\frac{\overline L}{L^\star}
\right]R^\star_T .
\]
\end{proof}

\subsection{Oracle scale comparison}
\begin{lemma}[Bounding $\kappa_T$ for RBF and Mat\'ern kernels]
\label{lem:kappa_rbf_matern}
Fix a horizon $T$ and let $\mathcal I_T$ denote the set of instantiated tasks by time $T$.
For each task $i\in\mathcal I_T$, let $d^{(i)}$ be its input dimension, and let
$\epsf_s=2\sqrt{C_\lambda \beta_s^{(i)}\gamma_s^{(i)}/s}$ be the within-task
gap bound after $s$ evaluations, with $\gamma_s^{(i)}$ the maximum information gain and
$\beta_s^{(i)}$ the corresponding confidence parameter (e.g., IGP-UCB).
Define $\epsfT_s:=\max_{i\in\mathcal I_T}\epsf_s$ and
\[
\kappa_T
:=
\frac{\sum_{s=1}^T \epsfT_s}{\sum_{s=1}^T \epsilon_s^{f, (i^\star_{\epsU})}}.
\]

\smallskip\noindent
(a) (Squared-exponential / RBF) If each task uses a squared-exponential kernel, then
\[
\kappa_T=\tilde{\mathcal O}\big((\log T)^{d_{\max}-d_\star}\big),
\quad
d_{\max}:=\max_{i\in\mathcal I_T}d^{(i)},\;\; d_\star:=d^{(i^\star_{\epsU})}.
\]

\smallskip\noindent
(b) (Mat\'ern-$\nu$) If each task uses a Mat\'ern kernel with smoothness $\nu_i>0$, then
\[
\kappa_T=\tilde{\mathcal O}\big(T^{\alpha_{\max}-\alpha_\star}\big),
\quad
\alpha_{\max}:=\max_{i\in\mathcal I_T}\alpha(d^{(i)},\nu_i),\;\;
\alpha_\star:=\alpha(d_\star,\nu_\star),
\]
where $\alpha(d,\nu):=\frac{d(d+1)}{2\nu+d(d+1)}$ and $\nu_\star:=\nu_{i^\star_{\epsU}}$.
\end{lemma}

\begin{proof}
We use the standard IGP-UCB choice of confidence parameter, for which (up to notation)
$\beta_s^{(i)}$ scales as a constant plus $\gamma_{s-1}^{(i)}$ and $\log(1/\delta_i)$.
Therefore $\beta_s^{(i)}=\tilde{\mathcal O}(\gamma_s^{(i)})$ under the normalization
assumptions, and hence
\[
\epsf_s
=
2\sqrt{\frac{C_\lambda\,\beta_s^{(i)}\,\gamma_s^{(i)}}{s}}
=
\tilde{\mathcal O}\!\left(\sqrt{\frac{(\gamma_s^{(i)})^2}{s}}\right)
=
\tilde{\mathcal O}\!\left(\frac{\gamma_s^{(i)}}{\sqrt{s}}\right).
\]

\paragraph{(a) RBF.}
For a squared-exponential kernel on a compact subset of $\mathbb R^{d^{(i)}}$ we have
$\gamma_s^{(i)}=\mathcal O((\log s)^{d^{(i)}+1})$.
Thus $\epsf_s=\tilde{\mathcal O}((\log s)^{d^{(i)}+1}/\sqrt{s})$ and hence
$\epsfT_s=\tilde{\mathcal O}((\log s)^{d_{\max}+1}/\sqrt{s})$.
Summing and comparing to the same bound for $i^\star_{\epsU}$ yields
\[
\kappa_T
\;=\;
\frac{\tilde{\mathcal O}\!\left(\sum_{s=1}^T (\log s)^{d_{\max}+1}s^{-1/2}\right)}
{\tilde{\mathcal O}\!\left(\sum_{s=1}^T (\log s)^{d_{\star}+1}s^{-1/2}\right)}
\;=\;
\tilde{\mathcal O}\big((\log T)^{d_{\max}-d_\star}\big),
\]
where we used $\sum_{s=1}^T s^{-1/2}(\log s)^p=\Theta(\sqrt{T}(\log T)^p)$.

\paragraph{(b) Mat\'ern.}
For a Mat\'ern-$\nu_i$ kernel on a compact subset of $\mathbb R^{d^{(i)}}$ we have
$\gamma_s^{(i)}=\mathcal O(s^{\alpha(d^{(i)},\nu_i)}\log s)$ with
$\alpha(d,\nu)=\frac{d(d+1)}{2\nu+d(d+1)}$.
Hence $\epsf_s=\tilde{\mathcal O}(s^{\alpha(d^{(i)},\nu_i)-1/2})$ and
$\epsfT_s=\tilde{\mathcal O}(s^{\alpha_{\max}-1/2})$.
Summing gives $\sum_{s=1}^T s^{\alpha-1/2}=\Theta(T^{\alpha+1/2})$ for $\alpha>-1/2$,
and therefore
\[
\kappa_T
=
\frac{\tilde{\mathcal O}(T^{\alpha_{\max}+1/2})}{\tilde{\mathcal O}(T^{\alpha_\star+1/2})}
=
\tilde{\mathcal O}\big(T^{\alpha_{\max}-\alpha_\star}\big).
\]
This proves both claims.
\end{proof}

\subsection{Sparse vote case}
\begin{corollary}[Explicit $\beta_T,\gamma_T$ form]
\label{cor:main_regret_beta_gamma}
Under the conditions of Theorem~\ref{thm:main_regret_adaptive}, recall
$\beta_T := \max_{i\in\mathcal I_T}\beta_T^{(i)}$ and
$\gamma_T := \max_{i\in\mathcal I_T}\gamma_T^{(i)}$.
Then for all $s\le T$,
\[
\epsfT_s
=\max_{i\in\mathcal I_T} 2\sqrt{C_\lambda\,\beta_s^{(i)}\gamma_s^{(i)}/s}
\ \le\
2\sqrt{C_\lambda\,\beta_T\,\gamma_T/s}.
\]
Consequently,
\[
R^\star_T
=L^\star\sum_{s=1}^T \epsfT_s
\ \le\
4\,L^\star\sqrt{C_\lambda\,\beta_T\,\gamma_T\,T},
\]
and the regret bound \eqref{eq:main_regret_adaptive} implies
\begin{equation}
\label{eq:main_regret_beta_gamma}
\Regval(T)
\ \le\
4C\,\Bigl(1+\sqrt{N_T}\Bigr)\,\overline{L}\,
\sqrt{C_\lambda\,\beta_T\,\gamma_T\,T}.
\end{equation}
\end{corollary}
If $K_s$ is controlled so that $\UCBu^{(i)}_{\nu}-\LCBu^{(i)}_{\nu}\le c_u\,\Lbar\,\epsf_\nu$
at every checkpoint $\nu\in\mathcal S_u$, then for all $s$,
\[
w_s^{(i)}(\Lbar)\le C_v\,\Lbar\,\epsf_s,
\]
where $C_v=c_u+1$ for the dense schedule $\mathcal S_u=\mathbb N$, and
$C_v:=\sqrt2(c_u+1)+1$ for $\mathcal S_u=\{1,2,4,\dots\}$ (or any valid constant from Appendix~\ref{app:sparse_utility}).

\section{Hyperparameters}\label{app:hyperparameters}
\subsection{Hyperparameter Adaptation by Balance-Eliminate}
\begin{figure}
    \centering
    \includegraphics[width=0.25
    \linewidth]{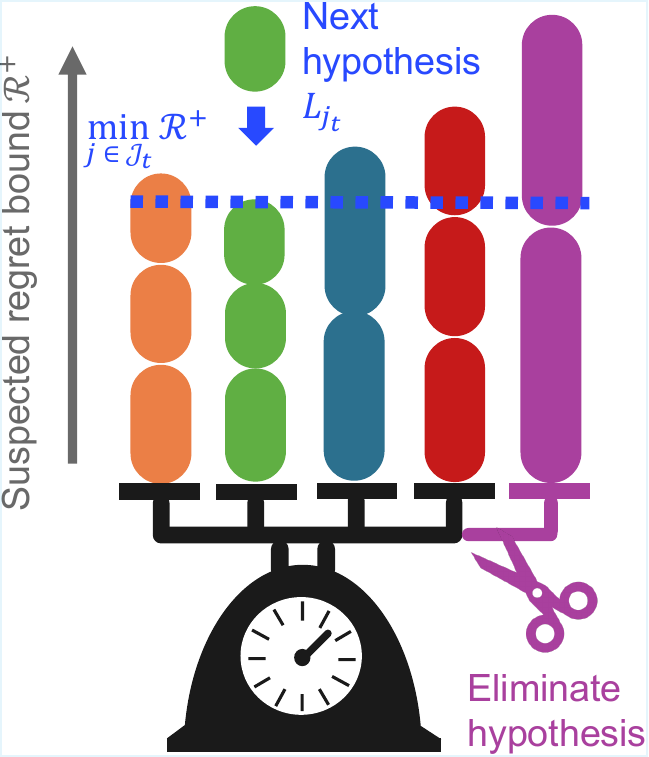}
    \caption{
    Alg.~\ref{alg:GSR}:
    We hold multiple candidate Lipschitz bounds $L_j$ as hypotheses. {\color{blue}Balance} decides which hypothesis to test by balancing suspected regret bound amongst hypotheses, and {\color{violet}eliminates} any hypothesis whose bound is falsified by observation.
    }
    \label{fig:GSR}
    \vspace{-0.5em}
\end{figure}
\begin{algorithm}[t!]
\caption{Evidence {\color{blue}Balancing}--Hypothesis {\color{violet}Elimination}}
\label{alg:GSR}
\begin{algorithmic}[1]
\STATE \textbf{Input:} base $L_0$, growth function $g_L(\cdot)$, failure budget $\delta_{\mathrm{be}}$.
\STATE \textbf{Hypotheses:} $L_j := q_L(j)=2^jL_0$.
\STATE \textbf{Init:} $j_{\max}\gets 0$; $\mathcal A_1\gets\{0\}$; and $\forall j\ge 0, S_0^j\gets\emptyset$.
\FOR{$t\in[T]$}
    \STATE {\color{blue}Choose $j_t\in \arg\min_{j\in\mathcal A_t}\Rsusp(n^j_{t-1}+1, L_j)$.}\label{alg_line:balance}
    \STATE $(i_t, \overline{y}_t, \tilde{u}_t) \gets$ RunOneStepAlg.\ref{alg:task-ucb}($\overline{L}=L_{j_t}$, state $\mathsf S^{j_t}$).
    \STATE For $j_t$, $S_t^{j_t}\gets S_{t-1}^{j_t}\cup\{t\}$ and $\forall j\neq j_t$, $S_t^j\gets S_{t-1}^j$.
    \STATE $\mathcal A_{t+1}\gets\mathcal A_t$.
    \IF{$\min_{j\in\mathcal A_t} n_t^j>0$}
        \STATE Set $M_t$, $\chi_t$, $\underline{U}_t(j)$, and
        $\overline{U}_t(j)$ for all $j\in\mathcal A_t$.
        \STATE {\color{violet}
        $\mathcal A_{t+1}\gets\Bigl\{j\in\mathcal A_t: $ Condition~\eqref{eq:eliminate} $\Bigr\}$.}
        \label{alg_line:eliminate}
    \ENDIF
    \WHILE{$L_{j_{\max}+1}\le L_0 g_L(t)$}
        \STATE $j_{\max}\gets j_{\max}+1$; $\mathcal A_{t+1}\gets \mathcal A_{t+1}\cup\{j_{\max}\}$.
    \ENDWHILE
\ENDFOR
\end{algorithmic}
\end{algorithm}
Theorem~\ref{thm:main_regret_adaptive} assumes we can supply a valid Lipschitz bound $\overline L \ge L^\star$.
In practice, however, $L^\star$ is unknown and tightly coupled to the (unknown) $\epsU$-optimal task.
We therefore perform \emph{hyperparameter adaptation} (not point estimation) via a regret balancing wrapper
(Alg.~\ref{alg:GSR}, Fig.\ref{fig:GSR}): we maintain multiple candidate bounds $L_j$ (hypotheses),
balance evidence across them using regret certificates, and eliminate candidates that are \emph{falsified} by observed utility $\tilde{u}_s$.

\paragraph{State used in the balancing wrapper.}
For the theoretical guarantee, each hypothesis $j$ maintains its own state $\mathsf S^j$ of
Algorithm~\ref{alg:task-ucb}, including its task set, envelopes, generation level, and GP states.
At round $t$, only the selected hypothesis $j_t$ is advanced by one physical evaluation.
This independent-state interpretation prevents an invalid hypothesis from triggering task-generation
events that contaminate the state of a valid hypothesis.  A shared-state implementation is a practical
heuristic and should be reported separately from the formal guarantee.

\textbf{Hypotheses.}
Fix $L_0>0$ and a growth function $g_L(T)\ge 1$ such that $L^\star\le L_0 g_L(T)$.
We consider the hypotheses generator $L_j:=q_L(j)=2^j L_0$ and maintain an active set $\mathcal A_t$ of hypotheses.
For each $j$, let
\[
S_t^j := \{\tau\le t : j_\tau=j\},
\qquad
n_t^j := |S_t^j|
\]
be the set on which hypothesis $j$ was selected so far.

\textbf{Evidence Balancer.}
When a hypothesis is \emph{valid} ($L_j\ge L^\star$), Theorem~\ref{thm:main_regret_adaptive} implies that running
Alg.~\ref{alg:task-ucb} for $n$ rounds with $\overline L=L_j$ incurs value-gap regret at most (Theorem~\ref{thm:main_regret_adaptive})
\begin{equation}
\label{eq:Rsusp_def_main}
\Rsusp(n,L_j)
:= C\,(1+\sqrt{N_T})\,L_j \sum_{s=1}^{n} \epsfT_s,
\end{equation}
For online computability, the certificate used by the master is evaluated with deterministic
upper bounds available at round $t$:
\[
N_t^+ := 1+J(\overline m_t+1),\qquad
{\epsf_s}^+(t) := 2\sqrt{C_\lambda\Psi_t^+/s},
\]
where $\Psi_t^+$ is the current monotone upper bound on
$\max_{i\in\mathcal I_t}\beta_s^{(i)}\gamma_s^{(i)}$ for $s\le t$.
Thus the causal certificate is
\[
\Rsusp_t(n,L_j)
:= C(1+\sqrt{N_t^+})L_j\sum_{s=1}^{n}{\epsf}_s^+(t).
\]
In practice, constants common across hypotheses are dropped in the argmin, but the elimination
test uses the same causal certificate.

Balancer selects which hypothesis to test next by Line~\ref{alg_line:balance}:
\[
j_t \in \arg\min_{j\in \mathcal A_t}\Rsusp(n_{t-1}^j+1, L_j).
\]
In practice, constant factors are irrelevant in minimization; we set $\Rsusp(n, L_j) = L_j \sum_{s=1}^n \epsfT_s$.
Intuitively, a hypothesis that has been played little has a small regret, so it gets selected until its regret
``catches up'' to the others. This prevents over-committing to any single hypothesis before the alternatives have accrued comparable evidence.

\textbf{Hypotheses Elimination.}
For each hypothesis $j$ with $n_t^j>0$, define the empirical mean utility
$
\widehat U_t(j)\ :=\ \frac{1}{n_t^j}\sum_{\tau\in S_t^j}\tilde{u}_\tau.
$
Similarly to Lemma~\ref{lem:utility_ci_hoeffding_main}, we form CIs
\begin{equation}
\label{eq:be_ci_defs_main}
\underline U_t(j)
:= \widehat U_t(j)-\sqrt{\sfrac{\chi_t}{n_t^j}},
\quad
\overline U_t(j)
:= \widehat U_t(j)+\sqrt{\sfrac{\chi_t}{n_t^j}},
\end{equation}
with
$
\chi_t \;:=\; 2 \sigma_u^2 \log\!\Bigl(\frac{4\pi^2 t^2 M_t}{\delta_{\mathrm{be}}}\Bigr),
$ and $M_t := 1+\max\mathcal A_t$\footnote{$\max\mathcal A_t$ denotes the largest \emph{index} in the active rung set (not its cardinality $|\mathcal A_t|$).
}.
We eliminate hypothesis $j$ if it does not satisfy:
\begin{equation}
\label{eq:eliminate}
\overline U_t(j) + \sfrac{\Rsusp(n_t^j,L_j)}{n_t^j}
\;\geq\;
\max_{k\in\mathcal A_t}\underline U_t(k).
\end{equation}
Intuitively, this rule compares the most optimistic bound on hypothesis $j$ (LHS) with the best pessimistic bound among all hypotheses (RHS). If even the weakest consistency condition—the upper bound exceeds the lower bound—is violated, then hypothesis $j$ cannot be valid and is eliminated. Valid hypotheses are never removed (Lemma~\ref{lem:no_false_elim_L}).

\begin{lemma}[Uniform concentration for per-rung utility averages]
\label{lem:be_concentration}
Fix $\delta_{\mathrm{be}}\in(0,1)$ and define $\chi_t$ as above.
Then with probability at least $1-\delta_{\mathrm{be}}$, simultaneously for all $t\le T$ and all rungs $j$ with $n_t^j>0$,
\[
\left|\widehat U_t(j)-\frac{1}{n_t^j}\sum_{\tau\in S_t^j}u_\tau\right|
\ \le\ 
\sqrt{\frac{\chi_t}{n_t^j}}.
\]
\end{lemma}

\begin{proof}
For fixed $(t,j)$, $\sum_{\tau\in S_t^j}\xi_\tau$ is sub-Gaussian with variance proxy $n_t^j\sigma_u^2$.
Thus,
\[
\Pr\!\left(\left|\sum_{\tau\in S_t^j}\xi_\tau\right|>\sqrt{n_t^j\chi_t}\right)
\le 2\exp\!\Big(-\frac{\chi_t}{2\sigma_u^2}\Big)
= \frac{\delta_{\mathrm{be}}}{2\pi^2 t^2 M_t}.
\]
Union bound over $t\in[T]$ and $j\in[M_T]$ and use $\sum_{t=1}^\infty t^{-2}=\pi^2/6$.
\end{proof}

\begin{lemma}[No false elimination of valid Lipschitz rungs]
\label{lem:no_false_elim_L}
Work on the intersection of Lemma~\ref{lem:be_concentration} and the confidence event of Theorem~\ref{thm:main_regret_adaptive}.
Then Algorithm~\ref{alg:GSR} never eliminates any rung $j$ with $L_j\ge L^\star$.
In particular, the smallest valid rung $j^\star$ remains in $\mathcal A_t$ for all $t$.
\end{lemma}

\begin{proof}
Fix $t$ and a valid rung $j$ with $n_t^j>0$.
Let $\hat u_t(j):=\frac{1}{n_t^j}\sum_{\tau\in S_t^j}u_\tau$ denote the average noise-free utility on plays of rung $j$.
Since $u_\tau\le U^\star$ for all $\tau$, we have $\hat u_t(j)\le U^\star$, hence
$\underline{U}_t(j)\le \hat u_t(j)\le U^\star$ on Lemma~\ref{lem:be_concentration}, and therefore
$\max_{k\in\mathcal A_t}\underline{\hat u}_t(k)\le U^\star$.

For a valid rung $j$, Theorem~\ref{thm:main_regret_adaptive} implies (up to the constant absorbed in $\Rsusp$) that
\[
\sum_{\tau\in S_t^j}\bigl(U^\star-u_\tau\bigr)\ \le\ \Rsusp(n_t^j,L_j).
\]
Rearranging gives
\[
U^\star\ \le\ \hat u_t(j)+\frac{\Rsusp(n_t^j,L_j)}{n_t^j}.
\]
On Lemma~\ref{lem:be_concentration}, $\hat u_t(j)\le \overline{U}_t(j)$, hence
\[
U^\star\ \le\ \overline{U}_t(j)+\frac{\Rsusp(n_t^j,L_j)}{n_t^j}.
\]
Combining the two displays yields
\[
\overline{U}_t(j)+\frac{\Rsusp(n_t^j,L_j)}{n_t^j}\ \ge\ U^\star\ \ge\ \max_{k\in\mathcal A_t}\underline{U}_t(k),
\]
so the falsification condition \eqref{eq:eliminate} cannot hold for any valid rung.
\end{proof}

\begin{theorem}[Lipschitz-balancing overhead]
\label{thm:lb_gse_formal}
Assume the conditions of Theorem~\ref{thm:main_regret_adaptive} and $L^\star\le L_0 g_L(T)$.
Let $M_T:=1+\lceil\log_2 g_L(T)\rceil$.
Then, with probability at least $1-(\delta_f+\delta_u+\delta_-+\delta_{\mathrm{be}})$,
\[
\mathcal{R}_\text{val, GSR}(T)
\ \le\
\tilde{\mathcal O}\!\bigl(M_T\bigr)\;
\mathcal{R}_{\text{val, GSR}^{(j^\star)}}(T)\ \le\
\tilde{\mathcal O}\!\bigl(\log g_L(T)\bigr)\;
\Bigl(1+\sqrt{N_T}\Bigr)\;
R^\star_T,
\]
where $\AlgoName^{(j^\star)}$ denotes Alg.~\ref{alg:task-ucb} run with $\overline L=L_{j^\star}$ and
$R^\star_T=L^\star\sum_{s=1}^T \epsfT_s$ is the single-task oracle scale from Theorem~\ref{thm:main_regret_adaptive}.
\end{theorem}

\begin{proof}[Proof sketch]
Algorithm~\ref{alg:GSR} is an instance of regret-bound balancing with elimination \citep{ziomek2024bayesian}
applied to the ladder $\{L_j\}_{j=0}^{M_T-1}$, using the candidate bounds $\Rsusp(\cdot,L_j)$.
Lemma~\ref{lem:no_false_elim_L} ensures that a well-specified rung (in particular $j^\star$) is never eliminated.
Standard RBLE analysis then yields a master regret bounded by $\tilde{\mathcal O}(M_T)$ times the regret
of the best well-specified rung at horizon $T$.
Finally, $L_{j^\star}\le 2L^\star$ and Theorem~\ref{thm:main_regret_adaptive} imply
$\mathcal{R}_{\text{val, GSR}^{(j^\star)}}(T)=\tilde{\mathcal O}((1+\sqrt{N_T})R^\star_T)$, giving the claim.
\end{proof}

\subsection{Other hyperparameters}
Other hyperparameters can also be estimated from regret balancing wrapper.
\begin{theorem}[Overhead factors under unknown hyperparameters]
\label{thm:table1_overheads}
Under Theorem~\ref{thm:main_regret_adaptive} with the parameter adaptation with Algorithm~\ref{alg:GSR}, we obtain:
\vspace{-0.5em}
\begin{table}[H]
\centering
\begin{tabular}{lc}
\toprule
Unknown(s) & $R(T) / R^\star_T$\\
\midrule
Lipschitz bound $\overline{L}$ & $\tilde{\mathcal O}(\log g_L(T))$ \\
Regret scale $C$ & $\tilde{\mathcal O}(\log g_C(T))$ \\
GP hyperparameters & $\tilde{\mathcal O}(\log g_k(T))$ \\
\bottomrule
\end{tabular}
\label{tab:lb_style}
\end{table}
\vspace{-1.5em}
where $g(T)$ are geometric generators for each parameters. 
\end{theorem}

\subsubsection{Unknown within-task GP hyperparameters: regret balancing over a kernel ladder}
\label{app:gp_hyper_ladder}

This is straightforward application of \citep{ziomek2024bayesian}.

\subsubsection{Unknown regret scale}
\label{app:cert_scale_ladder}

The elimination rule uses the regret certificate $\Rsusp(\cdot,L)$, which hides an absolute multiplicative constant.
In the main text we denote this constant by $C$ and treat it as a user-specified slack.
Here we show that $C$ can itself be adapted by the same regret-balancing + falsification mechanism.

\paragraph{Certificate-scale ladder.}
Assume there exists an (unknown) constant $C^\star<\infty$ such that,
whenever $\overline L\ge L^\star$, Algorithm~\ref{alg:task-ucb} satisfies the value-gap regret bound
\[
\sum_{t=1}^{n}\Bigl(U^\star-u^{(i_t)}(\overline{y}^{(i_t)}_t)\Bigr)
\ \le\
C^\star\,(1+\sqrt{\Ktmax})\,\overline L\sum_{s=1}^{n}\epsfT_s,
\qquad \forall n\le T,
\]
on the same confidence event as Theorem~\ref{thm:main_regret_adaptive}.
We introduce a geometric ladder
\[
c_r:=2^r c_0,\qquad r\in\{0,1,\dots,R_C\},\qquad R_C:=\left\lceil \log_2 g_C(T)\right\rceil,
\]
where $c_0>0$ and $g_C(T)\ge 1$ satisfy $C^\star\le c_0 g_C(T)$.

\paragraph{Two-dimensional hypotheses.}
To adapt both the unknown Lipschitz bound and the unknown certificate scale, treat each hypothesis as a pair
$h=(j,r)$ with $(L_j,c_r)$.
Define the hypothesis-specific certificate
\[
\Rsusp(n;h)
\ :=\
c_r\,(1+\sqrt{\Ktmax})\,L_j\sum_{s=1}^{n}\epsfT_s.
\]
Run the same balance--eliminate wrapper as Algorithm~\ref{alg:GSR} over hypotheses $h$
(where selecting $h=(j,r)$ executes Algorithm~\ref{alg:task-ucb} with $\overline L=L_j$ and uses $\Rsusp(\cdot;h)$ in the master).

\begin{lemma}[No false elimination of well-specified $(L_j,c_r)$]
\label{lem:no_false_elim_LC}
Let $h=(j,r)$ be \emph{well-specified}, i.e.\ $L_j\ge L^\star$ and $c_r\ge C^\star$.
Work on the intersection of:
(i) the envelope confidence event required by Theorem~\ref{thm:main_regret_adaptive}, and
(ii) the uniform concentration event for per-hypothesis utility averages (analogous to Lemma~\ref{lem:be_concentration}).
Then the balance--eliminate master never eliminates hypothesis $h$.
\end{lemma}

\begin{proof}
The proof is identical to Lemma~\ref{lem:no_false_elim_L}, replacing $\Rsusp(\cdot,L_j)$ by $\Rsusp(\cdot;h)$.
Well-specified hypotheses satisfy the regret guarantee with the corresponding certificate,
so their certificate-relaxed upper bound always dominates the best lower bound and cannot be falsified.
\end{proof}

\begin{theorem}[Overhead from adapting the certificate scale]
\label{thm:cert_scale_overhead}
Assume $L^\star\le L_0 g_L(T)$ and $C^\star\le c_0 g_C(T)$ and the conditions of
Theorem~\ref{thm:main_regret_adaptive}. Let
$M_L:=1+\lceil\log_2 g_L(T)\rceil$ and
$M_C:=1+\lceil\log_2 g_C(T)\rceil$.
Running the balance--eliminate master on the combined ladder incurs an overhead
$\tilde{\mathcal O}(M_LM_C)$ relative to a single well-specified hypothesis.
Equivalently, relative to the Lipschitz-only ladder of Theorem~\ref{thm:lb_gse_formal}, adapting
the certificate scale adds a multiplicative factor $\tilde{\mathcal O}(M_C)
=\tilde{\mathcal O}(\log g_C(T))$.
\end{theorem}

\begin{proof}[Proof sketch]
There are at most $M_L M_C$ hypotheses $(j,r)$.
By Lemma~\ref{lem:no_false_elim_LC}, at least one well-specified hypothesis survives forever.
Applying the standard regret-balancing + elimination master guarantee (the same black-box result used for the Lipschitz ladder)
bounds the master regret by a factor $\tilde{\mathcal O}(\sqrt{M_L M_C})$ times the regret of the best well-specified hypothesis.
Relative to the Lipschitz-only ladder (which contributes $\tilde{\mathcal O}(\sqrt{M_L})$),
this introduces an extra factor $\tilde{\mathcal O}(\sqrt{M_C})$.
\end{proof}

\section{Experiments}\label{app:experiments}
We performed experiments on MacBook Pro 2019, 2.4 GHz 8-Core Intel Core i9, 64 GB 2667 MHz DDR4.

\subsection{LLM interfaces: task mutation generator and evaluator}
\label{app:llm_interfaces}

This subsection specifies the \textbf{exact LLM interfaces} used in our experiments for
(i) task mutation generation $\Gen(\cdot)$ and (ii) cross-task evaluation (pairwise ``quick tasting'').
We implement both via chat-completions in \textbf{strict JSON mode} to ensure parseable outputs.

\paragraph{LLM API and decoding.}
We use the OpenAI chat-completions interface with:
\begin{itemize}
  \item \textbf{Model:} \texttt{gpt-4o-mini}.
  \item \textbf{Decoding:} \texttt{temperature=0.0} (deterministic given a fixed seed),
        \texttt{max\_tokens=800}.
  \item \textbf{Structured output:} \texttt{response\_format=\{type: "json\_object"\}} (JSON-only responses).
  \item \textbf{Robustness:} if parsing fails or the JSON schema is violated, we retry up to
        \texttt{MAX\_RETRIES=2} times; if still invalid, we fall back to a conservative default (e.g., random vote).
\end{itemize}
In the committee evaluator, we also randomize the A/B ordering per vote and map the returned winner back to the
original candidate/anchor to reduce positional bias.

\paragraph{Task specification as JSON.}
Across all domains, each task $i$ is represented as a structured JSON object with a fixed schema.
The schema is \emph{domain-dependent} (e.g., wine planning vs.\ acquisition-function analysis), but the LLM
interface is shared: the generator outputs a list of valid task-spec JSONs; the evaluator compares two
task-spec JSONs plus their current incumbents / confidence summaries.

\subsubsection{Task mutation generator: actor--critic prompting}
\label{app:llm_generator_prompts}

We implement $\Gen(a,m,J)$ as a two-stage LLM pipeline:
\textbf{Actor} proposes $J$ mutated children from an anchor task $a$ (at rung/level $m$),
and \textbf{Critic} reviews and repairs the list to satisfy schema constraints and remove duplicates.

\paragraph{Generator I/O format and verification pipeline.}
We implement task mutation as an Actor--Critic generator with an additional deterministic verifier.
The Actor proposes candidate child task specifications by editing an anchor task JSON under a rung-dependent edit budget $\rho_m$.
A Critic then audits and repairs the proposals (schema, edit-budget adherence, feasibility, and duplicates), after which a rule-based verifier enforces hard constraints before admission.
This layered design is intended to prevent pathological cases where a small number of edits (small $\rho_m$) causes a large semantic shift.

\textbf{Input to Actor/Critic:}
\begin{itemize}
  \item the anchor task specification (JSON), injected into the prompt as \texttt{\{\{TASK\_SPEC\}\}},
  \item the current rung/level $m$ and target mutation ratio $\rho_m$ (or an equivalent edit budget),
        injected as \texttt{\{\{LEVEL\_M\}\}} and \texttt{\{\{RHO\_M\}\}},
  \item the full mutation/evaluation history up to the current generation call—i.e., all previously generated
        tasks and, for each past global round $\tau$, the tuple $(i_\tau,x_\tau,y_\tau,\tilde u_\tau)$—
        injected as \texttt{\{\{MUTATION\_HISTORY\_JSON\}\}},
  \item the requested batch size $J$,
  \item domain constraints (allowed categorical choices; bounds format; dimensionality consistency), plus any rung-dependent
        safeguards (e.g., tighter bound-change allowances as $\rho_m$ shrinks).
\end{itemize}

\textbf{Output:}
a JSON object with key \texttt{"task\_list"} containing exactly $J$ mutated child task specifications that share the
same top-level schema as the anchor.

\paragraph{Mutation history JSON (conditioning signal).}
We provide the generator a machine-readable log that contains prior information needed for evidence-based mutations and for
duplicate avoidance. Concretely, \texttt{\{\{MUTATION\_HISTORY\_JSON\}\}} is a JSON object that includes:
(i) a per-round evaluation trace (at each global round $t$: selected task, evaluated design, observed outcome, and utility),
and (ii) a registry of all generated tasks (including parent/anchor lineage and best-so-far summaries).
A minimal schema is:
\begin{verbatim}
{
  "t_now": int,
  "eval_trace": [
    {
      "t": int,
      "task_spec": { ... },      // the task evaluated at round t
      "x": [ ... ],              // design used at round t (local-to-task design)
      "y": float,                // observed objective
      "u": float                 // utility (or proxy) at round t
    }
    // one entry per past round
  ],
  "task_registry": [
    {
      "task_spec": { ... },
      "parent_spec": { ... } or null,
      "m": int,                  // rung/level at which this task was generated
      "best_x": [ ... ] or null, // best-so-far design for this task
      "best_y": float or null,
      "best_u": float or null
    }
    // one entry per instantiated/generated task
  ]
}
\end{verbatim}
We pass the full history when it fits; otherwise we pass a compact summary that preserves the information needed for
conditioning and filtering (e.g., lineage and per-task best-so-far statistics), so the generator can still ground mutations
in prior evidence.

\paragraph{Exact prompt templates (generator).}
Tables~\ref{tab:generator_actor}, \ref{tab:generator_critic} are the templates used in code.
The placeholders
\texttt{\{\{TASK\_SPEC\}\}}, \texttt{\{\{LEVEL\_M\}\}}, \texttt{\{\{RHO\_M\}\}},
\texttt{\{\{MUTATION\_HISTORY\_JSON\}\}}, and \texttt{\{J\}} are filled programmatically.
In addition to enforcing $\rho_m$, the prompts explicitly instruct the models to avoid near-duplicates and to use the
history to steer mutations away from previously low-utility edit patterns.

\small
\begin{table}[t!]
    \centering
    \caption{Generator prompts (Actor)}
    \label{tab:generator_actor}
    \setlength{\tabcolsep}{4pt}
    \begin{tabular}{ll}
    \toprule
    Actor: system &
    \makecell[l]{
    You are an assistant that proposes new BO tasks (problem settings) by mutating an existing task JSON.\\
    You will be given the full mutation/evaluation history (generated tasks and past evaluations with x, y, and u).\\
    Use this history as in-context evidence: learn which mutations improved utility and avoid repeating failures.\\
    Prefer \emph{local}, budget-consistent edits when \texttt{rho\_m} is small; avoid semantic drift from the anchor.\\
    Follow the schema strictly and output only JSON.}\\
    \midrule
    Actor: user &
    \makecell[l]{
    You are given:\\
    (1) Current rung level and target mutation ratio: \texttt{m=\{\{LEVEL\_M\}\}}, \texttt{rho\_m=\{\{RHO\_M\}\}}.\\
    (2) Anchor task specification as JSON:\\
    \\
    \{\{TASK\_SPEC\}\}\\
    \\
    (3) Mutation/evaluation history as JSON (all past rounds + all generated tasks):\\
    \\
    \{\{MUTATION\_HISTORY\_JSON\}\}\\
    \\
    Generate \{J\} mutated child task specifications.\\
    Requirements:\\
    - Keep the same top-level schema keys (do not add/remove keys).\\
    - Respect the target mutation ratio \texttt{rho\_m}: change approximately a \texttt{rho\_m} fraction of editable fields\\
    \quad (about \texttt{round(rho\_m * \#editable\_fields)} fields); keep the rest identical to the anchor.\\
    - Use the history: prefer edits correlated with higher observed utility \texttt{u}; avoid repeating edit patterns\\
    \quad that correlate with lower \texttt{u} or clearly worse outcomes.\\
    - Do not introduce large semantic shifts for small \texttt{rho\_m} (e.g., drastic scale changes); keep mutations local.\\
    - Do not duplicate the anchor, any previously instantiated task in \texttt{task\_registry}, or each other\\
    \quad (treat near-duplicates as duplicates).\\
    - Ensure bounds are valid, consistent with dim, and not degenerate.\\
    - Output ONLY JSON with the schema:\\
    \{\\
    \qquad"task\_list": [\\
    \qquad\qquad\{\\
    \qquad\qquad\qquad"acquisition": "...",\\
    \qquad\qquad\qquad"objective": "...",\\
    \qquad\qquad\qquad"dim": int,\\
    \qquad\qquad\qquad"bounds": [[float, float], ...],\\
    \qquad\qquad\qquad"noise\_sigma": float,\\
    \qquad\qquad\qquad"notes": "..."\\
    \qquad\qquad\},\\
    \qquad...\\
    \qquad]\\
    \}\\
    }\\
    \bottomrule
    \end{tabular}
\end{table}

\begin{table}[t!]
    \centering
    \caption{Generator prompts (critic)}
    \label{tab:generator_critic}
    \setlength{\tabcolsep}{4pt}
    \begin{tabular}{ll}
    \toprule
    Critic: system &
    \makecell[l]{
    Review the proposed tasks using the anchor, rung (m, rho\_m), and the full mutation history.\\
    Fix schema violations, enforce the mutation-ratio constraint, remove duplicates (including with history),\\
    and ensure feasibility.\\
    When \texttt{rho\_m} is small, repair proposals that cause disproportionate semantic drift from the anchor.\\
    Return exactly the same JSON schema with "task\_list" of length \{J\}.}\\
    \midrule
    Critic: user &
    \makecell[l]{
    Review the following proposed tasks.\\
    Fix schema violations, enforce the target mutation ratio, remove duplicates (including with history),\\
    and ensure feasibility.\\
    Current rung and ratio: \texttt{m=\{\{LEVEL\_M\}\}}, \texttt{rho\_m=\{\{RHO\_M\}\}}.\\
    Anchor task JSON:\\
    \\
    \{\{TASK\_SPEC\}\}\\
    \\
    Mutation/evaluation history JSON:\\
    \\
    \{\{MUTATION\_HISTORY\_JSON\}\}\\
    \\
    Proposed tasks JSON:\\
    \\
    \{\{PROPOSED\_TASKS\_JSON\}\}\\
    \\
    Return ONLY:\\
    \{\\
    \qquad"task\_list": [ ... exactly \{J\} tasks ... ]\\
    \}\\
    }\\
    \bottomrule
    \end{tabular}
\end{table}
\normalsize

\paragraph{Deterministic (rule-based) verification and regeneration.}
After the Critic returns \texttt{"task\_list"}, we apply a rule-based verifier before any task is admitted.
This verifier enforces hard constraints and filters out near-duplicates against both the anchor and the full history.
If any task fails verification, it is rejected and replaced via critic-driven regeneration until a valid batch of size $J$
is obtained.

\paragraph{Practical validation checks (hard constraints).}
After parsing the JSON, we enforce:
(i) all required keys exist and types are valid; (ii) \texttt{dim} matches \texttt{len(bounds)};
(iii) each bound satisfies \texttt{lower < upper}; (iv) \texttt{noise\_sigma > 0};
(v) the mutation-ratio constraint is satisfied within tolerance by counting edited (editable) fields relative to the anchor;
(vi) near-duplicate tasks (within a small tolerance on numeric fields) are rejected if they duplicate the anchor,
each other, \emph{or any task already present in the mutation history} (e.g., \texttt{task\_registry}).
Only tasks that pass both the Critic audit and these deterministic checks are evaluated downstream.

\subsubsection{Evaluator: pairwise judge prompt (committee votes)}
\label{app:llm_evaluator_prompts}

We implement the cross-task utility oracle as a \textbf{pairwise judge} that compares
two task--incumbent summaries and returns a binary preference (A vs.\ B).
To reduce variance, we use a committee of $K$ independent votes and average them.

\paragraph{Evaluator I/O format.}
\textbf{Input:} two JSON blobs (A and B), each containing:
\begin{itemize}
  \item \texttt{task\_spec}: the task JSON (domain-specific schema),
  \item \texttt{g}: the current estimated utility / score (higher is better),
  \item \texttt{gbar}: an optimistic score (e.g., an upper confidence summary),
  \item \texttt{metrics}: optional diagnostic fields (noise, dim, etc.).
\end{itemize}

\textbf{Output:} a single JSON object
\[
\texttt{\{ "winner": "A" \text{ or } "B",  "notes": "..." \}}
\]
We parse \texttt{winner} and map it back to the original candidate/anchor ordering if A/B were swapped.

\paragraph{Exact prompt templates (evaluator).}
Table~\ref{tab:evaluator} is the exact templates used in code.
\begin{table}[t!]
    \centering
    \caption{Evaluator prompts (random personas)}
    \label{tab:evaluator}
    \begin{tabular}{ll}
    \toprule
    Judge: system & \makecell[l]{
    You are a careful judge that decides which BO task should get the next evaluation.\\
    You must output only JSON.
    }\\
    \midrule
    Judge: user & \makecell[l]{
    You will be given two candidate tasks A and B.\\
    Each is a JSON describing a BO task setting with:\\
    - task\_spec: the task specification JSON.\\
    - g: the current best utility/score achieved so far (higher is better).\\
    - gbar: an optimistic estimate / upper confidence summary for the long-run value.\\
    - metrics: optional extra diagnostics.\\
    \\
    Your job: choose which task should receive the next expensive BO evaluation.\\
    \\
    Decision guidelines:\\
    - Prefer larger gbar (more promising under uncertainty), but consider g (exploitation).\\
    - If similar, prefer settings that are more learnable / reliable (e.g., lower noise, simpler search space),\\
    \quad and prefer tasks that are not obviously dominated.\\
    - If both are weak or unclear, choose the one with larger gbar.\\
    \\
    Return ONLY JSON with schema:
    \{
    \qquad"winner": "A" or "B",
    \qquad"notes": "one sentence"
    \}\\
    \\
    Task A:\\
    \{\{TASK\_A\_JSON\}\}\\
    \\
    Task B:\\
    \{\{TASK\_B\_JSON\}\}\\
    }\\
    \bottomrule
    \end{tabular}
\end{table}

\subsubsection{Persona-conditioned evaluator prompt, loading, and vote aggregation}
\label{app:llm_persona_committee}

\paragraph{Why personas.}
To reduce “single-judge idiosyncrasy” and to better emulate heterogeneous stakeholder preferences,
we evaluate each pairwise comparison using a \emph{committee} of persona-conditioned judges. Each
committee member is the same base model but with a different persona injected into the system prompt.

\paragraph{Persona bank format and loading.}
We represent each persona as a single JSON object and store a persona bank as JSONL (one persona per
line). At runtime, we load the bank once and then sample personas for each query:
\[
\texttt{personas} \leftarrow \texttt{[json.loads(line) for line in open("personas.jsonl")]}.
\]
Each vote draws either (i) $K$ personas i.i.d.\ with replacement, or (ii) $K$ personas without
replacement when $|\texttt{personas}|\ge K$.

\paragraph{Persona schema (JSON).}
\begin{verbatim}
{
  "persona_id": "string",
  "role": "string",               % e.g., "risk-averse engineer",
  "goals": ["...","..."],         % short bullet goals
  "preferences": { ... },         % domain-specific knobs (weights, constraints, etc.)
  "avoid": ["...","..."]          % optional hard dislikes / constraints
}
\end{verbatim}

\paragraph{Evaluator input.}
Each comparison query provides two candidates (e.g., task $i$ vs.\ anchor task $a$), including the
task spec and the current incumbent performance summary (the “best-so-far” outcome for each task).
We pass these inputs as JSON blocks.

\paragraph{Evaluator output schema (JSON).}
The judge returns a single JSON object:
\begin{verbatim}
{
  "winner": "A" | "B",
  "confidence": 0.0-1.0,
  "notes": "short rationale"
}
\end{verbatim}
We discard any malformed outputs (rare under JSON mode), and in case of ties or ambiguity the judge
is instructed to pick a winner with \texttt{confidence} near $0.5$.

\paragraph{Prompt template (persona-conditioned judge).}
The system message injects the persona; the user message supplies the two candidates and reiterates
the output schema. (Our reference implementation uses JSON mode, i.e., the model is required to
emit a valid JSON object.)
Table~\ref{tab:persona} shows the exact prompt for persona-conditioned evaluator.

\begin{table}[t!]
    \centering
    \caption{Evaluator prompt (pooled personas)}
    \label{tab:persona}
    \begin{tabular}{ll}
    \toprule
    SYSTEM: &  \makecell[l]{
    You are simulating the following persona (JSON):\\
    \\
    \{persona\_json\}\\
    \\
    You will compare two candidates for THIS persona.\\
    Focus on the persona's goals and preferences.\\
    Return ONLY a valid JSON object with keys:\\
    winner, confidence, notes.\\
    }\\
    \midrule
    USER & \makecell[l]{
    Candidate A (JSON):\\
    \\
    \{candidate\_A\_json\}\\
    \\
    Candidate B (JSON):\\
    \\
    \{candidate\_B\_json\}\\
    \\
    Question:\\
    Which candidate is better for the persona?\\ Pick A or B.\\
    Output JSON only:\\
    \{"winner": "...", "confidence": ..., "notes": "..."\}\\
    }\\
    \bottomrule
    \end{tabular}
\end{table}

\subsection{New product development: wine planning benchmark}
\label{app:settingA_wine}

This subsection fully specifies the wine benchmark used in \S\ref{sec:experiments} (planning).
The goal is to model a realistic ``product design'' loop in which (i) within-task BO proposes wine recipes/chemistries,
and (ii) a cross-task evaluator (simulated or LLM committee) judges which \emph{target product brief} is more promising.

\subsubsection{Dataset and feature normalization}
We use the Wine Quality dataset~\citep{cortez2009modeling}, separately for red and white wines.
Each datapoint contains 11 physicochemical features and a quality label.
Let $x_{\rm raw}\in\mathbb R^{11}$ denote the raw features and $y_{\rm raw}$ the integer quality label.

\paragraph{Robust bounds and unit-cube mapping.}
To avoid extrapolation to extreme outliers, we compute per-feature lower/upper bounds using empirical quantiles:
for each feature dimension $k$, let $\ell_k$ be the 2\% quantile and $u_k$ the 98\% quantile over the dataset.
We map raw features to the unit cube via
\begin{equation}
x_k \;=\; \frac{x_{{\rm raw},k}-\ell_k}{u_k-\ell_k}
\quad\in[0,1],
\qquad
x=(x_1,\dots,x_{11})\in[0,1]^{11}.
\end{equation}
All task domains $\mathcal X^{(i)}$ are axis-aligned boxes in the unit cube.

\subsubsection{Quality surrogate and bounded tasting score}
We define a reusable, fast environment by fitting a regression surrogate
$\widehat g(x_{\rm raw})$ that predicts the quality label from physicochemical measurements.
In our released implementation we use histogram gradient boosting regression and normalize to $[0,1]$:
\begin{equation}
q(x)
\;=\;
\mathrm{clip}_{[0,1]}\!\left(
\frac{\widehat g(x_{\rm raw})-y_{\min}}{y_{\max}-y_{\min}}
\right)\in[0,1],
\end{equation}
where $y_{\min},y_{\max}$ are the minimum/maximum labels in the dataset split used to fit $\widehat g$.
This quantity $q(x)$ plays the role of a normalized ``intrinsic wine quality.''

\subsubsection{Task specification: product briefs as structured constraints}
A task corresponds to a \emph{product brief} (target customer/taste intent) and is represented as a structured specification.
Concretely, each task $i$ is a tuple
\begin{equation}
\mathcal T_i
=
\Big(
\mathcal X^{(i)}\subseteq[0,1]^{11},\;
w_i,\;
t_i,\;
\sigma_i
\Big),
\end{equation}
where:
\begin{itemize}
    \item $\mathcal X^{(i)}=\prod_{k=1}^{11}[L_{i,k},U_{i,k}]$ is a box constraint in the normalized space;
    \item $w_i=(w_{q,i},w_{m,i})$ with $w_{q,i}\in[0,1]$ and $w_{m,i}=1-w_{q,i}$ controls the trade-off between quality and style-match;
    \item $t_i$ specifies a target \emph{style profile} on a small subset of features (e.g., residual sugar, alcohol, acidity, sulphates),
    expressed in unit coordinates;
    \item $\sigma_i$ specifies per-feature tolerances (style ``bandwidths'') for the same subset.
\end{itemize}

\paragraph{Style-match score.}
Let $\mathcal F$ denote the subset of style-relevant feature indices.
For each $k\in\mathcal F$, the task defines a target $t_{i,k}\in[0,1]$ and tolerance $\sigma_{i,k}>0$.
We define the style match as a Gaussian-shaped preference:
\begin{equation}
m_i(x)
\;=\;
\exp\!\left(
-\frac{1}{2}\sum_{k\in\mathcal F}\left(\frac{x_k-t_{i,k}}{\sigma_{i,k}}\right)^2
\right)\in(0,1].
\label{eq:wine_style_match}
\end{equation}

\paragraph{Within-task objective (bounded).}
Given a task $\mathcal T_i$ and a design $x\in\mathcal X^{(i)}$, the environment returns a bounded score
\begin{equation}
f^{(i)}(x)
\;=\;
w_{q,i}\,q(x) + w_{m,i}\,m_i(x) + \epsilon,
\qquad
\epsilon\sim\mathcal N(0,\sigma^2_{\rm noise}),
\label{eq:wine_task_obj}
\end{equation}
clipped to $[0,1]$ for numerical stability. In our code, $\sigma_{\rm noise}$ is a small constant that emulates
measurement/estimation noise in tasting outcomes.

\subsubsection{Within-task BO details (wine)}
Within each task $i$, we run GP-based BO to maximize $f^{(i)}$ over $\mathcal X^{(i)}$.
We use a single-task GP surrogate and UCB-style acquisition:
\begin{itemize}
    \item \textbf{Initialization:} the first $n_{\rm init}$ evaluations for a task are uniform random points in $\mathcal X^{(i)}$.
    \item \textbf{Model:} a single-task GP (Mat\'ern-$5/2$ kernel in our implementation).
    \item \textbf{Acquisition:} $q$-UCB (with $q=1$), optimized by multi-start local search.
    \item \textbf{Incumbent:} $\overline{y}^{(i)}_s=\max_{r\le s} y^{(i)}_r$.
\end{itemize}
All tasks use the same BO hyperparameters so that differences in performance come from task-level task selection/generation.

\subsubsection{Cross-task utility: committee comparisons and transported CIs}
To compare tasks with potentially different scales and constraints, we score task progress through a bounded utility
$u^{(i)}(\overline{y}_s^{(i)})\in[0,1]$.
In the wine benchmark we use pairwise comparisons against an anchor task $a_t$ and recover cardinal utilities using
Bradley--Terry transport (Appendix~\ref{app:evaluator}).

\paragraph{Pairwise vote model.}
At global round $t$, the evaluator compares the pairs
$(i_t,\overline{y}_t^{(i_t)})$ vs.\ $(a_t,\overline{y}_t^{(a_t)})$ and returns $K_t$ binary votes.
The empirical win-rate $\hat p_t$ estimates
$p_t=\Pr(i_t\succ a_t)$, and a Hoeffding CI yields $p_t\in[p_t^-,p_t^+]$ with high probability.

\paragraph{Transport to cardinal utility.}
Given an anchor utility interval $u^{(a_t)}\in[\LCBu^{(a_t)},\UCBu^{(a_t)}]$,
we update the selected task's utility interval by the logit-add identity:
\begin{equation}
\logit u^{(i_t)} \;=\; \logit u^{(a_t)} + \logit p_t,
\end{equation}
implemented by transporting endpoints.
This provides $\LCBu^{(i_t)}$ and $\UCBu^{(i_t)}$, which are then lifted to value envelopes
$\LCBV,\UCBV$ using the within-task optimization-gap headroom as in Theorem~\ref{thm:value_envelopes_summary}.

\paragraph{Simulated vs.\ LLM committees.}
For reproducibility and ablations, we provide a simulated committee that follows a Bradley--Terry/logistic preference model.
In LLM experiments, the same interface is used but each vote is produced by an LLM judge (committee members differ by prompts/personas),
and $K_t$ is fixed or increased until the target CI width is reached (Appendix~\ref{app:sparse_call_bound}).

\subsubsection{Task generation: coarse-to-fine mutations of product briefs}
We represent tasks as structured specs (e.g., JSON) and generate new tasks by mutating editable fields.
Generation is coupled to the resolution ladder $\epsU_m$ via a decreasing mutation schedule.

\paragraph{Seed tasks (level $m=0$).}
We start from a small set of human-written product briefs (archetypes), which play the role of the initial seed pool.
Each archetype specifies: (i) a name (e.g., ``Dry \& Crisp''), (ii) a weight $w_{q}$ balancing quality vs.\ style,
and (iii) style targets on $\mathcal F$ (Eq.~\ref{eq:wine_style_match}). All begin with the full domain $\mathcal X=[0,1]^{11}$.

\paragraph{Mutation schedule and edited fields.}
At ladder level $m$, we target a mutation ratio $\rho_m=\rho_0 2^{-m}$:
higher $m$ corresponds to smaller edits.
A child task is generated by applying a small number of edits drawn from:
\begin{itemize}
    \item \textbf{Weight mutation:} $w_q \leftarrow \mathrm{clip}(w_q+\Delta_w(m)\xi,0,1)$ with $\xi\sim{\rm Unif}[-1,1]$ and $\Delta_w(m)=\Delta_w^{(0)}2^{-m}$.
    \item \textbf{Style-target mutation:} for $k\in\mathcal F$,
    $t_{k}\leftarrow \mathrm{clip}(t_k+\Delta_t(m)\xi_k,0,1)$ with $\Delta_t(m)=\Delta_t^{(0)}2^{-m}$.
    \item \textbf{Domain refinement (optional):} shrink a subset of bounds $[L_k,U_k]$ around the current anchor best design
    to focus on a more specific style region, while enforcing a minimum per-dimension width.
\end{itemize}
We discard duplicates (up to tolerance) and invalid/degenerate domains.

\subsubsection{Wine benchmark hyperparameters}
Table~\ref{tab:wine_hparams} summarizes the default hyperparameters used in the released wine benchmark implementation.

\begin{table}[t]
\centering
\caption{Default hyperparameters for the wine planning benchmark.}
\label{tab:wine_hparams}
\begin{tabular}{@{}ll@{}}
\toprule
\textbf{Component} & \textbf{Value (default)} \\
\midrule
Wine type & red and white (separate runs) \\
Global horizon & $T=200$ \\
Within-task BO init & $n_{\rm init}$ random points per task (e.g., $6$) \\
Within-task GP & single-task GP, Mat\'ern-$5/2$\\
Acquisition & $q$-UCB with $q=1$ \\
Acqf optimization & multi-start, fixed restarts/raw samples \\
Obs.\ noise (simulated tasting) & small Gaussian noise, clipped to $[0,1]$ \\
Utility oracle & simulated or LLM committee (pairwise votes) \\
Votes per utility query & $K = 64$ (chosen by Eq.~\ref{eq:Ks_explicit_non_circular})\\
Generator batch size & $J=3$ (chosen by Lemma~\ref{lem:batch_success_adaptive}) \\
Mutation schedule & $\rho_m=\rho_0 2^{-m}$, step sizes $\propto 2^{-m}$ \\
\bottomrule
\end{tabular}
\end{table}

\subsection{Synthesis Scaling Experiments}
\label{app:synthesis_scaling_summit}

\subsubsection{Benchmark environment and decision variables}
\label{app:summit_env}

We use the Summit SnAr benchmark. Each
evaluation corresponds to selecting a 4D reaction condition vector
$x=(\tau,\mathrm{equiv},\mathrm{conc},\mathrm{temp})$, with global bounds:
\[
\tau\in[0.5,2.0],\quad \mathrm{equiv}\in[1.0,5.0],\quad \mathrm{conc}\in[0.1,0.5],\quad
\mathrm{temp}\in[30,120].
\]
The environment returns two primary outcomes: space-time yield (STY; higher is better) and
E-factor (lower is better). We normalize them using fixed global bounds
(\texttt{STY\_BOUNDS} $=(0,8000)$ and \texttt{E\_BOUNDS} $=(0,20)$ to
obtain $\phi_{\mathrm{sty}}\in[0,1]$ and $\phi_{\mathrm{ef}}\in[0,1]$. We also define the
``good'' E-factor utility as $\phi_{\mathrm{ef,good}}=1-\phi_{\mathrm{ef}}$.

\subsubsection{Task specification (reaction plans)}
\label{app:summit_task_spec}

A task $i$ corresponds to a \emph{reaction plan} parameterized by:
(i) a local search region (box constraints) expressed in normalized coordinates
$\texttt{bounds\_unit}\subseteq[0,1]^4$, and
(ii) a scalarization weight vector over normalized objectives.
\[
\texttt{spec}_i = \big(\texttt{task\_id}_i,\ \texttt{bounds\_unit}_i,\ \texttt{weights}_i\big),
\quad
\texttt{weights}_i = (w_{\mathrm{sty}}, w_{\mathrm{ef}}),\ \ w_{\mathrm{sty}}+w_{\mathrm{ef}}=1.
\]
Given a design $x$, the task-level scalarized score is
\[
g_i(x)= w_{\mathrm{sty}}\phi_{\mathrm{sty}}(x) + w_{\mathrm{ef}}\phi_{\mathrm{ef,good}}(x),
\qquad g_i(x)\in[0,1].
\]
Each task instance maintains its incumbent best score
$\bar g_i = \max_{s\le t} g_i(x_{i,s})$ and incumbent best design $x_i^\star$.

\subsubsection{Within-task optimizer}
\label{app:summit_within_task_bo}

For each task, we run BO:
\begin{itemize}
  \item \textbf{Initialization.} We draw $n_{\mathrm{init}}=4$ i.i.d.\ designs uniformly from the
  task box constraints, evaluate the simulator, and seed a dataset $(X,Y)$.
  \item \textbf{Surrogate model.} We fit a single-task GP with an Mat\'ern-$5/2$ kernel and
  standard output normalization. Inputs are normalized to $[0,1]^d$.
  \item \textbf{Acquisition.} We use an upper-confidence bound (UCB) acquisition
  $\mu(x)+\beta\sigma(x)$.
  \item \textbf{Acquisition optimization.} We optimize UCB using multi-start gradient-based search
  with \texttt{n\_restarts}=10 and \texttt{raw\_samples}=512, returning the best candidate.
\end{itemize}
Each global meta step allocates exactly \emph{one} new simulator call to exactly one task.

\subsubsection{Task mutation generator}
\label{app:summit_task_generator}

We generate new reaction plans. At a generation event, we pick an \emph{anchor} plan and propose a
batch of mutated tasks by:
\begin{itemize}
  \item \textbf{Weight mutation.} Add Gaussian noise to $w_{\mathrm{sty}}$ with a
  resolution-dependent step size, then clip to $[0.05,0.95]$ and set
  $w_{\mathrm{ef}}=1-w_{\mathrm{sty}}$.
  \item \textbf{Bound mutation.} Shrink and re-center each dimension’s interval around the anchor’s
  incumbent $x^\star$ (with a small random jitter), while enforcing a minimum span of
  \texttt{min\_span\_frac}=0.15 times the global span. The shrink factor decreases with the
  resolution level $m$ via \texttt{shrink\_decay}.
  \item \textbf{Deduplication.} Reject candidate tasks whose (weights,bounds) are within a small
  $\ell_2$ tolerance \texttt{duplication\_eps} of any existing task.
\end{itemize}
Unless the global task budget is reached, accepted tasks are instantiated with fresh within-task BO
state and added to the active set.

\subsection{Algorithm Analysis Experiments: Acquisition-Function Inverse Optimization}
\label{app:algo_analysis}

This appendix provides implementation-level details for the \emph{algorithm analysis} experiments, where we use inverse optimization to discover black-box optimization settings (``tasks'') under which a \emph{target} Bayesian optimization acquisition function outperforms strong alternatives. Concretely, each task is a parameterized synthetic benchmark derived from standard BoTorch test functions, and the task utility is computed by running a \emph{BO race} between acquisition functions and converting their performance gap into a bounded utility in $[0,1]$.

\subsubsection{Task Parameterization and Search Space}
\label{app:algo_analysis:task_space}

\paragraph{Task schema.}
Each task is represented as a JSON-serializable structure:

\begin{verbatim}
TaskSpec {
  base_function: str,
  dim: int,
  bounds: List[[low_i, high_i]] for i=1..dim,
  noise_std: float,
  input_transform: {shift: List[float], scale: List[float]} or null,
  output_transform: {scale: float, offset: float} or null,
  negate: bool,
  metadata: {target_acq: str, rationale: str, m: int}
}
\end{verbatim}

\paragraph{Base function catalog.}
We draw and mutate tasks over the following BoTorch synthetic objectives:
\texttt{Ackley}, \texttt{Rosenbrock}, \texttt{Griewank}, \texttt{Levy},
\texttt{StyblinskiTang}, \texttt{Branin}, \texttt{Beale},
\texttt{SixHumpCamel}, and \texttt{Hartmann}.

\paragraph{Domain and transforms.}
Each task is evaluated over the unit hypercube $\mathbf{x}\in[0,1]^d$ and then mapped to the task bounds. Let $\mathbf{b}_i=[\ell_i,u_i]$ denote the task bounds for dimension $i$.
The wrapper used in the BO race applies (i) an optional affine transform in unit space,
(ii) a bounds map to $\mathbf{b}_i$, and (iii) an optional output affine transform:
\[
\tilde{\mathbf{x}} = \mathrm{clip}\big(\mathbf{x}\odot\mathrm{scale}+\mathrm{shift},\,0,1\big), \qquad
\mathbf{x}_{\mathrm{task},i} = \ell_i + \tilde{x}_i (u_i-\ell_i),
\]
\[
y = f_{\text{base}}(\mathbf{x}_{\mathrm{task}}), \quad
y \leftarrow (-y)\ \text{if negate}, \quad
y \leftarrow y\cdot s_{\text{out}} + o_{\text{out}} \ \text{if output\_transform enabled}.
\]
We add i.i.d.\ Gaussian observation noise only to the value fed into the surrogate model:
$y_{\mathrm{obs}} = y + \epsilon,\ \epsilon\sim\mathcal{N}(0,\sigma^2)$, where \texttt{noise\_std}$=\sigma$.

\paragraph{Validity checks.}
Generated tasks are validated before use: the base function must be in the catalog; \texttt{dim} must be within a configured range; \texttt{bounds} must have length \texttt{dim} and satisfy $\ell_i<u_i$; \texttt{noise\_std}\,$\ge 0$; and input/output transform shapes must be consistent.

\subsubsection{Within-Task Evaluator: BO Race and Ground-Truth Utility}
\label{app:algo_analysis:bo_race}

\paragraph{Normalized regret.}
To normalize across heterogeneous objectives/scales, we use a known global optimum $f^\star$ in BoTorch test functions.
We define a robust scale
\[
\mathrm{robust\_scale} = \mathrm{Quantile}_{q}\!\left(f^\star - f(\mathbf{x}_j)\right), \qquad q=0.10,
\]
and set the denominator
\[
\mathrm{denom}=\max\left(f^\star - y_{\mathrm{init}}^{\max},\ \mathrm{robust\_scale}\right),
\]
where $y_{\mathrm{init}}^{\max}$ is the best noiseless value amongst the initial points.
At each step $s$, an agent’s (simple) normalized regret is
\[
r_a(s)=\frac{f^\star - \overline{y}_t}{\mathrm{denom}},
\]
where $\overline{y}_t$ is the incumbent.
We also track cumulative regret $R_a(s)=\sum_{t=1}^s r_a(t)$.

\paragraph{Gap and utility.}
Let $t$ denote the target acquisition and let $\mathcal{A}\setminus\{t\}$ be its competitors.
We convert cumulative regret into a higher-is-better ``score'' by using negative average cumulative regret:
\[
\mathrm{score}_a(s) \equiv -\frac{R_a(s)}{s}.
\]
We then compute the target advantage (gap) against the strongest competitor:
\[
\bar{g}(s)=\gamma\left(\mathrm{score}_t(s) - \max_{o\ne t}\mathrm{score}_o(s)\right),
\]
where \texttt{gap\_scale}$=\gamma$ (default $\gamma=1$).
Thus $\bar{g}(s)>0$ indicates the target acquisition has lower average cumulative regret.

Finally, we map this gap to a bounded task utility using a sigmoid:
\[
u(s)=\sigma\!\left(\frac{\bar{g}(s)-\kappa}{\tau}\right), \qquad
\sigma(z)=\frac{1}{1+e^{-z}},
\]
with \texttt{sigmoid\_kappa}$=\kappa$ (default $\kappa=0$) and \texttt{sigmoid\_tau\_slope}$=\tau$ (default $\tau=0.25$).
To maintain a non-decreasing ``incumbent'' task utility over time, we store $\bar{g}_{\max}(s)=\max_{s'\le s}\bar{g}(s')$ and set $u(s)=\sigma\big((\bar{g}_{\max}(s)-\kappa)/\tau\big)$.

\subsubsection{LLM Task Mutation Generator for Algorithm Analysis}
\label{app:algo_analysis:llm_generator}

We generate new benchmark tasks by mutating existing tasks. For algorithm analysis, the goal of mutation is directional: produce children that make the \emph{target} acquisition function look better than its strongest alternative.

\paragraph{Actor--critic generation loop.}
Given a parent task, we:
(i) sample $J$ candidate children from an LLM ``actor'',
(ii) optionally score each child with an LLM ``critic'',
(iii) retain the top-$K$ candidates (and apply a diversity filter if enabled),
and (iv) insert the selected children into the task pool.

\paragraph{LLM I/O format.}
The actor is required to output strict JSON with a single key \texttt{child\_task} containing the fields of \texttt{TaskSpec} (excluding \texttt{metadata}, which is filled/updated by the runner). A minimal example is:
\begin{verbatim}
{
  "child_task": {
    "base_function": "Ackley",
    "dim": 6,
    "bounds": [[-5, 5], [-5, 5], [-5, 5], [-5, 5], [-5, 5], [-5, 5]],
    "noise_std": 0.05,
    "input_transform": {"shift": [0,0,0,0,0,0], "scale": [1,1,1,1,1,1]},
    "output_transform": {"scale": 1.0, "offset": 0.0},
    "negate": true,
    "rationale": "Slightly higher noise + mild input rescaling to stress exploration."
  }
}
\end{verbatim}
The critic outputs:
\begin{verbatim}
{"score": 0.82, "short_reason": "Adds noise and mild rescaling; likely helps UCB vs EI/MES."}
\end{verbatim}

\subsubsection{Hyperparameters and Reproducibility Checklist}
\label{app:algo_analysis:hyperparams}

Table~\ref{tab:algo_analysis_hparams} lists the default hyperparameters used in the evaluator and generator implementations.

\begin{table}[t]
\centering
\caption{Default hyperparameters for the Algorithm Analysis.}
\small
\begin{tabular}{l|l}
\hline
\textbf{Component} & \textbf{Key hyperparameters (defaults)} \\
\hline
BO race init & $n_0=4$ shared random initial points \\
GP surrogate & \texttt{SingleTaskGP} + \texttt{Standardize(1)}; exact MLL fit \\
Acq opt. & \texttt{raw\_samples}=256,\ \texttt{num\_restarts}=8 \\
UCB & $\beta_t$ is computed by \citep{hong2023optimization} \\
MES & candidate set $=1024$;\ fantasies $=16$;\ MV samples $=10$;\ Y samples $=128$ \\
Calibration & $N_{\mathrm{cal}}=1024$;\ quantile $q=0.10$ \\
Gap/utility & \texttt{gap\_scale} $=1.0$;\ $\kappa=0$;\ $\tau=0.25$;\ strongest competitor via \texttt{max} \\
\hline
LLM generator & $J=8$ proposals;\ keep $K=10$;\ diversity filter on \\
Actor & temp $=1.0$;\ max tokens $=500$;\ model \texttt{gpt-4o-mini} \\
Critic & temp $=0.0$;\ max tokens $=300$;\ model \texttt{gpt-4o-mini} \\
\hline
\end{tabular}
\label{tab:algo_analysis_hparams}
\end{table}

\subsection{Patent repurposing}
\label{app:patent_repurposing}

\paragraph{Goal.}
We instantiate our framework on a \emph{patent repurposing} proxy task based on the Materials Project (MP):
starting from a \emph{seed} material (e.g., MgO or Bi$_2$O$_3$), we search for candidate materials that are
(i) \emph{structurally similar} to the seed (capturing the intuition of “repurposing” with minimal design changes),
yet (ii) \emph{unusual} along some target combination of materials properties (capturing novelty / potential for new use-cases).
In our implementation, the candidate set is a finite subset of MP entries, and both within-task optimization and task discovery operate
on this fixed set.

\paragraph{Dataset construction (MP subset).}
We query MP via \texttt{mp-api} (requires an \texttt{MP\_API\_KEY}) and build a cached subset of $N$ candidate materials
(typical runs use $N\approx 2{,}000$, configurable).
We retain entries with valid values for a fixed list of scalar properties (examples used in our configs include
\texttt{formation\_energy\_per\_atom}, \texttt{energy\_per\_atom}, \texttt{band\_gap}, \texttt{efermi},
\texttt{poisson\_ratio}, and elastic surrogates such as \texttt{k\_vrh} / \texttt{g\_vrh} (or \texttt{bulk\_modulus} / \texttt{shear\_modulus})),
dropping entries with missing values.
We cache: MP material IDs, reduced formulas, the chosen property matrix, and structural embeddings (below)
(e.g., \texttt{mp\_subset.npz} / \texttt{mp\_subset.csv}; see \texttt{prepare\_mp\_subset.py}).

\paragraph{Structural representation (CGCNN $\rightarrow$ PCA).}
Each material is mapped to a learned structure embedding using a pretrained crystal-graph convolutional network (CGCNN).
We extract a $128$-dimensional penultimate-layer embedding $\phi(m)\in\mathbb{R}^{128}$ for each material $m$ and
reduce it to $d=6$ dimensions using PCA, followed by standardization:
\[
x(m)\;=\;\mathrm{Std}\!\left(\mathrm{PCA}_6(\phi(m))\right)\in\mathbb{R}^6.
\]
The within-task domain is the \emph{discrete} set $\mathcal{X}=\{x(m)\}_{m=1}^N$.

\paragraph{Property surprisal features (per-property KDE).}
Let $p_j(m)$ denote property $j\in\{1,\dots,P\}$ for material $m$.
We $z$-score each property across the dataset and fit an independent 1D Gaussian KDE $\hat{p}_j(\cdot)$ on the standardized values.
We then define the \emph{surprisal} feature
\[
s_j(m)\;=\;-\log \hat{p}_j\!\big(\overline{y}_j(m)\big),
\]
clipping the log-density from below for numerical stability (implementation: \texttt{mp\_dataset.compute\_surprisal}).
This yields a surprisal vector $s(m)\in\mathbb{R}^P$ per material.

\paragraph{Task parameterization and objective.}
A task corresponds to a particular “repurposing direction” specified by:
(i) a convex weight vector $w\in\Delta^{P-1}$ selecting which properties should be \emph{surprising}, and
(ii) a nonnegative regularization strength $\lambda\ge 0$ encouraging proximity to the seed in embedding space.
Fixing a seed material index $m_{\mathrm{ref}}$ (and $x_{\mathrm{ref}}=x(m_{\mathrm{ref}})$), the task objective is
\begin{equation}
f_{w,\lambda}(m)\;=\;w^\top s(m)\;-\;\lambda \,\|x(m)-x_{\mathrm{ref}}\|_2.
\label{eq:mp_task_obj}
\end{equation}
Thus, maximizing $f_{w,\lambda}$ favors candidates that are \emph{rare} in the selected properties while remaining
\emph{close} to the seed in embedding space.
(Implementation: \texttt{mp\_task.PatentTaskSpec.objective}.)

\paragraph{Within-task optimization: discrete GP-UCB over the MP subset.}
For each task $(w,\lambda)$, we run Bayesian optimization on the discrete set $\mathcal{X}$.
We fit an exact GP with an RBF kernel (Gaussian likelihood with small noise), update it on the evaluated points,
and choose the next candidate by maximizing a UCB acquisition on the full discrete pool:
\[
m_{t+1}\in\arg\max_{m\in\mathcal{X}\setminus\mathcal{D}_t}\;\mu_t\!\big(x(m)\big)\;+\;\beta_t^{1/2}\,\sigma_t\!\big(x(m)\big).
\]
Each task is initialized with a small random design (e.g., $4$ points) before UCB selection begins.
(Implementation: \texttt{mp\_bo.MPTaskBO} and \texttt{mp\_bo.DiscreteGPUCB}.)

\paragraph{Task mutation generator.}
Task generation is \emph{heuristic} (no LLM) and operates directly in parameter space.
Given an ``anchor'' task $(w,\lambda)$ at resolution level $m$, we sample $J$ children by Gaussian perturbations with
level-dependent step sizes:
\[
w'\leftarrow \Pi_{\Delta}\!\big(w + \sigma_w(m)\,\xi\big),\quad \xi\sim\mathcal{N}(0,I),
\qquad
\lambda'\leftarrow \mathrm{clip}\!\big(\lambda + \sigma_\lambda(m)\,\eta,\;[\lambda_{\min},\lambda_{\max}]\big),
\]
where $\Pi_\Delta$ denotes projection to the simplex and $\sigma_w(m),\sigma_\lambda(m)$ decrease geometrically with $m$
(e.g., $\sigma(m)=\sigma_0\cdot 2^{-m}$), implementing a coarse-to-fine search over task space.

\paragraph{task-level selection and generation.}
We run our main meta-algorithms (task-UCB / GSBE) on the evolving task set.
Each global step advances exactly one task by one within-task BO evaluation, producing an updated incumbent utility estimate.
Generation events add $J$ mutated tasks around the current anchor when the anchor’s uncertainty width falls below a rung-dependent threshold,
progressively refining the search.
For this benchmark, we use a simple decreasing optimization-gap schedule $\varepsilon_f(s)=\varepsilon_0/\sqrt{s}$
as a conservative surrogate for the within-task optimization error after $s$ evaluations.

\paragraph{LLM-based “repurposability” evaluator.}
An LLM-as-a-judge evaluator compares two candidate task outcomes and returns a binary preference vote.
Concretely, each vote is a call to a chat model with:
(i) a \emph{persona} system prompt (e.g., senior materials scientist / product engineer / patent strategist) and
(ii) a JSON user message containing option A and B (each includes a brief task description and the current best candidate material).
The model is constrained to return JSON of the form \texttt{\{"winner":"A"\}} or \texttt{\{"winner":"B"\}}.
A committee repeats this with multiple personas / samples and aggregates by majority vote.

\paragraph{Default hyperparameters.}
The following values are typical defaults used in our provided scripts (all configurable):
\begin{center}
\begin{tabular}{ll}
\hline
Quantity & Typical value \\
\hline
Candidate set size & $N\approx 2{,}000$ (MP subset; configurable) \\
Embedding dimension & $d=6$ (PCA-reduced CGCNN embedding) \\
\# properties & $P=6$--$7$ (configurable list) \\
Initial BO design per task & $n_0=4$ \\
Global budget & $T=100$ total BO evaluations \\
Children per generation & $J\in[3,8]$ \\
Mutation scales & $\sigma_{w,0}\approx 0.25$--$0.30$, $\sigma_{\lambda,0}\approx 0.20$--$0.25$, scaled by $2^{-m}$ \\
$\lambda$ range & $[0,2]$ (clipped) \\
Within-task kernel & RBF (exact GP), $\sigma_\text{noise}=0.01$\\
\hline
\end{tabular}
\end{center}

\subsubsection{Experimental analysis and ablation study details}
\label{app:ablation_details}

This appendix provides implementation-level details for the three diagnostic studies reported in
§\ref{sec:ablation}: (i) \emph{resolution controllability}, (ii) the \emph{mutation scheduler} ablation,
and (iii) \emph{Lipschitz constant adaptation}.
Unless otherwise stated, we use the same planning benchmark, within-task BO implementation,
and LLM interfaces as in Appendix~\ref{app:settingA_wine} and Appendix~\ref{app:llm_interfaces}.
All plots report the mean and $\pm$ one standard error over random seeds (same aggregation protocol as the main experiments).

\paragraph{What is held fixed across all ablations.}
Across all ablation variants we keep fixed:
(i) the seed task(s) $i_0$ (and the initial seed pool if multiple seeds are used),
(ii) the total global budget $T$ (objective evaluations),
(iii) the within-task BO stack (GP family, acquisition rule, optimization settings; Appendix~\ref{app:settingA_wine}),
(iv) the utility elicitation protocol (committee voting / simulated voting and its CI target; Appendix~\ref{app:evaluator}),
and (v) the generator interface and JSON schema (Appendix~\ref{app:llm_generator_prompts}).
Each ablation modifies only the component stated in its title.

\paragraph{Notation (used in the plots).}
At rung/level $m$, the target mutation ratio is $\rho_m$ and the resolution threshold is $\epsU_m=\epsU_0 2^{-m}$.
A generation event proposes a batch of $J$ child tasks
$\mathcal{B}_m \sim \Gen(a_t,m,J)$ from the current anchor $a_t$.
For each task $i$, we evaluate it by its best-so-far (incumbent) utility $u^{(i)}(\overline{y}_s^{(i)})\in[0,1]$,
and we use the long-run-value notation $U^{(i)}=u^{(i)}(f^{\star(i)})$ from §\ref{sec:bo_resolution_oracle}.

\subsubsection{Resolution controllability (Fig.\ref{fig:controllability})}
\label{app:resolution_controllability_details}

The goal of this experiment is to verify that the \emph{effective task step size} induced by our JSON mutation operator
is controlled by the target mutation ratio $\rho_m$, and to empirically characterize how often mutations produce
near-optimal children as refinement proceeds.

\paragraph{Anchor tasks and sampling protocol.}
We select a small set of representative \emph{anchor tasks} $\{a^{(1)},\dots,a^{(A)}\}$ (we use $A=4$ in the paper)
from independent runs of the LLM-powered wine planning setting.
Anchors are chosen to span multiple refinement levels and utility ranges (early/mid/late in the run).
For each anchor $a$ and each rung $m$ considered, we generate a batch of $J$ children
(using the same generator $\Gen$ as in the main experiment; Appendix~\ref{app:llm_generator_prompts}).
In the plots we use $J=40$ children per condition, and repeat the procedure over multiple random seeds
(for both the generator sampling and the within-task BO randomness) to estimate variability.

\paragraph{Task distance (Fig.\ref{fig:controllability}a).}
Each task is represented as a structured specification (JSON), but the controllability plot requires a numeric distance.
Let $\theta^{(i)}\in\mathbb{R}^{p}$ denote the vector obtained by concatenating the \emph{numeric} fields of the task
(e.g., continuous weights, target style parameters, tolerances, and/or numeric domain parameters; we exclude purely textual fields).
We rescale each coordinate to $[0,1]$ using its natural bounds from the task schema (or global min/max if appropriate),
yielding $\tilde\theta^{(i)}\in[0,1]^p$.
We define the task distance between a child $i$ and its parent anchor $a$ as
\begin{equation}
\label{eq:task_distance_def}
d(i,a)\ :=\ \left\lVert \tilde\theta^{(i)} - \tilde\theta^{(a)} \right\rVert_2.
\end{equation}
We then correlate $d(i,a)$ with the requested mutation ratio $\rho_m$ used when generating $i$.
The plot in Fig.\ref{fig:controllability}(a) shows that larger $\rho_m$ leads to larger distances on average,
confirming that the mutation ratio provides a practical knob for controlling resolution.

\paragraph{Empirical near-optimal generation probability $\widehat\delta_+(m)$ (Fig.\ref{fig:controllability}b--c).}
Condition~\ref{cond:stoch_coverage_adaptive} is stated in terms of the (unknown) optimal value $U^\star$.
To estimate the probability of producing a near-optimal child at rung $m$ in a reproducible way, we use
a proxy \emph{within the sampled pool} at that rung:
for each anchor $a$ and rung $m$, we evaluate each generated child $i\in\mathcal{B}_m$ with a fixed probe budget
of within-task BO steps (identical across children), producing an estimated value $\widehat U^{(i)}$
(e.g., the final value-envelope midpoint or the final best-so-far utility; the choice is fixed across all children).
We then define the best observed value within that batch as
\[
\widehat U^{\star}_{m}(a) \ :=\ \max_{j\in\{a\}\cup \mathcal{B}_m}\widehat U^{(j)}.
\]
A child $i$ is counted as a success at rung $m$ if it is within $\epsU_m$ of the batch best:
\begin{equation}
\label{eq:delta_plus_hat_def}
\mathbf{1}\{i\ \text{successful at rung }m\}
\ :=\
\mathbf{1}\Big\{\,\widehat U^{\star}_{m}(a)-\widehat U^{(i)} \le \epsU_m\,\Big\}.
\end{equation}
Finally, we estimate the generation success probability by averaging over anchors, random seeds, and children:
\begin{equation}
\label{eq:delta_plus_hat}
\widehat\delta_+(m)
\ :=\
\mathbb{E}\Big[\,\mathbf{1}\{i\ \text{successful at rung }m\}\,\Big].
\end{equation}
Fig.\ref{fig:controllability}(b) plots $\widehat\delta_+(m)$ across refinement levels and shows that it remains
non-zero across all measured rungs, while typically decreasing as $m$ increases.
Fig.\ref{fig:controllability}(c) further stratifies success by the task distance $d(i,a)$ in \eqref{eq:task_distance_def},
illustrating how coarse mutations (larger distance) can increase the chance of finding a stronger child when improvement is still possible.

\paragraph{Improvement magnitude (Fig.\ref{fig:controllability}d).}
To quantify how much a child improves over its anchor, we compute
\begin{equation}
\label{eq:improvement_def}
\Delta(i;a) \ :=\ \widehat U^{(i)} - \widehat U^{(a)},
\end{equation}
with $\widehat U^{(\cdot)}$ obtained under the same probe budget and evaluation protocol as above.
Fig.\ref{fig:controllability}(d) reports the distribution (or mean) of $\Delta(i;a)$ across rungs and confirms the expected
exploration-to-exploitation pattern: coarse edits tend to produce larger improvements when improvements are available,
while finer edits are better suited once the anchor is already strong.

\subsubsection{Mutation scheduler ablation (Fig.\ref{fig:scheduler})}
\label{app:mutation_scheduler_details}

This ablation isolates the effect of coupling the mutation ratio schedule $\rho_m$ with the
\emph{width-gated} generation rule in Algorithm~\ref{alg:task-ucb} (Line~\ref{alg_line:scheduler}).

\paragraph{Default scheduler (\AlgoName).}
Our default generator uses the rung-dependent schedule
\[
\rho_m \ =\ \rho_0 2^{-m},
\]
and we \emph{only} advance the rung $m$ and generate a new batch when the current anchor is sufficiently resolved:
\[
w_t \ \le\ c_g \epsU_m,
\]
where $w_t$ is the anchor width from Eq.~\eqref{eq:anchor}.
Intuitively, this ensures that coarse edits are used when uncertainty is still high, while finer edits are used only after
the anchor task has been evaluated enough for BO feedback to become informative.

\paragraph{Ablation variants.}
We compare the default scheduler against three replacements, keeping all other components unchanged
(task selector, generator prompt format/schema, within-task BO, and the evaluator):
\begin{enumerate}
    \item \textbf{Constant mutation ratio.} Use $\rho_m\equiv\rho_0$ for all rungs, while keeping the same width-gated
    generation trigger $w_t\le c_g\epsU_m$.
    \item \textbf{Time-based rung updates (with gating).} Increase the rung deterministically every fixed number of global rounds
    (e.g., every 20 iterations in our implementation), but still only \emph{generate} children when $w_t\le c_g\epsU_m$.
    \item \textbf{No gating (time-only).} Increase the rung deterministically every fixed number of global rounds and generate
    children at those times regardless of the width condition (i.e., ignore $w_t$).
\end{enumerate}
Fig.\ref{fig:scheduler} shows that removing either the rung-dependent $\rho_m$ schedule or the width gating
degrades performance, supporting the design choice that the refinement schedule should be \emph{evidence-driven}
(i.e., coupled to the task-level uncertainty summarized by $w_t$).

\paragraph{What is plotted.}
We plot the same headline metric as in the main wine experiments: the best-so-far achieved utility over global rounds,
aggregated over seeds (mean $\pm$ SEM). All variants receive the same global budget $T$ of objective evaluations.

\subsubsection{Lipschitz constant adaptation ablation (Fig.\ref{fig:scheduler})}
\label{app:lipschitz_adaptation_details}

The envelope construction and the task-UCB rule use an upper bound $\overline L \ge L^\star$
(Theorem~\ref{thm:value_envelopes_summary}), but $L^\star$ is unknown in practice.
We therefore adapt $\overline L$ online using the Balance--Eliminate mechanism in
Algorithm~\ref{alg:GSR} (Appendix~\ref{app:hyperparameters}).

\paragraph{Candidate ladder and update rule.}
We maintain a geometric ladder of candidate bounds
\[
L_j \ =\ L_0 2^j,\qquad j=0,1,\dots,J_L,
\]
and treat each candidate $L_j$ as a hypothesis about a valid global Lipschitz constant.
At each global round, Algorithm~\ref{alg:GSR}:
(i) chooses which hypothesis to test next by balancing the suspected regret certificate (Line~\ref{alg_line:balance}),
and (ii) eliminates hypotheses whose certificate is falsified by observed utility (Line~\ref{alg_line:eliminate}).
The active set $\mathcal A_t$ thus shrinks over time while (with high probability) retaining at least one valid hypothesis.

\paragraph{Ablation variants.}
We compare the adaptive procedure to three simple alternatives:
\begin{enumerate}
    \item \textbf{Fixed-$L$ (underestimated).} Always use the smallest candidate $\overline L=L_0$ (no adaptation).
    \item \textbf{Max-candidate.} Always use $\overline L=\max_{j\in\mathcal A_t} L_j$ among the remaining candidates.
    This is conservative and typically maintains $\overline L\ge L^\star$ once a sufficiently large rung is introduced.
    \item \textbf{Min-candidate.} Always use $\overline L=\min_{j\in\mathcal A_t} L_j$ among the remaining candidates.
    This is aggressive and can become overconfident if the smallest surviving rung still underestimates $L^\star$.
\end{enumerate}
All variants share the same ladder, the same elimination test, and the same within-rung task-UCB logic; they differ only in how
$\overline L$ is chosen from the active set at each time.
Fig.\ref{fig:scheduler} shows that the adaptive procedure performs best overall, while the max-candidate heuristic is competitive,
consistent with the theory that validity ($\overline L\ge L^\star$) is crucial for maintaining correct envelopes.

\paragraph{Reproducibility note.}
For all Lipschitz ablations, we use identical random seeds and the same LLM committee / simulated committee protocol as the main run.
This ensures that changes in performance are attributable to the choice of $\overline L$ (and not to different generator/evaluator randomness).

\subsection{Non-LLM Task Generator Experiment: Unknown Search Spaces}
\label{app:nonllm_taskgen}

This experiment isolates the \emph{multi-task search and selection} components of our method by
replacing the LLM-based task generator with a simple, fully algorithmic (non-LLM) generator.
We instantiate tasks as \emph{candidate search-space bounds} and study optimization when the true
search bounds are unknown and the initial bounds do not contain the global optimum.

\subsubsection{Problem setup and metrics}
\label{app:nonllm_taskgen:setup}

We consider black-box maximization of a scalar objective $f:\mathbb{R}^d \rightarrow \mathbb{R}$.
The optimizer is given an \emph{initial} axis-aligned box
$\mathcal{X}_0 = [\ell_0, u_0] \subset \mathbb{R}^d$, while the true domain
$\mathcal{X}_{\text{true}} = [\ell_{\text{true}}, u_{\text{true}}]$ (and thus the optimizer $x^\star$)
may lie outside $\mathcal{X}_0$.
The algorithm has access to point evaluations $y = f(x)$ (noise-free in our synthetic study).

We report performance using \textbf{best-so-far objective} and the equivalent \textbf{simple regret}
\[
r_t \;=\; f(x^\star) \;-\; \max_{s \le t} f(x_s),
\]
where $x^\star$ is the known global maximizer of the benchmark (used only for evaluation).

\subsubsection{Benchmarks}
\label{app:nonllm_taskgen:benchmarks}

We use standard continuous test functions (implemented as maximization by negating the canonical
minimization form): $f(x) = -f_{\text{std}}(x)$.
The initial box $\mathcal{X}_0$ is a strict subset of $\mathcal{X}_{\text{true}}$ and excludes $x^\star$.

\begin{table}[t]
\centering
\caption{Default hyperparameters for the unknown search space.}
\small
\begin{tabular}{lcccc}
\hline
Benchmark & $d$ & $\mathcal{X}_{\text{true}}$ & $\mathcal{X}_0$ & Budget $T$ \\
\hline
Beale      & 2 & $[-4.5,\,4.5]^2$ & $[-1.0,\,0.0]^2$   & 75  \\
Hartmann6  & 6 & $[0.0,\,1.0]^6$   & $[0.0,\,0.5]^6$    & 150 \\
\hline
\end{tabular}
\label{tab:nonllm_unknownspace_benchmarks}
\end{table}
Figure~\ref{fig:non-llm} displays the first 60 common objective evaluations for visual comparability.
The full runs use the budgets listed in Table~\ref{tab:nonllm_unknownspace_benchmarks};
the truncation is for plotting only.

\subsubsection{Non-LLM task generator: domain expansion operator}
\label{app:nonllm_taskgen:generator}

Concretely, given current bounds $[\ell,u]$, expansion factor $\rho>1$, and an anchor point
$a\in\mathbb R^d$, we define
\[
w=u-\ell,\qquad
\ell'=a-\frac{\rho}{2}w,\qquad
u'=a+\frac{\rho}{2}w.
\]
In the experiments, $a$ is the current incumbent best design of the parent task.  If the parent has
not yet been evaluated, we use the center $(\ell+u)/2$.  The expanded box is clipped to
$\mathcal X_{\rm true}$ only for benchmark feasibility, not for selecting $a$.

\paragraph{Warm-starting new tasks.}
When a new task/domain is created, we immediately collect a small initial design of
$n_{\text{init}}=4$ points sampled uniformly in that box to fit an initial surrogate.
(The same $n_{\text{init}}$ is used for the initial task on $\mathcal{X}_0$.)

\subsubsection{Surrogate model and acquisition optimization}
\label{app:nonllm_taskgen:bo}

Each task $m$ maintains its own Gaussian process surrogate trained on evaluations collected
while querying within $\mathcal{X}^{(m)}$.
We use an exact single-task GP (BoTorch/GPyTorch) trained by maximizing the exact marginal
log-likelihood, and select points using an upper confidence bound (UCB) acquisition.

\subsubsection{Our method without an LLM: GSBE with deterministic expansion}
\label{app:nonllm_taskgen:ours}

We evaluate a non-LLM instantiation of our method by setting tasks to be candidate boxes and
using the deterministic generator in §\ref{app:nonllm_taskgen:generator}.
At a high level, the method alternates between:
(i) selecting \emph{which} task/box to query next using \emph{global} confidence bounds over each box,
(ii) performing a standard within-box BO step, and
(iii) occasionally expanding the search space by generating a new box.

\subsubsection{Baselines for unknown search spaces}
\label{app:nonllm_taskgen:baselines}

We compare against representative specialized methods for unknown search bounds:

\paragraph{UBO (Unknown Bounds Optimization).}
We implement a GP-UCB-based UBO-style expansion rule in which the algorithm expands a domain when
a confidence-based boundary proximity statistic exceeds a threshold $d_{\max}=1.0$.
Expansion uses the same doubling operator ($\rho=2$) and the same cap $m_{\max}=10$.
The UBO confidence parameter is computed via the standard $\beta$-schedule used in the code
(\texttt{tau}=1.0, \texttt{delta}=0.1) and per-step updates based on the model and current domain.

\paragraph{HuBO and HD-HuBO.}
We also include HuBO-style and HD-HuBO-style schedules that expand the domain according to a
predefined growth rule (volume-based growth for HuBO, radius-based growth with an exponential
waiting schedule for HD-HuBO), while selecting points via GP-UCB within the currently active domain(s).
All methods share the same initial design size ($n_{\text{init}}=4$), within-task confidence ($\beta_{\text{bo}}=5.0$),
and the same maximum number of expansions ($m_{\max}=10$).

\subsection{Offline Experiment: Fixed Task Selection}
\label{app:offline_fts}

\paragraph{Setting and goal.}
This experiment removes the LLM task generator entirely and studies \emph{fixed task selection}: a finite set of $K$ tasks is given upfront and the learner must decide \emph{which task to spend the next evaluation on} under a global evaluation budget.
At each global step $t=1,\dots,T$, the algorithm selects a task index $i_t \in \{1,\dots,K\}$ and performs \emph{one} black-box function evaluation on that task using a within-task Bayesian optimizer (described below).
The objective is to maximize the best achieved task utility (and equivalently minimize regret w.r.t.\ the best achievable utility) under the same total budget $T$.

\paragraph{Benchmark suite (fixed tasks).}
We use $K=6$ standard continuous optimization benchmarks implemented via BoTorch test functions.
Following the code, we treat each benchmark as a \emph{maximization} problem by negating the standard minimization form, i.e., for each task $i$ we optimize
$f_i(x)=-y_i(x)$ where $y_i$ is the usual BoTorch test function.
All tasks use the standard box constraints listed in Table~\ref{tab:offline_fts_tasks}.
We also store each task's known optimum value $f_i^\star$ (Table~\ref{tab:offline_fts_tasks}), which is used only for regret-based utility variants (not the default).

\begin{table}[t]
  \centering
    \caption{Fixed-task benchmark suite.}
  \small
  \setlength{\tabcolsep}{6pt}
  \begin{tabular}{lccc}
    \hline
    Task $i$ & Dim.\ $d$ & Domain $\mathcal{X}_i$ & $f_i^\star$ \\
    \hline
    Ackley-2D      & $2$ & $[-5,5]^2$                    & $0$ \\
    Beale-2D       & $2$ & $[-4.5,4.5]^2$                & $0$ \\
    Branin-2D      & $2$ & $[-5,10]\times[0,15]$         & $-0.397887$ \\
    Hartmann-6D    & $6$ & $[0,1]^6$                     & $3.322368$ \\
    Levy-2D        & $2$ & $[-10,10]^2$                  & $0$ \\
    Rosenbrock-4D  & $4$ & $[-2,2]^4$                    & $0$ \\
    \hline
  \end{tabular}
  \label{tab:offline_fts_tasks}
\end{table}

\paragraph{Noisy function evaluations.}
When querying task $i$, we observe a noisy outcome
$
y = f_i(x) + \epsilon,\ \ \epsilon\sim\mathcal{N}(0,\sigma_{\text{noise}}^2)
$
with $\sigma_{\text{noise}}=0.01$ (default).
The within-task GP model is fit to these noisy observations, while the experiment state also tracks an incumbent $\overline{y}_i$ as the best achieved value so far on task $i$.

\paragraph{Utility normalization (comparable utilities across tasks).}
Because raw objective scales differ across tasks, we map incumbents to a comparable $[0,1]$ utility.
The default utility model in the code is
\begin{equation}
u_i(z)\;=\;\Phi\!\left(\frac{z-\mu_i}{\sigma_i}\right),
\label{eq:offline_fts_cdf_objective}
\end{equation}
where $\Phi$ is the standard normal CDF and $(\mu_i,\sigma_i)$ are task-specific location/scale parameters.
By default, $(\mu_i,\sigma_i)$ are \emph{estimated} per task using Monte Carlo sampling:
we draw $S=20000$ points uniformly from $\mathcal{X}_i$, evaluate $f_i(\cdot)$, and set $\mu_i$ and $\sigma_i$ to the sample mean and standard deviation (robust estimation is disabled by default).

\paragraph{Task-selection methods compared.}
All methods operate under the same global budget $T=200$ evaluations.
\begin{itemize}
  \item \textbf{GSBE-MetaUCB (ours, no generator).}
  We use the fixed-task variant implemented in \texttt{gsbe\_meta\_ucb.py}.
  At each step, select the task maximizing a task-UCB score based on the oracle upper confidence bound and an exploration ``headroom'' term that decays with the number of pulls $s_i$ of task $i$:
  $
  \mathrm{score}_i \;=\; \mathrm{UCB}_i \;+\; \bar L \cdot \varepsilon(s_i),
  $
  where $\varepsilon(s)=c/\sqrt{s}$ with default \texttt{headroom\_coef} $c=0.5$ and $\bar L$ is the running estimate used by the implementation.
  \item \textbf{Round-robin.} Cycle through tasks deterministically.
  \item \textbf{Uniform random.} Sample a task uniformly at random each step.
  \item \textbf{Successive halving.}
  Use $\eta=3$ and repeatedly allocate equal per-task budgets, then prune the bottom $1-1/\eta$ fraction of tasks based on the observed utility estimate $\hat u_i$.
  \item \textbf{Hyperband.}
  Use Hyperband with $\eta=3$ and maximum per-bracket budget $R=T$, running successive-halving brackets with different initial task counts/budgets.
\end{itemize}
The global budget cap is enforced for all baselines.  For Hyperband, $R=T$ is the maximum resource
parameter used to construct brackets, but execution is stopped as soon as the cumulative number of
black-box evaluations reaches $T$; initialization evaluations are included in this count.

\subsection{Offline Experiment: objective selection}
\label{app:offline_objective_selection}

This subsection provides implementation-level details for the \textbf{offline} experiment in
Fig.\ref{fig:synthetic}(b), where the task pool is fixed \emph{a priori} and the task-learner must decide
\emph{which objective (task) to allocate each evaluation to}. This experiment isolates the \emph{task-selection}
component of GSR, i.e., Algorithm~\ref{alg:task-ucb} without task generation.

\paragraph{Task pool and shared domain.}
We fix a set of $K$ objectives (tasks) $\{f^{(i)}\}_{i=1}^{K}$ drawn from the BoTorch synthetic benchmark suite.
In the \textbf{objective-selection} variant, all tasks share the \emph{same} design domain
\[
\mathcal{X}^{(i)} \equiv \mathcal{X} = [0,1]^d,
\qquad
d=6,
\]
so the only difference across tasks is the objective function.
At each global round $t\in\{1,\dots,T\}$, a method chooses an objective index $i_t\in[K]$ and a design
$x_t\in\mathcal{X}$, observes one scalar outcome
\[
y_t^{(i_t)} = f^{(i_t)}(x_t) + \xi^f_t,
\]
and updates the incumbent $\overline{y}_{s}^{(i)}=\max_{r\le s} y_r^{(i)}$ for the selected objective only.
Importantly, \emph{only one objective is observed per round}, which matches a ``decoupled evaluation'' scenario
where measuring different objectives has separate costs.

The six objectives used in Fig.~\ref{fig:synthetic}(b) are:
Ackley-6D, Griewank-6D, Levy-6D, Rosenbrock-6D, Styblinski--Tang-6D, and Hartmann-6D.
All are evaluated on the shared unit cube after the standard BoTorch input transform.

\paragraph{Utility normalization (cross-task comparability).}
Because each $f^{(i)}$ can have a different scale and offset, we map incumbents to a common $[0,1]$ utility scale
using a task-wise Gaussian CDF:
\begin{equation}
\label{eq:offline_ucdf}
u^{(i)}(z)\ :=\ \Phi\!\Big(\frac{z-\mu^{(i)}}{\sigma^{(i)}}\Big)\in[0,1],
\end{equation}
where $\Phi(\cdot)$ is the standard normal CDF and $(\mu^{(i)},\sigma^{(i)})$ are \emph{fixed per task}.
In our implementation, $(\mu^{(i)},\sigma^{(i)})$ are computed once from a calibration set
$\mathcal{C}=\{x_j\}_{j=1}^{N_{\mathrm{cal}}}\subset[0,1]^d$ (Sobol points):
\[
\mu^{(i)} := \frac{1}{N_{\mathrm{cal}}}\sum_{x\in\mathcal{C}} f^{(i)}(x),
\qquad
(\sigma^{(i)})^2 := \frac{1}{N_{\mathrm{cal}}-1}\sum_{x\in\mathcal{C}}\bigl(f^{(i)}(x)-\mu^{(i)}\bigr)^2,
\]
with a small floor $\sigma^{(i)}\leftarrow\max\{\sigma^{(i)},\sigma_{\min}\}$ for numerical stability.
This monotone normalization allows task selection to be evaluated in a scale-invariant manner.

\paragraph{Ground-truth values (for evaluation only).}
For plotting, we use the best utility found up to round $t$:
\[
\text{simple-regret}(t)
:=
U^\star -
\max_{\tau\le t}
u^{(i_\tau)}\!\left(\overline y^{(i_\tau)}_{s_\tau^{(i_\tau)}}\right).
\]
This is why the curves in Fig.~\ref{fig:synthetic}(b) are monotone nonincreasing.

\paragraph{Within-task Bayesian optimization (shared across all methods).}
Each objective $i$ maintains its own GP surrogate fit to its observed data
$\{(x_r,y_r^{(i)})\}_{r=1}^{s}$.
Unless otherwise noted, we use the same BO stack for all objectives and all methods:
\begin{itemize}
    \item \textbf{Surrogate:} \texttt{SingleTaskGP} (BoTorch/GPyTorch) with an Mat\'ern-$5/2$ kernel,
    input normalization to $[0,1]^d$, and \texttt{Standardize} outcome transform.
    \item \textbf{Noise:} synthetic objectives are treated as noiseless; we fit with a small jitter/noise floor.
    \item \textbf{Acquisition:} UCB-style acquisition (e.g., \texttt{qUpperConfidenceBound} with $q{=}1$).
    \item \textbf{Acqf optimization:} multi-start gradient-based optimization (fixed \texttt{raw\_samples} and
    \texttt{num\_restarts}) over $\mathcal{X}$.
\end{itemize}
When an objective $i$ is selected for the first time, it is initialized with a small number of random points;
these initialization evaluations are counted toward the global budget $T$ (i.e., there is no ``free'' warm start).

\paragraph{Baselines.}
We compare against three families of baselines, each using the \emph{same} within-task BO module above:
Preference-based MOBO baselines (MOBO-PC / BOPE): we adapt these methods to the \emph{decoupled-observation} setting where only one objective is observed per round. Concretely, we maintain separate GP surrogates for each objective, compute the MOBO acquisition using the current posterior over the \emph{full} objective vector (integrating over unobserved objectives via GP posteriors), propose a candidate design $x_t$, and then choose which objective to evaluate at $x_t$ using an uncertainty-driven rule (evaluate the objective with the largest posterior uncertainty at $x_t$). The resulting single observation is then used to update only that objective's GP.

\begin{table}[t]
\centering
\small
\caption{Default hyperparameters for the objective selection benchmark.}
\label{tab:offline_obj_select_hparams}
\begin{tabular}{@{}ll@{}}
\toprule
\textbf{Item} & \textbf{Value / description} \\
\midrule
\# objectives (tasks) & $K=6$ BoTorch test functions \\
Domain & $\mathcal X=[0,1]^6$ \\
Global budget & $T=200$ \\
Utility map & $u^{(i)}(z)=\Phi((z-\mu^{(i)})/\sigma^{(i)})$ \\
Within-task GP & Mat\'ern-$5/2$, \texttt{SingleTaskGP} + \texttt{Standardize} \\
\bottomrule
\end{tabular}
\end{table} 

\section{Additional Experiments and Discussion}
\vspace{-0.5em}
\subsection{Additional experiments}\label{app:add_exp}
\begin{figure}
    \centering
    \includegraphics[width=0.6\textwidth]{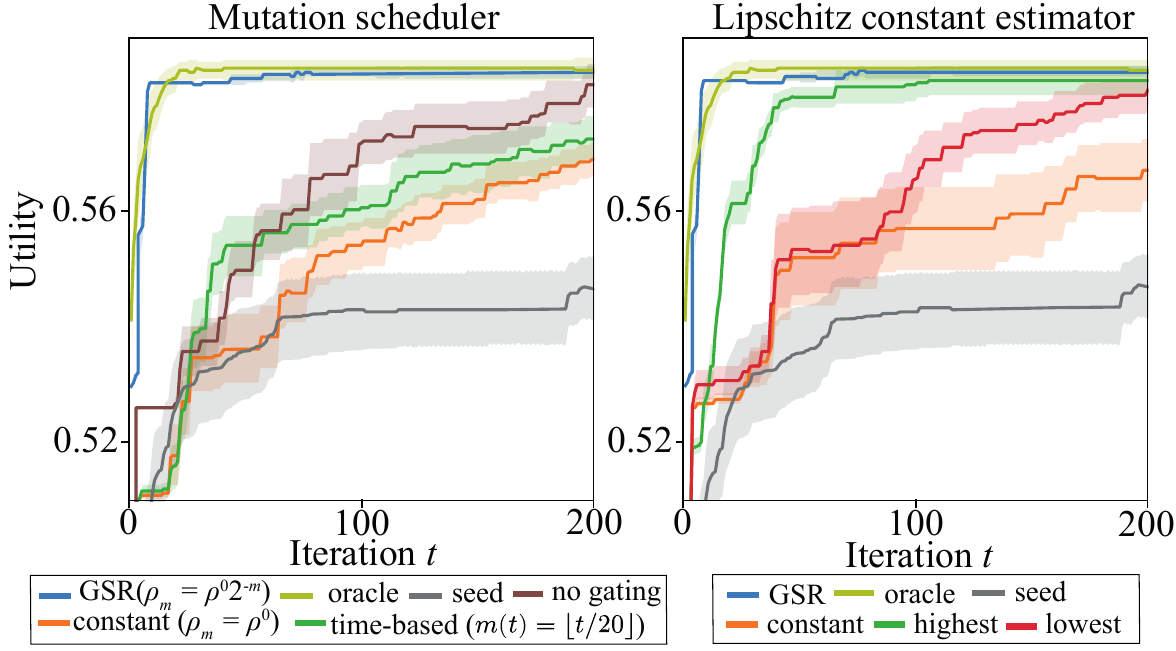}
    \caption{Ablation study. Replacing the mutation scheduler $\rho_m$ or the Lipschitz constant adaptation causes a performance drop.
    }
    \label{fig:scheduler}
\end{figure}

\textbf{Lipschitz constant adaptation.}
We evaluate adaptation of the Lipschitz constant $\overline{L}$ via regret balancing (Appendix~\ref{app:hyperparameters}). The method maintains candidates $L_j = L_0 2^j$ and discards them when falsified. We compare against three variants: (i) constant $L_j = L_0$; (ii) highest $\max_j L_j$, selecting the largest remaining candidate; (iii) lowest $\min_j L_j$, selecting the smallest. Adaptive selection performs best overall. Notably, the largest-candidate variant also performs well, as it guarantees $\overline{L} \geq L^\star$, supporting our theoretical analysis.

\textbf{Mutation scheduler $\rho_m$.}
We further study the mutation scheduler $\rho_m = \rho_0 2^{-m}$ (Fig.\ref{fig:scheduler}) by replacing it with three variants: (i) constant $\rho_m = \rho_0$; (ii) time-based scheduling, where $m$ increases every 20 iterations and child tasks are generated when $w_t \leq c_g \epsU_m$; and (iii) no gating, which ignores the $w_t$ condition and updates $m$ every 20 iterations. The results show that coupling $\rho_m$ with the $w_t$-based generation rule is crucial for efficiency.

\subsection{LLM-human evaluation alignment}\label{app:alignment}
To test whether the committee is merely exploiting idiosyncratic LLM preferences, we ran a blinded human–LLM alignment study on two representative scientific settings: synthesis scaling (SnAr) and patent repurposing (MgO). Two PhD-level materials scientists compared 100 randomly sampled task pairs using exactly the same task descriptions and partial optimization histories shown to the committee; the annotators were blinded to the committee votes. We measured agreement using Cohen’s kappa, since it measures agreement beyond chance and avoids comparing raw vote counts across groups of different size (64 committee votes vs. 2 human votes).

As Table~\ref{tab:human} shows, agreement between each human annotator and the committee majority vote is high. We interpret this as evidence that the committee is not obviously pathological under the shared-information protocol, while not claiming full validation of scientific judgment. GSR can also be instantiated with human feedback directly, at higher annotation cost; Appendix C.1.2 studies a feedback-efficient variant. For transparency, the anonymous repository now includes raw evaluator traces for SnAr, including the compared tasks, rationale, and vote.

\begin{table}
    \centering
    \caption{Human-LLM alignment using Cohen’s kappa.}
    \begin{tabular}{ccc}
    \toprule
      Voter   & synthesis  & repurposing\\
      \midrule
      human A   & 0.898 & 0.926 \\
      human B   & 0.912 & 0.947\\
      \bottomrule
    \end{tabular}
    \label{tab:human}
\end{table}

\paragraph{Annotator instructions, interface, and compensation.}
The annotators were shown a spreadsheet-style interface with two anonymized task--incumbent
summaries, labeled A and B, plus the same partial optimization history available to the LLM committee.
They were given the following written instruction:

\begin{quote}\small
For each row, compare Candidate A and Candidate B as scientific or product repurposing hypotheses.
Use only the task description, current best candidate, and partial optimization history shown in the row.
Select the candidate that you would prioritize for the next expensive evaluation.  If both are close, choose the one that appears more scientifically plausible, safer, or more actionable.  Record only A or B; optional comments may be added for quality control.
\end{quote}

No personally identifying information was collected.  The study used unpaid internal collaborator annotators; compensation was not paid. We follow the existing protocols in our institutions.

\subsection{Computational cost}\label{app:computation}
\paragraph{Model choice}
We chose GPT-4o-mini because this instantiation makes many structured generation and committee-evaluation calls, so both latency and cost matter. To test model sensitivity, we varied only the task-generation model on the red-wine benchmark while keeping the evaluator fixed at GPT-4o-mini (Table~\ref{tab:model_cost}). We isolate generator sensitivity because changing the evaluator would change the utility definition itself, making runs no longer directly comparable. The committee uses 64 votes per step run in parallel; in this benchmark, that adds about 15 s wall-clock latency and about \$3.81 total committee-evaluation cost over 100 steps. Task generation is triggered only intermittently (6.2 times per 100 steps on average); with GPT-4o-mini this contributes about 21 s latency per trigger and about \$2.00 total generation cost over 100 steps.

\begin{table}
    \centering
    \caption{LLM variants for task generation}
    \begin{tabular}{ccccc}
    \toprule
        alg. & utility & latency & API cost & average raw proposals per trigger \\
    \midrule
        random	&0.5602	&N/A	&N/A	&N/A \\
        SH	&0.5646	&N/A	&N/A	&N/A\\
        hyperband	&0.5524	&N/A	&N/A	&N/A\\
        gpt-4o-mini	&0.5804	&21 s	&\$2.00	&6.6\\
        gpt-5.4	&0.6184	&60 s	&\$40.00	&4.8\\
        gpt-5.4-mini	&0.6079	&28 s	&\$12.00	&5.2\\
        gpt-5.4-nano	&0.5980	&12 s	&\$3.27	&5.5\\
        gpt-5.4-nano-reasoning-low	&0.6002	&15 s	&\$3.60	&5.3\\
        gpt-5.4-nano-reasoning-high	&0.6047	&24 s	&\$4.90	&5.2\\
        \bottomrule
    \end{tabular}
    \label{tab:model_cost}
\end{table}

In our current implementation, the task-generation prompt grows with the mutation history, so generation latency does increase over time. Empirically, we observe an approximately linear increase in generation latency (about +2.93 s per additional mutation history in the red-wine benchmark). The main context-growth cost is therefore in the generation call; the hard validation layer itself is relatively lightweight, since it mainly consists of structured parsing and rule-based feasibility checks on the current proposal. Also, because task generation is triggered only intermittently (6.2 times per 100 steps on average in our experiments), the overall wall-clock overhead remained manageable in our setting. For substantially longer runs, a natural extension would be to summarize or truncate early mutation history to cap prompt length.

\paragraph{Prompt effects.}
To test prompt sensitivity, we evaluated four semantically equivalent task-generation prompts on red-wine while keeping the evaluator model and evaluator prompt fixed. Mean utility changed only modestly (0.5798–0.5973 versus 0.5804 for the original prompt), and all variants remained above Random/SH/HB. For transparency, the anonymous repository \url{https://anonymous.4open.science/r/Generate-Select-Refine-2D48/} includes the task-generation prompts for all four variants.

\subsection{Discussion}\label{app:discussion}
\textbf{Why not preferential BO?}
Preferential BO (PBO; \citep{gonzalez2017preferential}) optimizes a preferred \emph{input} $x$; GSR optimizes a preferred \emph{task} $i$ using LLMs. Our framework also supports non-preference utilities (see §\ref{sec:io}-\ref{sec:ablation}).\\
\textbf{Can kernels vary across tasks?}
Yes, within kernel families with tractable information gain (see \citep{lee2025consequences}).\\
\textbf{Why not use a single GP over $(x_\tau, \tilde{u}_\tau)_{\tau=1}^t$?}
Utility observations are nonlinear expectations, so standard GP assumptions do not apply (see \citep{astudillo2019bayesian}).\\
\textbf{Aren't LLM evaluations unreliable?}
Both the generator and the utility feedback can be replaced by humans.\\
\textbf{LLMs/human feedback is costly.}
Utility queries can be sparse (e.g., at $\{1,2,4,8,\dots\}$) while keeping $\mathcal{O}(\sqrt{\log T})$ regret with a sublinear number of votes (Appendix~\ref{app:sparse_call_bound}).\\
\textbf{What if small edits ($\rho_m$) cause large semantic shifts?}
We mitigate this with two safeguards: (i) an actor–critic generator, where a critic LLM verifies desiderata, and (ii) a rule-based verifier enforcing hard feasibility constraints (Appendix~\ref{app:llm_interfaces}); Fig.\ref{fig:controllability} confirms their effectiveness.

\subsection{Societal impact}
This work could be especially useful for self-driving laboratories, where planning is increasingly becoming the bottleneck, and faster application discovery could accelerate the industrial adoption of new materials. On the other hand, greater autonomy also increases the risk of misalignment with human values. Our framework is explicitly controlled by the utility function, so users must carefully consider the consequences of how that utility is specified. This is a general risk shared by highly automated decision-making systems.

\subsection{Asset licenses}
\label{app:asset_licenses}
Table~\ref{tab:asset_licenses} lists the existing assets used in the experiments and their licenses.

\begin{table}[t]
\centering
\small
\caption{Existing assets and licenses.}
\label{tab:asset_licenses}
\begin{tabular}{lll}
\toprule
Asset & Use in this paper & License / terms \\
\midrule
Wine Quality dataset~\citep{cortez2009modeling} & Wine planning benchmark & CC BY 4.0 \\
\textsc{SUMMIT} simulator~\citep{felton2021summit} & Synthesis scaling benchmark & MIT License \\
Materials Project data~\citep{jain2013commentary} & Materials repurposing benchmark & CC BY 4.0 \\
CGCNN~\citep{xie2018crystal} & Crystal-structure embeddings & MIT License \\
BoTorch~\citep{balandat2020botorch} & BO implementation & MIT License \\
GPyTorch~\citep{gardner2018gpytorch} & GP implementation & MIT License \\
\bottomrule
\end{tabular}
\end{table}



\end{document}